\documentclass[a4paper,fleqn]{cas-dc}
\usepackage[authoryear]{natbib}
\usepackage{setspace}
\usepackage{tikz}
\usepackage{amsmath}
\usepackage[authoryear]{natbib}
\usepackage{tikz}
\usepackage{amsmath}
\usepackage{enumitem}
\usepackage{placeins}
\usepackage{algorithm}
\usepackage{algorithmic}
\usepackage{enumitem}
\usepackage{subcaption} 
\usepackage{float}      
\usepackage{multirow}
\usepackage{booktabs}
\usepackage{graphicx}
\usepackage{etoolbox}
\usepackage{natbib}   
\newdefinition{definition}{Definition}
\newtheorem{theorem}{Theorem}
\newproof{proof}{Proof}
\newtheorem{lemma}{Lemma}
\tikzset{
    inputnode/.style={circle, draw, fill=orange!10, minimum size=20pt, inner sep=0pt, font=\small},
    middlenode/.style={circle, draw, fill=blue!10, minimum size=20pt, inner sep=0pt, font=\small},
    same_module/.style={rectangle, draw, fill=green!10, minimum size=20pt, inner sep=0pt, font=\small},
    hiddennode/.style={circle, draw, fill=blue!10, minimum size=20pt, inner sep=0pt, font=\small},
    outputnode/.style={circle, draw, fill=green!10, minimum size=20pt, inner sep=0pt, font=\small},
    defnode/.style={minimum size=20pt, inner sep=0pt, font=\small},
    signal/.style={thick, ->, >=latex, color=gray!50} 
}

\begin{document}
\let\WriteBookmarks\relax
\def\floatpagepagefraction{1}
\def\textpagefraction{.001}
\shorttitle{UniSymNet: A \textbf{Uni}fied \textbf{Sym}bolic \textbf{Net}work with Sparse Encoding and Bi-level Optimization}
\shortauthors{X. Li, J. Zhang, D.Li et~al.}
\tnotetext[1]{Supported by the National Natural Science Foundation of China, No. 12201024.}
\title [mode = title]{UniSymNet: A \textbf{Uni}fied \textbf{Sym}bolic \textbf{Net}work with Sparse Encoding and Bi-level Optimization}

\author[1]{Xinxin Li}[
                        auid=000,bioid=1,
                        orcid=0009-0008-8962-0084
                        ]

\ead{51265500102@stu.ecnu.edu.cn}
\affiliation[1]{organization={School of Mathematical Sciences, East China Normal University},
                city={Shang Hai},
                postcode={200241}, 
                country=China}

\author[2,3]{Juan Zhang}
\cormark[1]
\affiliation[2]{organization={Institute of Artiﬁcial Intelligence, Beihang University},
                postcode={100191}, 
                city={Beijing},
                country={China}}

\affiliation[3]{organization={Shanghai Zhangjiang Institute of Mathematics},
                postcode={201203}, 
                city={Shanghai},
                country={China}}
\ead{zhang_juan@buaa.edu.cn}

\author[4,5]{Da Li}
\affiliation[4]{organization={Institute of Applied Physics and Computational Mathematics},
                postcode={100094}, 
                city={Beijing},
                country={China}}

\affiliation[5]{organization={Academy for Advanced Interdisciplinary Studies, Northeast Normal University},
                postcode={130024}, 
                city={Changchun},
                state={Jilin},
                country={China}}

\ead{dli@nenu.edu.cn}

\author[2]{Xingyu Liu}
\ead{ZY2342118@buaa.edu.cn}

\author[3, 4]{Jin Xu}
\ead{xujin22@gscaep.ac.cn}

\credit{Data curation, Writing - Original draft preparation}

\author[3,4]{Junping Yin}
\cormark[1]
\ead{yinjp829829@126.com}

\cortext[cor1]{Corresponding author}







\begin{abstract}
Automatically discovering mathematical expressions is a challenging issue to precisely depict natural phenomena, in which Symbolic Regression (SR) is one of the most widely utilized techniques. Mainstream SR algorithms target on searching for the optimal symbolic tree, but the increasing complexity of the tree structure often limits their performance. Inspired by neural networks, symbolic networks have emerged as a promising new paradigm. However, existing symbolic networks still face certain challenges: binary nonlinear operators $\{\times, \div\}$ cannot be naturally extended to multivariate, training with fixed architecture often leads to higher complexity and overfitting. In this work, we propose a \textbf{Uni}fied \textbf{Sym}bolic \textbf{Net}work that unifies nonlinear binary operators into nested unary operators, thereby transforming them into multivariate operators. The capability of the proposed UniSymNet is deduced from rigorous theoretical proof, resulting in lower complexity \textcolor{black}{and stronger expressivity.} Unlike the conventional neural network training, we design a bi-level optimization framework: the outer level pre-trains a Transformer with sparse label encoding \textcolor{black}{scheme} to guide UniSymNet structure selection, while the inner level employs objective-specific strategies to optimize network parameters. This allows for flexible adaptation of UniSymNet structures to different data, leading to reduced expression complexity. The UniSymNet is evaluated on low-dimensional Standard Benchmarks and high-dimensional SRBench, and shows excellent symbolic solution rate, high fitting accuracy, and relatively low expression complexity.
\end{abstract}


\maketitle

\section{Introduction}
\label{sec:introduction}
In the natural sciences, discovering interpretable and simplified models from data is a core goal. For example, Kepler discovered precise mathematical laws governing planetary motion through astronomical observations, laying the foundation for Newton’s theory of gravitation.

Symbolic Regression (SR) aligns with this goal by discovering mathematical expressions that capture the underlying relationships in the data. Specifically, given a dataset \(\{(x_i, y_i)\}_{i=1}^N\), the aim is to identify a function \(f(x)\) such that \(y_i \approx f(x_i)\) for all \(i\). SR has proven effective in various fields, from the derivation of physical laws based on empirical data \citep{udrescu2020ai} to the advancement of research in materials science \citep{wang2019symbolic}. However, unlike traditional regression, SR requires discovering functional forms and identifying parameters without prior structural assumptions, which makes SR challenging. SR generally involves two stages: (i) searching the functional structure space for viable functional skeletons, then (ii) optimizing parameters for the selected structure. How to represent the functional structure space remains a challenge due to the diversity and complexity of real physical equations.

Most research in SR uses symbolic trees to construct the functional structure space. Specifically, the symbolic tree representation method encodes basic operators and variables as tree nodes and represents equations in the form of a pre-order traversal for the forward pass process. In this way, SR tasks are transformed into the problem of searching for the optimal symbolic tree within a vast symbolic tree space. To tackle this, Genetic Programming (GP)-based methods \citep{schmidt2009distilling} employ biology-inspired strategies to conduct this search, while Reinforcement Learning (RL)-based methods \citep{petersen2020deep, tenachi2023deep} leverage risk-seeking policy gradients to guide the search process. These methods\textcolor{black}{, learning from scratch,} contrast with Transformer-based methods \citep{biggio2021neural, kamienny2022end}, which adapt the machine translation paradigm from NLP \citep{vaswani2017attention,devlin2019bert} to symbolic tree search \citep{valipour2021symbolicgpt}. The pre-training mechanisms in Transformer-based methods allow them to improve with experience and data.

Recently, based on the theoretical foundation of the Universal Approximation Theorem\citep{hornik1989multilayer}, deep neural networks demonstrate remarkable capabilities to approximate complex functions and capture intricate nonlinear patterns within high-dimensional data spaces. In this context, EQL \citep{martius2017extrapolation} introduced modifications to feedforward networks by replacing standard activation functions with operators such as \( \cos \) and \( \times \), thus transforming them into symbolic networks. \textcolor{black}{We refer to symbolic networks that incorporate this design as EQL-type networks.} Since then, an increasing number of SR methods based on symbolic network representations have emerged \citep{sahoo2018learning,liu2023snr,li2024neural}, showing competitive performance against methods based on symbolic tree representations. Moreover, compared to the black-box nature of standard neural networks, symbolic networks are transparent and possess better extrapolation capabilities, which has led to their increasing use in scientific exploration and discovery \citep{long2019pde,kim2020integration}. 

In symbolic network, $\{ +, -\}$ and scalar multiplication operator are implemented through affine transformations, while$\{\times,\div,\mathrm{pow}\}$ are implemented using activation functions. This \textcolor{black}{design} has two advantages: 1) it allows $\{+,-\}$ to extend from binary operators in symbolic trees to multivariate operators, and 2) the symbolic network can support cross-layer connections through the $\mathrm{id}$ operator\citep{wu2023discovering}. These features help \textcolor{black}{to} reduce the complexity of the equation representation. However, as noted in current research \citep{martius2017extrapolation, sahoo2018learning, li2024neural}, $\mathrm{pow}$ operators are typically restricted to activation functions with predefined real exponents (e.g., squaring \((\cdot)^2\) or taking the square root \(\sqrt{\cdot}\)), which inherently limits the expressive power of network.

Integrating operators $\{\times,\div,\mathrm{pow}\}$ into neural architectures as naturally as $\{+,-\}$ and scalar multiplication operators would be more consistent with human intuition. In algebra, the isomorphism between the multiplicative group \((\mathbb{R}^+, \times)\) and the additive group \((\mathbb{R}, +)\) is established through the homomorphic mapping \(\ln\) and its inverse \(\exp\)\citep{lang2012algebra}, which provides a unified framework for translating nonlinear operators into linear ones. This transformation has been successfully applied in machine learning to enhance numerical stability and computational efficiency, such as log softmax\citep{goodfellow2016deep} and logarithmic likelihood of probability calculation. \textcolor{black}{Motivated by the above, we propose a new $\Psi$ representation (Equation \ref{eq:binary_unification_constrained}), which represents binary operators \(\{\times, \div, \mathrm{pow}\}\) using nested unary operators \(\{\ln, \exp\}\). This formulation naturally establishes a \textbf{Uni}fied \textbf{Sym}bolic \textbf{Net}work consisting solely of \textbf{Un}ary operators (UniSymNet).} In the UniSymNet, \(\{\times, \div\}\) are transformed from binary operators in \textcolor{black}{EQL-type networks} into multivariate operators, similar to \(\{+, -\}\), \textcolor{black}{leading to more compact representations.}



\textcolor{black}{UniSymNet not only incorporates the $\Psi$ representation into its network architecture but also introduces a Transformer-guided structural search strategy. This strategy leverages prior knowledge acquired through pre-training to flexibly and efficiently search network structures in the process of discovering new equations. Unlike conventional approaches that rely on backpropagation within a fixed network structure, this design mitigates overfitting and favors the discovery of simpler mathematical expressions. This aligns with a fundamental goal of natural science: deriving interpretable and parsimonious models from data. To realize this design, UniSymNet adopts a bi-level optimization framework, where the outer optimization employs a pre-trained Transformer to guide structural selection, and the inner optimization applies objective-specific optimization strategies to learn the parameters of the symbolic network.}

In summary, we introduce the main contributions of this work as follows:
\begin{itemize}
\item We introduce UniSymNet, a symbolic network \textcolor{black}{with $\Psi$ representation, which} unifies binary operators $\{\times, \div, \\ \mathrm{pow}\}$ as nested unary operators $\{\ln,\exp\}$. The capability of the UniSymNet is deduced from rigorous theoretical proof, resulting in lower complexity \textcolor{black}{and stronger expressivity.}

\item \textcolor{black}{We propose a Transformer-guided structural search strategy that integrates a sparse label encoding scheme during Transformer pre-training. The pre-trained Transformer guides the structural discovery of UniSymNet, introducing a learning-from-experience mechanism.}


\item We design objective-specific optimization strategies to learn the parameters of the symbolic network and analyze the strengths of different optimization methods.  

\end{itemize}
The remainder of this paper is organized as follows: Section \ref{sec: related Work} reviews recent works in SR. Section \ref{sec: methods} details the framework of UniSymNet, including its core mechanisms. Sections \ref{sec: experiments} and \ref{sec: results} describe the experimental settings and detailed results, respectively. Section \ref{sec: Discussion} discusses \textcolor{black}{$\Psi$ representation and }the specific differences between UniSymNet and symbolic tree representations. Section \ref{sec: conclusion} summarizes \textcolor{black}{the} contributions of our method and outlines future research directions.
\section{Related Works}
\label{sec: related Work}

\subsection{\mbox{Genetic Programming for Symbolic Regression}} 

\textcolor{black}{Traditional Symbolic Regression (SR) methods primarily rely on genetic algorithms \citep{forrest1993genetic}, particularly Genetic Programming (GP) \citep{koza1994genetic}.} GP-based methods iteratively evolve symbolic trees through selection, crossover, and mutation, optimizing their fit to observed data. Subsequent advancements, such as the development of Eureqa \citep{dubvcakova2011eureqa}, improved the scalability and efficiency of GP-based SR. However, these methods\citep{lalearning,virgolin2021improving} still face challenges, including high computational costs, sensitivity to hyperparameters, and difficulties in scaling well to large datasets or high-dimensional problems. 

\subsection{Neural Networks for Symbolic Regression} 

In more recent studies, neural networks have been integrated with Symbolic Regression tasks, primarily to enhance the search for mathematical expressions. For example, AIFeynman \citep{udrescu2020ai} used fully connected neural networks to identify simplified properties of underlying function forms (e.g., multiplicative separability), effectively narrowing down the search space by decomposing complex problems into smaller subproblems. Similarly, DSR \citep{petersen2020deep} utilized a recurrent neural network (RNN) to generate symbolic trees, refining its expression generation based on a reward signal that adjusts the likelihood of generating better-fitting expressions. \(\Phi\)-SO\citep{tenachi2023deep} extended this method by embedding unit constraints to incorporate physical information into the symbolic framework. 

Unlike prior methods that employ neural networks to accelerate symbolic search, recent methods explore neural networks as novel frameworks for constructing interpretable symbolic representations. 
\textcolor{black}{EQL \citep{martius2017extrapolation} and EQL$^{\div}$ \citep{sahoo2018learning} are representative examples of these recent methods and also belong to the EQL-type methods discussed in Section \ref{sec:introduction}.} 

EQL transformed fully connected networks into symbolic networks by using operators such as {$\sin$, $cos$, $\times$} as activation functions, but it didn't introduce $\div$ operator. Therefore, in EQL\(^{\div}\), the nonlinear binary operators \(\{\times, \div\}\) are incorporated into the network through the architectural design illustrated in Fig. \ref{fig: The overall architecture of EQL.}.
\begin{figure}[!ht]
    \begin{center}
    \resizebox{\linewidth}{!}{
    \begin{tikzpicture}[shorten >=1pt,->, node distance=1.5cm and 1cm]
        \node[defnode] (I) at (0, +0.75) {$x$};
        \foreach \i/\name in {0/0, 1/1, 2/\cdots, 3/n-1,4/n}
            \node[inputnode] (I\i) at (0, -\i){} ;
        \node[defnode] (M) at (2, +0.75) {$y^{(1)}$};
        \foreach \i in {0, 1, 2, 3, 4}
            \node[middlenode] (M1\i) at (2, -\i){} ;
        \node[defnode] (H) at (3, +0.75) {$z^{(1)}$};
        \foreach \i/\name in {0/$op^{(1)}_1$, 1/$op^{(1)}_2$, 2/$\cdots$, 3/$op^{(1)}_{d_1}$}
            \node[hiddennode] (H1\i) at (3, -\i-0.5) {\name};
        \node[defnode] (M) at (5, +0.75) {$y^{(2)}$};
        \foreach \i in {0, 1, 2, 3, 4}
            \node[middlenode] (M2\i) at (5, -\i){} ;
        \node[defnode] (H) at (6, +0.75) {$z^{(2)}$};
        \foreach \i/\name in {0/$op^{(2)}_1$, 1/$op^{(2)}_2$, 2/$\cdots$, 3/$op^{(2)}_{d_2}$}
            \node[hiddennode] (H2\i) at (6, -\i-0.5) {\name};
        \node[defnode] (M) at (7.5, +0.75) {$y^{(L-1)}$};
        \foreach \i in {0, 1, 2, 3, 4}
            \node[middlenode] (M3\i) at (7.5, -\i) {};
        \node[defnode] (L) at (8.75, +0.75) {$z^{(L-1)}$};
        \foreach \i/\name in {0/$op^{(l)}_1$, 1/$op^{(l)}_2$, 2/$\cdots$, 3/$op^{(l)}_{d_l}$}
            \node[hiddennode] (H3\i) at (8.75, -\i-0.5){\name};
        \node[defnode] (M) at (10.35, -0.85) {$y^{(L)}$};

        \node[hiddennode] (O) at (10.25, -1.5){} ;
        \node[hiddennode] (O1) at (10.25, -2.25){} ;
        \node[defnode] (M) at (11.25, -0.85) {$z^{(L)}$};
        \node[hiddennode] (1) at (11.25, -2) {$\div$};
        \node[outputnode] (2) at (12.25, -2) {};

        \foreach \i in {0,1,2,3,4}
            \foreach \j in {0,1,2,3,4}
                \draw[signal] (I\i) -- (M1\j);
        \foreach \i in {0,1,2,3}
            \draw[signal] (M1\i) -- (H1\i);
        \foreach \i in {4}
            \draw[signal] (M1\i) -- (H13);
        \node[defnode] at (6.75, 0) {$\cdots$};
        \node[defnode] at (6.75, -1) {$\cdots$};
        \node[defnode] at (6.75, -3) {$\cdots$};
        \node[defnode] at (6.75, -2) {$\cdots$};
        \node[defnode] at (12.25, -1.5) {$y$};
        \foreach \i in {0,1,2,3}
            \foreach \j in {0,1,2,3,4}
                \draw[signal] (H1\i) -- (M2\j);
        \foreach \i in {0,1,2,3}
            \draw[signal] (M2\i) -- (H2\i);
         \foreach \i in {4}
            \draw[signal] (M2\i) -- (H23);   
        \foreach \i in {0,1,2,3}
            \draw[signal] (M3\i) -- (H3\i);
        \foreach \i in {4}
            \draw[signal] (M3\i) -- (H33); 
        \foreach \i in {0,1,2,3}
            \draw[signal] (H3\i) -- (O);
        \foreach \i in {0,1,2,3}
            \draw[signal] (H3\i) -- (O1);
        \draw[signal] (O) -- (1);
        \draw[signal] (O1) -- (1);
        \draw[signal] (1) -- (2);
    \end{tikzpicture}
    }
    \centering
    \end{center}
    \caption{The overall architecture of EQL$^\div$.}
    \label{fig: The overall architecture of EQL.}
\end{figure}
In such symbolic networks, the \( l \)-th layer contains \( p_l \) unary operators and \( q_l \) binary operators, where \( q_l = d_l - p_l \) and \( d_l \) denotes the total number of operators in this layer. Since each unary operator consumes one input \textcolor{black}{node, }whereas each binary operator requires two input nodes, the output after affine transformation must provide \( n_l = p_l + 2q_l \) input nodes to satisfy the computational requirements of this layer. This above output state, denoted as \( y^{(l)} \), therefore contains \( n_l \) nodes prior to operator application. Following the application of operators, the resulting vector \( z^{(l)} \) contains \( d_l\) nodes, the forward propagation rules are defined as follows:  
\begin{equation}
\begin{aligned}
    z^{(0)}&=x,\quad x\in\mathbb{R}^{d_0},\\
    {y}^{(l)} &= {W}^{(l)}{z}^{(l-1)} + {b}^{(l)}, \quad l = 1, 2, \dots, L, \\
    z_i^{(l)} &= op_i^{(l)}(y_i^{(l)}), \quad i = 1, 2, \dots, p_l, \\
    z_i^{(l)} &= op_{i+p_l}^{(l)}(y_{p_l+2i-1}^{(l)},y_{p_l+2i}^{(l)}), \quad i = 1,2, \dots, q_l, 
\end{aligned}
\label{eq: EQL forward}
\end{equation}  
where $z^{(l-1)}\in \mathbb{R}^{d_{l-1}}$ is the previous layer's output,  $W^{(l)}\in \mathbb{R}^{n_{l}\times d_{l-1}}$ ,$b^{(l)}\in \mathbb{R}^{n_{l}}$ are the weight matrix and bias vector for the current layer, respectively. In this way, EQL$^\div$ not only has strong fitting and expressive capabilities but also offers better interpretability compared to black-box models. However, \textcolor{black}{EQL-type methods typically constrain the operator \(\mathrm{pow}\) to fixed exponents when used as activation functions, which limits their flexibility in modeling arbitrary polynomial relationships. In addition, because they rely on backpropagation within a fixed neural network architecture, they often generate overly complex expressions and are prone to overfitting.} \textcolor{black}{Compared with EQL-type networks, UniSymNet introduces the \(\Psi\) representation (Equation \ref{eq:binary_unification_constrained}), which represents the operators \(\{\times,\div,\mathrm{pow}\}\) through nested unary operations. This enhances the expressive power of the network while reducing representational complexity. Moreover, UniSymNet leverages a Transformer to guide structure search, enabling the discovery of expressions with lower complexity and alleviating overfitting.}

Meanwhile, DySymNet \citep{li2024neural} employed an RNN controller to sample various symbolic network architectures guided by policy gradients, demonstrating effectiveness in solving high-dimensional problems. \textcolor{black}{However, these methods typically perform symbolic regression from scratch and, as a result, do not benefit from accumulated experience. In contrast, thanks to the Transformer-guided search strategy, UniSymNet does not need to start the learning process from scratch each time. As a result, it not only significantly accelerates the search at test time but also produces more interpretable expressions by reusing prior knowledge, unlike methods such as DySymNet that rely on RL exploration.
}

In addition, recent studies have proposed novel representation methods: ParFam\citep{scholl2025parfam} incorporates residual neural networks with rational function layers to parameterize mathematical \textcolor{black}{expressions,} GraphDSR \citep{liu2025mathematical} models mathematical expressions as directed acyclic graphs (DAGs) and optimizes symbolic constants through graph neural networks (GNNs).

\subsection{Transformer Model for Symbolic Regression}

The Transformer model, introduced by \citep{vaswani2017attention}, represents a pivotal advancement in deep learning. Its remarkable success has driven its extensive adoption across various domains, including SR tasks. \textcolor{black}{Among the earliest Transformer-based SR methods,} NeSymReS\citep{biggio2021neural} was the first to use a large-scale pre-trained Transformer model to discover symbolic trees from input-output pairs, improving performance with more data and computational resources, but it is limited to low-dimensional problems. \textcolor{black}{In parallel,} \textcolor{black}{SymbolicGPT \citep{valipour2021symbolicgpt} converted the point cloud of the dataset into order-invariant vector embeddings and used a GPT model to translate them into mathematical expressions in string form that describe the point sets.} \textcolor{black}{Recently, a growing number of Transformer-based methods have emerged. The first major line of extending Transformer-based SR focuses on end-to-end expression generation. Along this direction,} \textcolor{black}{\citet{kamienny2022end} proposed an end-to-end symbolic regression method using a Transformer model to predict complete mathematical expressions, including constants, achieving fast inference while maintaining high accuracy.} \textcolor{black}{Moreover, SymFormer \citep{vastl2024symformer} introduced a novel constant-encoding scheme and did not restrict the precision of the generated constants compared to the End-to-end method \citep{kamienny2022end}.} \textcolor{black}{The second major line integrates Transformer models with other optimization strategies.} \textcolor{black}{TPSR\citep{shojaee2023transformer} combined pre-trained Transformer models with Monte Carlo Tree Search to optimize program sequence generation by considering non-differentiable performance feedback.} \textcolor{black}{Additionally, some methods extend genetic programming into Transformer-based SR frameworks, with PGGP\citep{han2025transformer} incorporating Transformer-guided initialization and mutation and TSGP\citep{anthes2025transformer} employing a Transformer to generate semantic-aware offspring.} \textcolor{black}{The third direction of research centers on Transformer-based multimodal frameworks. MMSR\citep{li2025mmsr} aligned the data modality with the expression modality and achieved cross-modal fusion through contrastive learning. In addition, several recent studies \citep{li2024visymre,chen2025bootstrapping} further extend this line of research by introducing visual modalities. Most of the above methods encode mathematical expressions as symbolic trees and use them as labels to train Transformer models. In contrast, our work focuses on encoding symbolic networks that provide a more efficient representation.}

\begin{figure*}
	\centering
	\includegraphics[width=\textwidth]{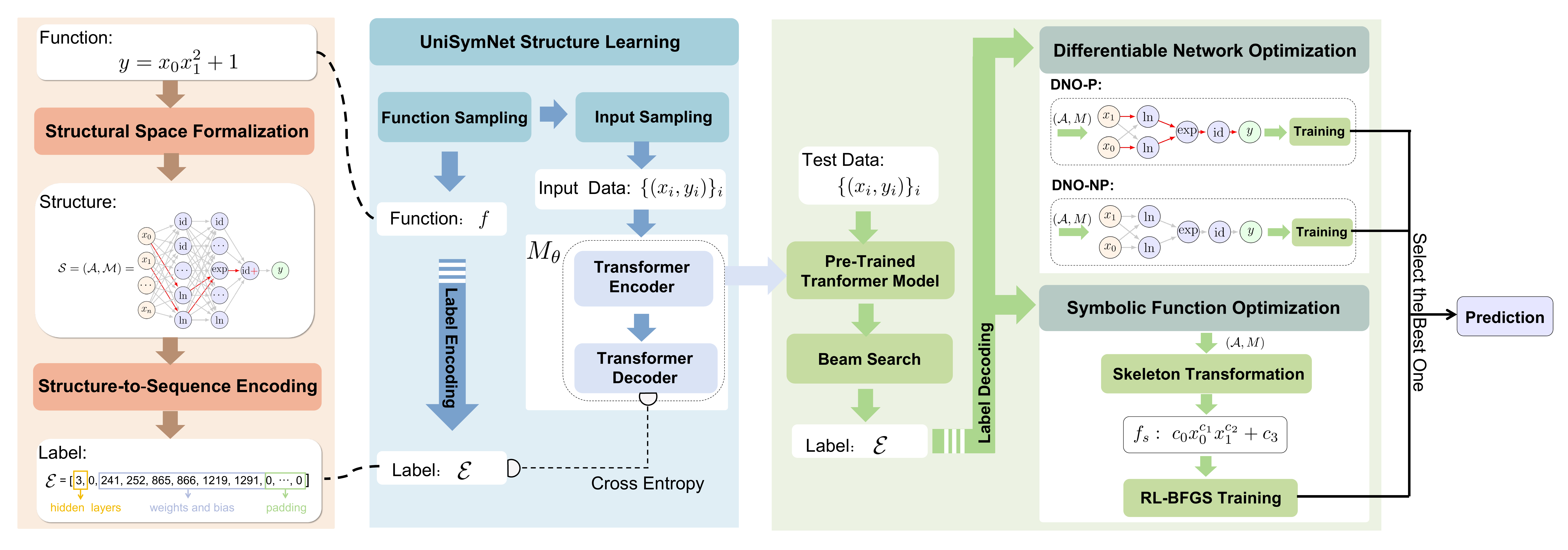}
	\caption{\textcolor{black}{The overall architecture of our method. \textbf{Structural Space Formalization} establishes the mapping between functions and network structures, and \textbf{Structure-to-Sequence Encoding} builds the mapping between network structures and labels. \textbf{UniSymNet Structure Learning} trains a Transformer to learn the relationship between input data and the corresponding network structures. For new test data \(\{x_i, y_i\}_i\), the pre-trained Transformer outputs a predicted sequence \(\mathcal{E}\). Different optimization strategies, \textbf{Symbolic Function Optimization} and \textbf{Differentiable Network Optimization}, then use the Transformer’s guidance to optimize parameters and select the best prediction according to the target objective.}}
	\label{fig: framework}
\end{figure*}
\section{Methods}
\label{sec: methods}

Given a dataset $\{(x_i,y_i)\in \mathbb{R}^{d}\times \mathbb{R}\}_{i=1}^{N}$, SR seeks to identify a mathematical function $f$ that approximates the relationship $y_i \approx f(x_i)$ for all observations. This process inherently involves a dual optimization challenge: (1) exploring the space of potential functional skeletons, and (2) determining optimal parameters for each candidate structure. Formally, we frame SR as a bi-level optimization problem:
\begin{equation}
    \label{eq: bi-level optimization problem}
    s^*, c^* = \mathop{\arg\min}_{s \in \mathcal{S}} \left( \min_{c \in \mathcal{C}_s} \sum_{i=1}^N \left( y_i - f_s(x_i; c) \right)^2 \right),
\end{equation}
where $\mathcal{S}$ represents the space of functional skeletons (each defining a function $f_s(x; c)$), and $\mathcal{C}_s$ denotes the parameter space for structure $s$. The inner optimization identifies the optimal parameters $c$ for a given skeleton, while the outer optimization selects the \textcolor{black}{best} skeleton.


\textcolor{black}{Fig.~\ref{fig: framework} presents the overall framework of our method, where a Transformer model learns structure–data mappings and subsequently guides symbolic equation discovery through parameter optimization.} Specifically, our method solves the above bi-level optimization problem, consisting of outer optimization for structural discovery and inner optimization for parameter estimation. For the outer part, our method has three key components: Structural Space Formalization, Structure-to-Sequence Encoding, and \textcolor{black}{UniSymNet Structure Learning}. For the inner part, we implement objective-specific optimization strategies: Symbolic Function Optimization and Differentiable Network Optimization. Then we will subsequently introduce the above contents.

\subsection{Structural Space Formalization}
\label{sec: Structural Space Formalization}
\begin{figure}[!h]
    \begin{center}
    \resizebox{0.5\textwidth}{!}{
    \begin{tikzpicture}[shorten >=1pt,->, node distance=1.5cm and 1cm]
        \node[defnode] (I) at (0, +0.75) {$x$};
        \foreach \i/\name in {0/0, 1/1, 2/\cdots, 3/n}
            \node[inputnode] (I\i) at (0, -\i){} ;
        \node[defnode] (M) at (2, +0.75) {$y^{(1)}$};
        \foreach \i in {0, 1, 2, 3}
            \node[middlenode] (M1\i) at (2, -\i){} ;
        \node[defnode] (H) at (3, +0.75) {$z^{(1)}$};
        \foreach \i/\name in {0/$op^{(1)}_{1 }$, 1/$op^{(1)}_2$, 2/$\cdots$, 3/$op^{(1)}_{d_1}$}
            \node[hiddennode] (H1\i) at (3, -\i) {\name};
        \node[defnode] (M) at (5, +0.75) {$y^{(2)}$};
        \foreach \i in {0, 1, 2, 3}
            \node[middlenode] (M2\i) at (5, -\i){} ;
        \node[defnode] (H) at (6, +0.75) {$z^{(2)}$};
        \foreach \i/\name in {0/$op^{(2)}_1$, 1/$op^{(2)}_2$, 2/$\cdots$, 3/$op^{(2)}_{d_2}$}
            \node[hiddennode] (H2\i) at (6, -\i) {\name};
        \node[defnode] (M) at (7.5, +0.75) {$y^{(L-1)}$};
        \foreach \i in {0, 1, 2, 3}
            \node[middlenode] (M3\i) at (7.5, -\i) {};
        \node[defnode] (L) at (8.75, +0.75) {$z^{(L-1)}$};
        \foreach \i/\name in {0/$op^{(l)}_1$, 1/$op^{(l)}_2$, 2/$\cdots$, 3/$op^{(l)}_{d_{l}}$}
            \node[hiddennode] (H3\i) at (8.75, -\i){\name};
        \node[defnode] (M) at (10.35, -0.85) {$y^{(L)}$};
        \node[hiddennode] (O) at (10.25, -1.5){} ;
        \node[defnode] (M) at (11.25, -0.85) {$z^{(L)}$};
        \node[hiddennode] (1) at (11.25, -1.5) {$id$};
        \node[outputnode] (2) at (12.25, -1.5) {};

        \foreach \i in {0,1,2,3}
            \foreach \j in {0,1,2,3}
                \draw[signal] (I\i) -- (M1\j);
        \foreach \i in {0,1,2,3}
            \draw[signal] (M1\i) -- (H1\i);
        \node[defnode] at (6.75, 0) {$\cdots$};
        \node[defnode] at (6.75, -1) {$\cdots$};
        \node[defnode] at (6.75, -3) {$\cdots$};
        \node[defnode] at (6.75, -2) {$\cdots$};
        \node[defnode] at (12.25, -1.5) {$y$};
        \foreach \i in {0,1,2,3}
            \foreach \j in {0,1,2,3}
                \draw[signal] (H1\i) -- (M2\j);
        \foreach \i in {0,1,2,3}
            \draw[signal] (M2\i) -- (H2\i);
        \foreach \i in {0,1,2,3}
            \draw[signal] (M3\i) -- (H3\i);

        \foreach \i in {0,1,2,3}
            \draw[signal] (H3\i) -- (O);
        \draw[signal] (O) -- (1);
        \draw[signal] (1) -- (2);
    \end{tikzpicture}
    }
    \centering
    \end{center}
    \caption{The overall architecture of UniSymNet.}
    \label{fig: The overall architecture of UniSymNet.}
\end{figure}

The architectural \(\mathcal{A}\) of UniSymNet is illustrated in Fig. \ref{fig: The overall architecture of UniSymNet.}, where $x$ is the input of network, $L$ denotes the number of hidden layers, and ${op}$ denotes the unary operator selected from the library \(\mathcal{O}_u = \{id, \sin, \cos, \exp, \ln\}\). 
Specifically, \(\mathcal{A}\) refers to depth \(L\), the number of operators per layer \(d_l\), and the set of operators at the \(l\)-th layer \(\{{op}_i^{(l)}\}_{i=1}^{d_l}\).

With the architecture $\mathcal{A}$, the forward propagation of a fully-connected network is governed by the following principles:
\begin{equation}
\begin{aligned}
    z^{(0)}&=x,\quad x\in\mathbb{R}^{d_0},\\
    {y}^{(l)} &= {W}^{(l)}{z}^{(l-1)} + {b}^{(l)}, \quad l = 1, 2, \dots, L, \\
    z_i^{(l)} &= op_i^{(l)}(y_i^{(l)}), \quad i = 1, 2, \dots, d_l, 
\end{aligned}
\label{eq: UniSymNet forward}
\end{equation}  
where $z^{(l-1)}\in \mathbb{R}^{d_{l-1}}$ is the previous layer's output,  $W^{(l)}\in \mathbb{R}^{d_{l}\times d_{l-1}}$ ,$b^{(l)}\in \mathbb{R}^{d_{l}}$ are the weight matrix and bias vector for the current layer, respectively. 

To effectively represent the functional skeleton of a symbolic network structure, we convert a fully-connected neural network into a sparsely-connected one by employing mask matrices,  \(\mathcal{M} = \{M^{(l)}_{w}, M^{(l)}_{b}\}_{l=1}^{L}\), as defined in Equation Equation \ref{eq: Mask matrix}. Specifically, during forward propagation, if $M^{(l)}_{w}(i,j) = 0$, it means that the $i$-th node in the $l$-th layer is not connected to the $j$-th node in the $(l-1)$-th layer. In this way, the sparse structure induced by the mask matrices reduces the overall connectivity of the network.
\begin{equation}
W^{(l)}=W^{(l)} \odot M^{(l)}_w,\  b^{(l)}=b^{(l)} \odot M^{(l)}_b,\quad l = 1, \cdots, L, \\
\label{eq: Mask matrix}
\end{equation}
where $\odot$ denotes the Hadamard product, \(M^{(l)}_w\) has the same dimensions as the weight matrix \(W^{(l)}\), and \(M^{(l)}_b\) has the same dimensions as the bias vector \(b^{(l)}\).

\begin{definition}
    \label{def: The structure of UniSymNet}
The structure of UniSymNet is defined as $S = (\mathcal{A}, \mathcal{M})$, where $\mathcal{A}$ denotes the network architecture and $\mathcal{M}$ is the associated mask matrix. The structural space $\mathcal{S}$ is the set of all combinations of $\mathcal{A}$ and $\mathcal{M}$, i.e., $\mathcal{S} = \{ (\mathcal{A}, \mathcal{M})\}.$
\end{definition}

We define the structure $S$ of UniSymNet and structural space $\mathcal{S}$ in Definition \ref{def: The structure of UniSymNet}. \textcolor{black}{Defining the structural space alone is not sufficient.} It is essential to establish a systematic mapping between the \(\mathcal{S}\) and the functional structure space, so that each structure $S$ of UniSymNet corresponds to a unique functional skeleton $f_s$. Before introducing this mapping, we first present a key representation $\Psi$ introduced in UniSymNet.

Although the binary operators \(\mathcal{O}_b = \{+, -, \times, \div, \mathrm{pow}\}\) are not explicitly defined in the network, the \textcolor{black}{above} network can still efficiently represent these operations.
The linear operators \(\mathcal{O}_b^{lin}  = \{+,-\} \) have been integrated into the affine transformations described in Equation \ref{eq: UniSymNet forward}. And the nonlinear operators \(\mathcal{O}_b^{non-lin}=\{\times, \div, \mathrm{pow}\}\) can be represented by nested compositions of unary operators \(\{\ln, \exp\}\). \textcolor{black}{Moreover, we also represent the $\text{abs}$ operator using nested compositions of the unary operators $\{\ln, \exp\}$, but with a different order of nesting.} Specifically, we establish unified symbolic representation $\Psi$ as follows:
\begin{equation}
\label{eq:binary_unification_constrained}
\begin{aligned}
&\Psi(x_1 \times x_2) \triangleq \exp\left( \ln x_1 + \ln x_2 \right), \\
&\Psi(x_1 \div x_2) \triangleq \exp\left( \ln x_1 - \ln x_2 \right), \\
&\textcolor{black}{\Psi(\mathrm{pow}(x_1,x_2)) \triangleq \exp\left( x_2 \cdot \ln x_1 \right),}\\
&\textcolor{black}{\Psi(\text{abs}(x)) \triangleq \ln(\exp(x)).}
\end{aligned}
\end{equation}
In the UniSymNet, the nonlinear binary operators can be systematically expressed as:
\begin{equation}
    \label{eq: unified trans}
    \Psi(\mathcal \mathcal{O}_b^{\text{non-lin}}) \triangleq \mathcal{O}_u^{\otimes n} \circ \mathcal{O}_b^{\text{lin}},
\end{equation}
where $\mathcal{O}_u^{\otimes n}$ indicates nested compositions of unary operators \(\{\ln, \exp\}\), $\circ$ represents the interaction between two operations. It should be noted that \(\Psi^{-1}\) refers to the inverse transformation, which recovers the original operators by reversing the nesting of unary operators.

\textcolor{black}{The impact of the $\Psi$ representation on network expressivity and complexity highlights its advantages. In the following, we examine how the $\Psi$ representation affects the network’s expressivity and complexity. For comparison, we use UniSymNet (with the $\Psi$ representation) and EQL-type networks (without the $\Psi$ representation) as reference models.
}

\noindent\textcolor{black}{\textbf{Expressivity: }EQL-type networks require predefined activation functions to represent non-integer powers (e.g., $x^{1/2}, x^{1/3}$) and cannot represent equations involving variable exponents (e.g., $x^{y}$), which inherently limits their expressiveness. In contrast, UniSymNet flexibly encodes such functions using nested ${\ln, \exp}$ operations; for example, $2x^{1/2} = 2\exp(\frac{1}{2}\ln x),\ x^y=\exp(y\ln x) $. Moreover, while addition and subtraction are naturally modeled as multi-ary operators via linear mappings, EQL-type networks restrict multiplication and division to binary forms. UniSymNet generalizes these operators into multi-ary versions using ${\ln, \exp}$, such as $\prod_{i=1}^n x_i = \exp(\sum_{i=1}^n \ln x_i)$.}

\noindent\textcolor{black}{\textbf{Complexity: }}In Theorem \ref{thm:exact-neural-rep} and Theorem \ref{thm:composition-generalization}, we show that applying the representation $\Psi$ can reduce depth and node count in a symbolic network under specific conditions, thereby lowering its complexity. The proof of theorems will be given in the Appendix \ref{appendix: proof of Theorem}.
\begin{theorem}
\label{thm:exact-neural-rep}
Let \( P(\mathbf{x}) = \sum_{k=1}^{K} c_k \prod_{i=1}^{d} x_i^{m_{k i}} \), where \( \mathbf{x} \in \mathbb{R}^d \), \( c_k \in \mathbb{R} \), \( m_{k i} \in \mathbb{R} \setminus \{0\} \) and $d,K\in \mathbb{N}$. Suppose that $\exists (k_0, i_0)$, $s.t. \ m_{k_0 i_0} \neq 1$. Then, for any function \( P(\mathbf{x}) \), there exists a UniSymNet architecture that can represent it exactly using depth \( L_1 \) and  \( N_1 \) nodes. Likewise, it can also be represented by an EQL-type network with depth \( L_2 \) and \( N_2 \) nodes.

For \( d \geq 2 \), it holds that \( L_1 \leq L_2 \) and \( N_1 \leq N_2 \). That means, for any such function \( P(\mathbf{x}) \) with \( d \geq 2 \), the UniSymNet architecture achieves representational efficiency that is better than or equal to that of EQL-type networks in terms of both depth and node count.
\end{theorem}
\textcolor{black}{To further illustrate Theorem \ref{thm:exact-neural-rep}, we provide the  example in Fig. \ref{fig:Comparison between EQL-type and UniSymNet.}.}
\begin{figure}[!ht]
    \centering
    \begin{minipage}[t]{0.4\linewidth}
        \centering
        \resizebox{\linewidth}{!}{
        \begin{tikzpicture}
            \foreach \i/\name in {0/$\epsilon$, 1/$h$, 2/$m$, 3/$v$}
                \node[inputnode] (I\i) at (0, -\i){\name} ;
            \foreach \i/\name in {0/$id$, 1/$(\cdot)^2$, 2/$id$, 3/$(\cdot)^2$}
                \node[hiddennode] (H1\i) at (1.5, -\i) {\name};
            \foreach \i/\name in {0/$\times$, 1/$\times$}
                \node[hiddennode] (H2\i) at (3, -\i-1) {\name};
            \foreach \i/\name in {0/$\div$}
                \node[hiddennode] (H3\i) at (4, -\i-1.5){\name};

            \foreach \i in {0,1,2,3}
                \foreach \j in {0,1,2,3}
                    \draw[signal] (I\i) -- (H1\j);
            \foreach \i in {0,1,2,3}
                \foreach \j in {0,1}
                    \draw[signal] (H1\i) -- (H2\j);
            \foreach \i in {0,1}
                \foreach \j in {0}
                    \draw[signal] (H2\i) -- (H3\j);

            \draw[thick, ->, >=latex,  red] (I0) -- (H10); 
            \draw[thick, ->, >=latex,  red] (I1) -- (H11); 
            \draw[thick, ->, >=latex,  red] (I2) -- (H12); 
            \draw[thick, ->, >=latex,  red] (I3) -- (H13);  
            \draw[thick, ->, >=latex,  red] (H10) -- (H20); 
            \draw[thick, ->, >=latex,  red] (H11) -- (H20); 
            \draw[thick, ->, >=latex,  red] (H12) -- (H21); 
            \draw[thick, ->, >=latex,  red] (H13) -- (H21);  
            \draw[thick, ->, >=latex,  red] (H21) -- (H30); 
            \draw[thick, ->, >=latex,  red] (H20) -- (H30);  
        \end{tikzpicture}
        }
        \subcaption{\textcolor{black}{EQL-type}}
    \end{minipage}
    \hspace{0.05\linewidth}
    \begin{minipage}[t]{0.32\linewidth}
        \centering
        \resizebox{\linewidth}{!}{
        \begin{tikzpicture}
            \foreach \i/\name in {0/$\epsilon$, 1/$h$, 2/$m$, 3/$v$}
                \node[inputnode] (I\i) at (0, -\i){\name} ;
            \foreach \i/\name in {0/$\ln$, 1/$\ln$, 2/$\ln$, 3/$\ln$}
                \node[hiddennode] (H1\i) at (1.5, -\i) {\name};
            \foreach \i/\name in {0/$\exp$}
                \node[hiddennode] (H2\i) at (3, -\i-1.5){\name};

            \foreach \i in {0,1,2,3}
                \foreach \j in {0,1,2,3}
                    \draw[signal] (I\i) -- (H1\j);
            \foreach \i in {0,1,2,3}
                \foreach \j in {0}
                    \draw[signal] (H1\i) -- (H2\j);

            \draw[thick, ->, >=latex,  red] (I0) -- (H10); 
            \draw[thick, ->, >=latex,  red] (I1) -- (H11); 
            \draw[thick, ->, >=latex,  red] (I2) -- (H12); 
            \draw[thick, ->, >=latex,  red] (I3) -- (H13);  
            \draw[thick, ->, >=latex,  red] (H10) -- (H20); 
            \draw[thick, ->, >=latex,  red] (H11) -- (H20); 
            \draw[thick, ->, >=latex,  red] (H12) -- (H20); 
            \draw[thick, ->, >=latex,  red] (H13) -- (H20); 
        \end{tikzpicture}
        }
        \subcaption{\textcolor{black}{UniSymNet}}
    \end{minipage}
    \caption{\textcolor{black}{Comparison between EQL-type and UniSymNet for equation $F=\frac{4\pi\epsilon h^2}{m q^2}$. The equation, where $K=1, d=4, c_1=4\pi, m_{11}=1, m_{12}=2, m_{13}=-1, m_{14}=-2$,  satisfies the condition in Theorem \ref{thm:exact-neural-rep}, i.e., $\exists (k_0, i_0)$ such that $m_{k_0 i_0} \neq 1$.}}
    \label{fig:Comparison between EQL-type and UniSymNet.}
\end{figure}
\textcolor{black}{We observe that $L_1=2, N_1=5$, whereas $L_2=3, N_2=7$. Therefore, $L_1 < L_2, N_1 < N_2$. It means that UniSymNet represents the same equation with shallower depth and fewer nodes.} It is worth noting that functions in the form of $P(\mathbf{x})$ are often used in the PDE Discover\citep{long2019pde} task. Furthermore, we present Theorem \ref{thm:composition-generalization}, which extends Theorem \ref{thm:exact-neural-rep}.
\begin{theorem}
\label{thm:composition-generalization}
Under the hypotheses of Theorem~\ref{thm:exact-neural-rep}, consider replacing each \(x_i\) with \(g_i(\mathbf{x})\), where \(g_i(\mathbf{x})\) is an arbitrary function represented using a UniSymNet (\(L_{1i}\) layers, \(N_{1i}\) nodes) or an EQL-type network (\(L_{2i}\) layers, \(N_{2i}\) nodes). If $\max\{L_{1i}\} \le \max\{L_{2i}\}$ and $\sum_iN_{1i} \le \sum_i N_{2i}$, then for \(P(g_1(\mathbf{x}),\ldots,g_d(\mathbf{x}))\), the conclusion of Theorem \ref{thm:exact-neural-rep} still holds.
\end{theorem}
Theorem \ref{thm:exact-neural-rep} and \ref{thm:composition-generalization} show that, for high-dimensional tasks with the functional form $P(\cdot)$, UniSymNet should be preferentially selected due to its superior representational efficiency in both architectural depth and node complexity.

To establish a direct connection between the structure \( S \) and its functional skeleton $f_s$, we introduce the systematic mapping \(\Phi(S):S\xrightarrow{}f_s\) formalized by the Equation \ref{eq: mapping}:  
\begin{equation}
\label{eq: mapping}
\begin{aligned}
    z^{(0)} &= [x_0,x_1\cdots,x_{d_0}]^\top,  \\
    y^{(l)} &= \left( C_w^{(l)} \odot M_w^{(l)} \right) z^{(l-1)} + \left( C_b^{(l)} \odot M_b^{(l)} \right), \ l = 1, \cdots, L, \\
    z^{(l)} &= \left[ op_1^{(l)}(y_1^{(l)}), \dots, op_{d_l}^{(l)}(y_{d_l}^{(l)}) \right]^\top, \\
    f_s &= \Psi^{-1}(z^{(L)}). \\
\end{aligned}
\end{equation}  
where \( C_w^{(l)} \) and \( C_b^{(l)} \) denote symbolic matrices, whose shapes align with the matrices \( M_w^{(l)} \) and \( M_b^{(l)} \). \textcolor{black}{During computation in the mapping, constants serving equivalent roles merged to simplify expressions.} For example, \( c_0(c_1 x_0 + c_2) \) will be transformed into \( c_0 x_0 + c_1 \), where redundant constants are absorbed.  

\subsection{Structure-to-Sequence Encoding}
\label{sec: Structure-to-Sequence Encoding}

In Section~\ref{sec: Structural Space Formalization}, we define the structure of UniSymNet as \( S = (\mathcal{A}, \mathcal{M}) \). How can we learn the \( S \) corresponding to equations from data? Inspired by methods in the NLP field, we encode \( S \) into a sequence compatible with transformer-based frameworks. However, different \( \mathcal{A} \) results in \( \mathcal{M} \) with varying dimensions, which makes it difficult to construct a unified embedding of \( \mathcal{A} \) and \( \mathcal{M} \). To address this issue, we partially fix \(\mathcal{A}\) and \textcolor{black}{sample} sub-networks via sparsified \(\mathcal{M}\), masking redundant components. Specifically, we adopt the following strategies:  
\begin{enumerate}[label=(\roman*)]
     \item Each layer \textcolor{black}{is assigned} a sequence of unary operators in the set \(\mathcal{O}_u\). Specifically, each operator in \(\mathcal{O}_u\) \textcolor{black}{repeats} \( m \) times per layer, resulting in \( 5m \) units per layer. And the operators are arranged in the fixed order of $\mathcal{O}_u$. For example, with \( m = 2 \), the operators are in order: \( 2 \times \text{id} \), \( 2 \times \sin \), \( 2 \times \cos \), \( 2 \times \exp \), and \( 2 \times \ln \).  
    \item  The width (i.e., number of nodes per layer) is \textcolor{black}{fixed} uniformly to \( 5m \) for all layers. The depth $L$ is limited to a maximum value \( L_{\text{max}} \), allowing flexibility in depth while ensuring computational tractability.  
    \item The sparse mask matrix set \(\mathcal{M}\), governed by the forward propagation rules in Equation~\ref{eq: mapping}, prunes redundant connections while preserving the predefined operator sequence and width.  
\end{enumerate}

Therefore, the labels consist of three components: the depth $L$, mask matrices $\{M_w^{(l)}\}_{l=1}^{L}$ and $\{M_b^{(l)}\}_{l=1}^{L}$. The depth is designated as the first component of the label, followed by a separator \( 0 \) as the second component. Subsequently, all matrices in  $\{M_w^{(l)}\}_{l=1}^{L}$ and $\{M_b^{(l)}\}_{l=1}^{L}$ are flattened and concatenated to form a unified vector. The label $\mathcal{L}$ corresponding to the structure ${S}$ can be expressed as:
$$
\begin{aligned}
E = \big[\, &L,\ 0,\ M_w^{(1)}(1,1),\ M_w^{(1)}(1,2),\ \cdots,\ M_w^{(1)}(5m,d),\ \cdots, \\
           &M_w^{(l)}(i,j),\ \cdots,\ M_w^{(L)}(1,1),\ \cdots,\ M_w^{(L)}(1,5m), \\
           &M_b^{(1)}(1),\ \cdots,\ M_b^{(1)}(5m),\ \cdots,\ M_b^{(l)}(i),\ \cdots,\ M_b^{(L)}(1) \,\big].
\end{aligned}
$$
where \( M_w^{(l)}(i,j) \) denotes the element at position \((i,j)\) in the mask matrix of the \(l\)-th layer. It is important to note that, since the width of the \(L\)-th layer is 1, both \( M_w^{(L)} \) and \( M_b^{(L)} \) have only one row. To effectively shorten the length of the label while retaining the essential structural information of UniSymNet, we implement a transformation $\mathcal{T}(E):E\rightarrow \mathcal{E}$ defined by the following operation:
\begin{equation}
\label{eq:label_trans}
\begin{split}
\mathcal{I} &= \left\{\, i \,\middle|\, i \geq 3,\ E_i = 1 \,\right\} = \{ i_1, i_2, \dots, i_{|\mathcal{I}|} \}, \\
\mathcal{E} &= \left[ L,\ 0,\ i_1 - 2,\ i_2 - 2,\ \dots,\ i_{|\mathcal{I}|} - 2 \right].
\end{split}
\end{equation}
where $\mathcal{I}$ represents the indices of all elements starting from the third position (\(i \geq 3\)) where the value equals \(1\), \(i - 2\) applies an index offset adjustment to each qualifying index. \textcolor{black}{Thus, the structure \( S \) is transformed into a sequence \( \mathcal{E} \).}

\subsection{\textcolor{black}{UniSymNet Structure Learning}}
\label{sec: UniSymNet Structure Learning}
To implement the outer optimization in Equation \ref{eq: bi-level optimization problem}, we use a parameterized set-to-sequence model \( M_\theta \) that maps a collection of input-output pairs \( \mathcal{D} = \{(x_i,y_i)\}_{i=1}^N \) to the symbolic sequence \( \mathcal{E} \) defining the structure $S$ of UniSymNet. The model architecture adopts an encoder-decoder framework: the encoder projects the dataset \( \mathcal{D} \) into a latent space through permutation-invariant operations, producing a fixed-dimensional representation \( z \) that encapsulates the input-output relational features; subsequently, the decoder autoregressively generates the target sequence \( \mathcal{E} \) by computing conditional probability distributions \( \mathcal{P}(\mathcal{E}_{k+1} | \mathcal{E}_{1:k}, z) \) at each step, where the next token prediction \textcolor{black}{depends} on both the latent code \( z \) and the preceding token sequence \( \mathcal{E}_{1:k} \). We first pre-train the model on large synthetic datasets of varied math expressions to improve generalization, then adapt it to specific SR tasks. The \textcolor{black}{\textbf{UniSymNet Structure Learning} part of} Fig. \ref{fig: framework} describes the process of pre-training \(M_\theta \), which mainly includes data generation and training.

\subsubsection{Data Generation}
Training data consists of input data $\mathcal{D}$ and label sequence $\mathcal{E}$. Data generation follows three main steps.

\noindent\textbf{Function Sampling:} We follow the method proposed by \citet{lampledeep} to sample functions \( f \) and generate random trees with mathematical operators as internal nodes and variables or constants as leaves. The process is as follows:
\begin{enumerate}[label=(\roman*)]
    \item Constructing the Binary Tree:  The input dimension \( d \) is drawn from a uniform distribution \( \mathcal{U}\{1, d_{\text{max}}\} \), and the number of binary operators \( b_{\text{op}} \) is sampled from \( \mathcal{U}\{d-1, b_{\text{max}}\} \). The \( b_{\text{op}} \) binary operators are then selected from the operator library \textcolor{black}{\( \mathcal{O}_b'=\{+,-,\times,\div\} \)}.
    \item Inserting Unary Operators:  The number of unary operators \( u_{\text{op}} \) is sampled from \( \mathcal{U}\{0, u_{\text{max}}\} \), and \( u_{\text{op}} \) unary operators are chosen randomly from \textcolor{black}{\( \mathcal{O}_u'=\{\ln,\exp,\sin,\cos,\text{sqrt},\text{inv},\text{abs},\text{square}\} \)} and inserted into the tree.
    \item Applying Affine Transformations:  Affine transformations \textcolor{black}{apply to} selected variables \( \{x_j\} \) and unary operators \( \{u_j\} \). Each variable \( x_j \) is transformed to \( w x_j + b \), and each unary operator \( u_j \) becomes \( w u_j + b \), where the parameters \( (w, b) \) are sampled from the distribution \( \mathcal{D}_{\text{aff}} \). This step simulates the linear transformations in UniSymNet.
\end{enumerate}
Then, we get the function \( f \). The first \(d\) variables are sampled in ascending order to ensure the correct input dimensionality. This method avoids the appearance of functions with missing variables, such as \( x_1 + x_3 \). 

\noindent\textbf{Input Sampling:} We independently sample \(\{x_i\}_{i=1}^{N} \subset \mathbb{R}^d\) in a \(d\)-dimensional space, \textcolor{black}{where the \(j\)-th component \(x_{i,j}\) of each sample is drawn from a uniform distribution \( \mathcal{U}(S_{1,j}, S_{2,j}) \). } \textcolor{black}{The interval boundaries \(S_{1,j}\) and \(S_{2,j}\) are constrained to \((-10, 10)\) for all \(j\), ensuring that sampling remains within a predefined range.}
The corresponding output values \(\{y_i\}_{i=1}^{N}\) are \textcolor{black}{then} computed by \(y_i = f(x_i)\). During this process, each \(x_i\) is padded with zeros to match a fixed feature dimension \(d_{\text{max}}\). \textcolor{black}{Any sample \((x_i, y_i)\) whose output \(y_i\) contains NaN values is replaced by a zero pair \((0, 0)\).} \textcolor{black}{Finally, to mitigate numerical instability arising from large variations in \(y\), the dataset \(\{(x_i, y_i)\}_{i=1}^{N}\) is transformed from floating-point values into a multi-hot bit representation following the IEEE-754 standard.}

\noindent\textbf{Label Encoding:} For a random function \( f \), the initial step is to identify the structure $S$ of the UniSymNet associated with the functional skeleton $f_s$, detailed in Appendix \ref{appendix: Algorithms for Label Encoding}. Then, we convert $S$ to the label sequence $\mathcal{E}$ as described in Section \ref{sec: Structure-to-Sequence Encoding}. As a final step, \textcolor{black}{all label sequences are padded with zeros to ensure a consistent sequence length across the dataset.} We still face the challenge that the same function can be represented in different forms, resulting in different labels for the same function, such as \(x_0 + \sin(2x_1)\) and \(x_0 + 2\sin(x_1)\cos(x_1)\). Therefore, we merge different labels of equivalent expressions by assigning the label of the shorter sequence. 

\textcolor{black}{It is important to note that we generate training data by first sampling symbolic trees and then converting them into the corresponding UniSymNet structures $S$, rather than directly sampling structures from UniSymNet. This choice is motivated by three factors: sampling symbolic trees makes it easier to ensure the validity of equations, our structure identification algorithm (given in Appendix \ref{appendix: Algorithms for Label Encoding}) guarantees that the resulting $S$ remains compact, and the overall procedure runs more than an order of magnitude faster than direct sampling from UniSymNet. 
Moreover, the tree-first strategy also facilitates comparisons with symbolic tree encoding methods.
}

\subsubsection{Training}
After generating the training dataset, we begin training the Transformer model \( M_{\theta} \). \textcolor{black}{The detailed description of the complete training procedure is provided in Algorithm \ref{alg:pre-training}.}
\begin{algorithm}[h]
\caption{Transformer Model Training}
\label{alg:pre-training}
    \begin{algorithmic}[1]
        \STATE {\bfseries Input:} Transformer model $M_\theta$, batch size $B$, training distribution $\mathcal{P}_{\mathcal{E},D}$, learning rate $\eta$, maximum iterations $K$, similarity tolerance $\epsilon$
        \STATE Initialize $\mathcal{L}_{\text{prev}} \leftarrow \infty$
        \FOR{$k = 1$ {\bfseries to} $K$}
            \STATE $\mathcal{L}_{\text{train}} \leftarrow 0$
            \FOR{$i = 1$ to $B$}
                \STATE $e, (X,Y) \leftarrow$ sample from $\mathcal{P}_{\mathcal{E}, (X,Y)}$
                \STATE $\mathcal{L}_{\text{train}} \leftarrow \mathcal{L}_{\text{train}} - \sum_{k} \log P_{M_{\theta}}(e_{k+1}|e_{1:k}, D)$
            \ENDFOR
            \STATE Update $\theta$ using $\theta \leftarrow \theta - \eta \nabla_{\theta}\mathcal{L}_{\text{train}}$
            \STATE Compute $\mathcal{L}_{\text{val}}$ on the validation set
            \IF{$0 < \mathcal{L}_{\text{prev}} - \mathcal{L}_{\text{val}} < \epsilon$}
                \STATE \textbf{break}
            \ENDIF
            \STATE Update $\mathcal{L}_{\text{prev}} \leftarrow \mathcal{L}_{\text{val}}$
        \ENDFOR
    \end{algorithmic}
\end{algorithm}

Once the model $M_{\theta}$ has been trained, given a test dataset \( D \), we perform beam search \citep{goodfellow2016deep} to obtain a set of candidate sequences \( \mathcal{E} \). 

\subsection{Parameters Optimization Strategies}
\label{sec: Parameters Optimization Strategies}
With the candidate sequence $\mathcal{E}$, we \textcolor{black}{adopt} objective-specific optimization strategies: Symbolic Function Optimization and Differentiable Network Optimization. As for how to select based on the objective, we will provide a detailed discussion in Section \ref{sec: Ablation Study on Parameter Optimization}.

\subsubsection{Symbolic Function Optimization}
\label{sec: Symbolic Function Optimization}
For each candidate sequence $\mathcal{E}$, we transform the structure \( S \) of UniSymNet into functional skeletons \( f_s \) according to Equation~\ref{eq: mapping}. It's worth noting that UniSymNet merges nested $\exp$ and $\ln$ operators into a nonlinear binary operator during skeleton transformation, which helps avoid numerical instability. 

\textcolor{black}{In natural systems, exponents are usually integers or simple rational numbers. To enhance the credibility and interpretability of the discovered equations, we employ a hybrid strategy in which each exponent can either be optimized directly or restricted to a value from a predefined candidate set. Formally, we define the candidate set of replaceable values as $\mathcal{V} = \{1, -1, 0.5, -0.5, 2, -2, 3, -3, 4, -4, 5, *\}$, where the discrete numbers represent fixed exponents and $*$ serves as a special option indicating that the exponent remains continuous and is optimized directly. This hybrid design strikes a balance between flexibility and stability by combining continuous optimization with discrete choices, leading to equations that are both interpretable and robust.} 

However, \textcolor{black}{adding both discrete and continuous choices makes the search space much larger. If we directly enumerate all possible replacements from $\mathcal{V}$, the search space grows exponentially, and the computation becomes very costly.} To address this, we maintain a learnable probability distribution \(P_{\eta}(v_i|\mathcal{V})\) for each exponent position \(v_i\) in the equation, representing the priority of sampling different exponent values. 

To enable differentiable discrete sampling, we adopt the Gumbel-Softmax trick:
\begin{equation}
v_i = \text{Gumbel-Softmax}\left( \frac{\log P_\eta(v_i) + g_i}{\tau} \right), 
\label{gumble-softmax}
\end{equation}
where \(g_i \sim \text{Gumbel}(0,1)\) is a noise term used to perturb the logits, and \(\tau(t) = \tau_0 \exp(-kt/T)\) is a temperature parameter that controls the degree of approximation to one-hot sampling. 
The key challenge is how to guide the sampling of exponents and optimize the parameters effectively. To this end, we adopt the risk-seeking policy gradient method proposed by \citet{petersen2020deep}. 

In this method, multiple rollouts are generated at each iteration, and each rollout represents a symbolic skeleton $f_s'(x_i;c^*)$. In $f'_{s}$, those $v$ are sampled from $P_\eta(\cdot|\mathcal{V})$, and the other constants $c^*$ are calculated by the BFGS\citep{nocedal1999numerical} method to minimize the mean squared error:
\[
  c^* = \arg\min \sum_{i=1}^N (f_s(x_i; c) - y_i)^2
\]
Each rollout has its corresponding reward return, which is defined as follows:
\begin{gather}
    R = {1}/({1 + \mathcal{L}_{MSE}^*}) \\
    \mathcal{L}_{MSE}^* = \sum_{i=1}^N (f_s(x_i; c^*) - y_i)^2
\end{gather}

To focus the optimization on promising candidates, we calculate the \(1 - \varepsilon\) quantile of the reward distribution, denoted by \(R_\varepsilon(\theta)\), and we compute the policy gradient only for samples with rewards above this threshold:
\[
\nabla_\theta J_{\text{risk}} \propto \mathbb{E}_{R \geq R_\varepsilon(\theta)} \left[ R \nabla_\theta \log P_\theta(f_s) \right] + \lambda_{\mathcal{H}} \nabla_\theta \mathcal{H}(P_\theta),
\]
where \(\mathcal{H}(P_\theta)\) denotes the entropy of the policy distribution, and the coefficient \(\lambda_{\mathcal{H}}\) controls the level of exploration. This entropy term prevents premature convergence to suboptimal solutions.

\textcolor{black}{Regarding this method, it may be sensitive to non-convexity and scaling issues. To address this, we compared different initialization strategies (random, grid search, Latin hypercube sampling, and metaheuristics) and loss functions (MSE, Huber, and Quantile). 
The detailed settings and final results of different initialization strategies and loss functions are provided in the Appendix \ref{appendix: Initialization, Loss Functions}. We ultimately adopted Latin hypercube sampling for initialization and MSE as the default loss, using Quantile loss only for the Nguyen dataset.}

\textcolor{black}{In addition, the temperature schedule in Equation \ref{gumble-softmax} balances exploration and exploitation, while the entropy coefficient regulates the degree of stochasticity during optimization. Both factors directly affect symbolic discovery and fitting accuracy, making it necessary to analyze their influence on the symbolic solution rate and the $R^2$ score. To this end, we compared different temperature schedules (constant, linear decay, and exponential decay) and entropy coefficients (0.005, 0.01, and 0.05). The detailed results are provided in the Appendix \ref{appendix: Sensitivity Analysis in SFO method}. Based on these analyses, we ultimately adopted the exponential decay temperature schedule with an entropy coefficient of 0.005.}

\subsubsection{Differentiable Network Optimization}
\label{sec: Differentiable Network Optimization}

For each candidate sequence \(\mathcal{E}\), we generate a corresponding \textcolor{black}{UniSymNet structure \({S}\)}. 
\textcolor{black}{Given UniSymNet structure $S=(\mathcal{A},\mathcal{M})$, two processes follow:}
\begin{enumerate}[label=(\roman*)]
    \color{black}
    \item \textcolor{black}{\textbf{Mandatory:} $\mathcal{M}$  simplifies the network architecture by eliminating unused neural nodes, thereby reducing the unified architecture $\mathcal{A}$ (introduced in Section \ref{sec: Structure-to-Sequence Encoding}) to a lightweight variant $\mathcal{A}'$.} 
    \item \textcolor{black}{\textbf{Optional:} After the mandatory process, the original mask matrices \(\mathcal{M}\) are reduced and updated to \(\mathcal{M}'\). During subsequent differentiable network training, these masks \(\mathcal{M}'\) can be applied in forward propagation to sparsify inter-node connectivity. If applied, they create a sparse architecture; otherwise, the network remains fully connected.} 
\end{enumerate}
\textcolor{black}{As illustrated in Fig. \ref{fig: The process of simplifying UniSymNet's structure.} for the case \(c_0x_0^{c_1}x_1^{c_2}\), these two processes simplifies the network topology while preserving functional equivalence. }
\begin{figure}[!ht]
    \centering
    \includegraphics[width=\linewidth]{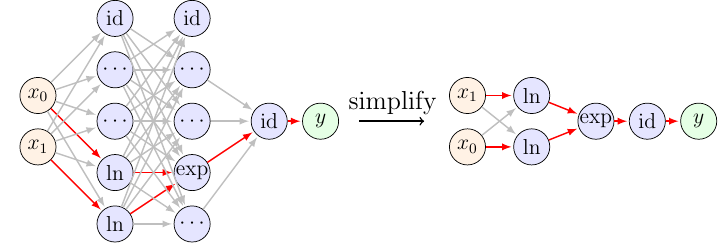}
    \caption{\textcolor{black}{The two processes of simplifying UniSymNet's structure.}}
    \label{fig: The process of simplifying UniSymNet's structure.}
\end{figure}

\textcolor{black}{Once the mandatory lightweight architecture $\mathcal{A}'$ is obtained, the UniSymNet can be trained with a stabilized loss function combining prediction error and activation regularization:}  
\[
\mathcal{L} = \underbrace{\bigl |\min\left(0, x-\theta_{\ln}\right)\bigr | + \max \bigl (0, |x|-\theta_{\exp}\bigr )}_{\text{Activation Regularization } \mathcal{L}_{\text{penalty}}} + \frac{1}{N}\sum_{i=1}^N (\hat{y}_i - y_i)^2,
\]  
where \(\theta_{\ln}\) and \(\theta_{\exp}\) denote the thresholds for the two functions. For the \(\ln\) function, we enforce inputs to lie within the positive domain, while for the \(\exp\) function, inputs are constrained to remain as small as possible to prevent exponential explosion. Thus, the regularized activation functions are defined as:  
\[
\ln_{\text{reg}}(x) = 
\begin{cases} 
0 & x < \theta_{\ln}, \\ 
\ln(|x|+\epsilon) & \text{otherwise},
\end{cases} 
\]
\[
\exp_{\text{reg}}(x) = 
\begin{cases} 
\exp({\theta_{\exp}}) & x \geq \theta_{\exp}, \\ 
\exp(x) & \text{otherwise}.
\end{cases}
\]  
And the \(\theta_{\ln}\) and \(\theta_{\exp}\) are activation thresholds that suppress unstable regions of logarithmic and exponential operators respectively, with \(\epsilon > 0\) ensuring numerical safety. 

\textcolor{black}{As for the optional pruning, the masks $\mathcal{M}'$ provide a concrete pruning strategy, which is implemented during network training as described in Algorithm \ref{alg:pruning strategy}. In this algorithm, the application of pruning is controlled by the \emph{Prune\_flag}.} Since pruning may influence gradient backpropagation and potentially affect model accuracy, its use during training is kept optional. \textcolor{black}{For brevity, we refer to the method with pruning as DNO-P and the method without pruning as DNO-NP. The criteria for selecting whether pruning is applied are discussed in detail in Section \ref{sec: Ablation Study on pruning strategy}.}

{\color{black}
\begin{algorithm}[!ht]
\color{black}
\caption{\textcolor{black}{Differentiable Network Training}}
\label{alg:pruning strategy}
\begin{algorithmic}
       \STATE {\bfseries Input:} network architecture $\mathcal{A}'$, batch size $B$, training data $(X,Y)$, mask matrices $\mathcal{M}'=\{\mathcal{M}'^{(l)}\}_l=\{(M_w'^{(l)},M_b'^{(l)})\}_l$, learning rate $\eta$, max iterations $K$, pruning control flag \emph{Prune\_flag}
        \FOR{$k = 1$ to $K$}
            \STATE Sample $(X^{(0)}, Y^{(0)})$ from $(X,Y)$
            \STATE $\mathcal{L}_{\text{train}} \gets 0$
            \FOR{$l = 1$ to $L$}
                \IF{\emph{Prune\_flag}} 
                    \STATE $Z^{(l)} \gets (W^{(l)}\odot M_w'^{(l)})X^{(l-1)}+b^{(l)} \odot M_b'^{(l)}$
                \ELSE
                    \STATE $Z^{(l)} \gets W^{(l)}X^{(l-1)}+b^{(l)}$
                \ENDIF
                \STATE $X^{(l)} \gets op^{(l)}(Z^{(l)})$
            \ENDFOR
            \STATE $\hat{Y} \gets X^{(L)}$
            \STATE $\mathcal{L}_{\text{train}} \gets \frac{1}{B}( \sum_{j} (\hat{Y}_j - Y^{(0)}_j)^2 +\mathcal{L}_{penalty}$
            \FOR{$l = 1$ to $L$}
                \IF{\emph{Prune\_flag}} 
                    \STATE $\theta_l \gets \theta_l - \eta (\nabla{\theta_l} \mathcal{L}_{\text{train}} \odot \mathcal{M}'^{(l)})$
                \ELSE
                    \STATE $\theta_l \gets \theta_l - \eta \nabla{\theta_l} \mathcal{L}_{\text{train}}$
                \ENDIF
            \ENDFOR
        \ENDFOR
    \end{algorithmic}
\end{algorithm}
}

\section{Experiments}
\label{sec: experiments}
\subsection{Metrics}
\label{sec: Metrics}
\textcolor{black}{In scientific discovery, a simplified and interpretable form of  \( f \) is essential for revealing the governing principles encoded in the data.} Thus, the goal of symbolic regression is not only to achieve high fitting accuracy but also to ensure low complexity and high symbolic solution rate\citep{la2021contemporary}. 

We evaluate \textbf{fitting performance} using $R^2$, a higher \( R^2 \) indicates a higher fitting accuracy of the model, defined as follows: 
\begin{definition}
For \( N_{\text{test}} \) independent test samples, the coefficient of determination \( R^2 \)   is defined as:
\[
R^2 = 1 - \frac{\sum_{i=1}^{N_{\text{test}}} (y_i - \hat{y}_i)^2}{\sum_{i=1}^{N_{\text{test}}} (y_i - \bar{y})^2},
\]
where \( N_{\text{test}} \) is the number of test samples, \( y_i \) is the true value of the \( i \)-th sample, \( \hat{y}_i \) is the predicted value of the \( i \)-th sample, and \( \bar{y} \) is the mean of the true values. 
\end{definition}

We access the \textbf{complexity} using the Definition \ref{def: complexity} and calculate the complexity of the models after simplifying via sympy\citep{meurer2017sympy}. 
\begin{definition}
    \label{def: complexity}
    A model's complexity is defined as the number of mathematical operators, features, and constants in the model, where the mathematical operators are drawn from a set \(\{+, -, \times, \div, \sin, \cos, \exp, \ln, \mathrm{pow}, \text{abs}\}\). 
\end{definition}

For the ground-truth problems, we use the \textbf{Symbolic Solution} defined as Definition \ref{def: Symbolic Solution} to capture models that differ from the true model by a constant or scalar.

\begin{definition}
    \label{def: Symbolic Solution}
    A model $\hat{\phi}(x,\hat{\theta})$ is a Symbolic Solution to a problem with ground-truth model $y = \phi^{*}(x,\theta^{*}) + \epsilon$, if $\hat{\phi}$ does not reduce to a constant, and if either of the following conditions are true: 1) $\phi^{*} - \hat{\phi} = a$; or 2) $\phi^{*} / \hat{\phi} = b, b \neq 0$, for some constants $a$ and $b$.
\end{definition}

The extrapolation ability reflects whether a model can effectively generalize to unseen domains in practice. It is well \textcolor{black}{established} that conventional neural networks, such as black-box models (e.g., MLPs), are often limited in their \textbf{extrapolation ability}. We are interested in evaluating the extrapolation ability of symbolic networks. We \textcolor{black}{access} their extrapolation ability using the following criterion.

\begin{definition}
\label{def: extrapolation ability}
Let \(\mathcal{D}_{\text{train}} = \{(X, y)\}\) be a training set with inputs \(X\) sampled from an interval \([a, b] \subset \mathcal{X}\), where \(\mathcal{X}\) denotes the domain of \(X\). The corresponding test set is constructed with inputs sampled from the extended interval \([a - c, b + c] \cap \mathcal{X}\). The extrapolation ability is evaluated using the \(R^2\) score computed on this test set.
\end{definition}

\subsection{Datasets}
\label{sec: Datasets}
To comprehensively evaluate our model and compare its performance with the baselines, we conduct tests on various publicly available datasets, with the detailed information provided in Appendix \ref{appendix: Description of the Testing Data}. We categorize the datasets into two groups based on their dimension. 

\textbf{Standard Benchmarks:} The feature dimension of the datasets in the Standard Benchmarks satisfies $d\leq 2$. The Standard Benchmarks consist of eight widely used public datasets: Koza\citep{koza1994genetic}, Keijzer\citep{keijzer2003improving}, Nguyen\citep{uy2011semantically}, Nguyen* (Nguyen with constants), R rationals\citep{krawiec2013approximating}, Jin\citep{jin2019bayesian}, Livermore\citep{mundhenk2021symbolic}, Constant.

\textbf{SRBench:} The feature dimension of most datasets in SRBench\citep{la2021contemporary} satisfies \( d \geq 2 \). SRBench used black-box regression problems available in PMLB\citep{olson2017pmlb} and extended PMLB by including 130 datasets with ground-truth model forms. These 130 datasets are sourced from the Feynman Symbolic Regression Database and the ODE-Strogatz repository, derived from first-principles models of physical systems.

\subsection{Baselines}
In our experiments, we evaluate the performance of our method against several baselines, which \textcolor{black}{are grouped into} two main types: In-domain methods and Out-of-domain methods. The comprehensive overview of baselines \textcolor{black}{can be found} in Appendix \ref{Appendix: Comprehensive Overview of Baselines}.

We compare the performance of our method against In-domain SR methods, including the \emph{transformer-based methods}, NeSymRes \citep{biggio2021neural}, End-to-End \citep{kamienny2022end} \textcolor{black}{, and} TPSR \citep{shojaee2023transformer}, as well as the \emph{symbolic network–based method} EQL \citep{lalearning}, \emph{neural guided method} DySymNet \citep{li2024neural}. The Out-of-domain methods refer to the 14 SR methods mentioned in SRBench\citep{la2021contemporary}, with the 14 SR methods primarily consisting of GP-based and DL-based methods.

The computing infrastructure and hyperparameter settings \textcolor{black}{used in the} experiments are \textcolor{black}{provided} in Appendix \ref{appendix: Computing Infrastructure} and Appendix \ref{appendix: Hyperparameter Settings}.

\section{Results}
\label{sec: results}
\subsection{Ablation Studies} 
\label{sec: Ablation Studies}
We performed a series of ablation studies on 200 randomly generated equations, which were not part of the training set, focusing on three key aspects: encoding methods, pruning strategy, and parameter optimization.

\subsubsection{Ablation Study on Encoding methods}
\label{sec: Ablation Study on Encoding methods}
We evaluated the influence of two encoding methods on model performance: \emph{Symbolic Tree method} encoding symbolic trees via prefix traversal\citep{biggio2021neural}, and \emph{UniSymNet Structure method} encoding the UniSymNet structure S as a label in Section \ref{sec: Structure-to-Sequence Encoding}. To eliminate the influence of the parameter optimization method, both methods use standard BFGS optimization without reinforcement learning to guide the optimization direction.
\begin{figure}[!h]
    \centering
    \includegraphics[width=\linewidth]{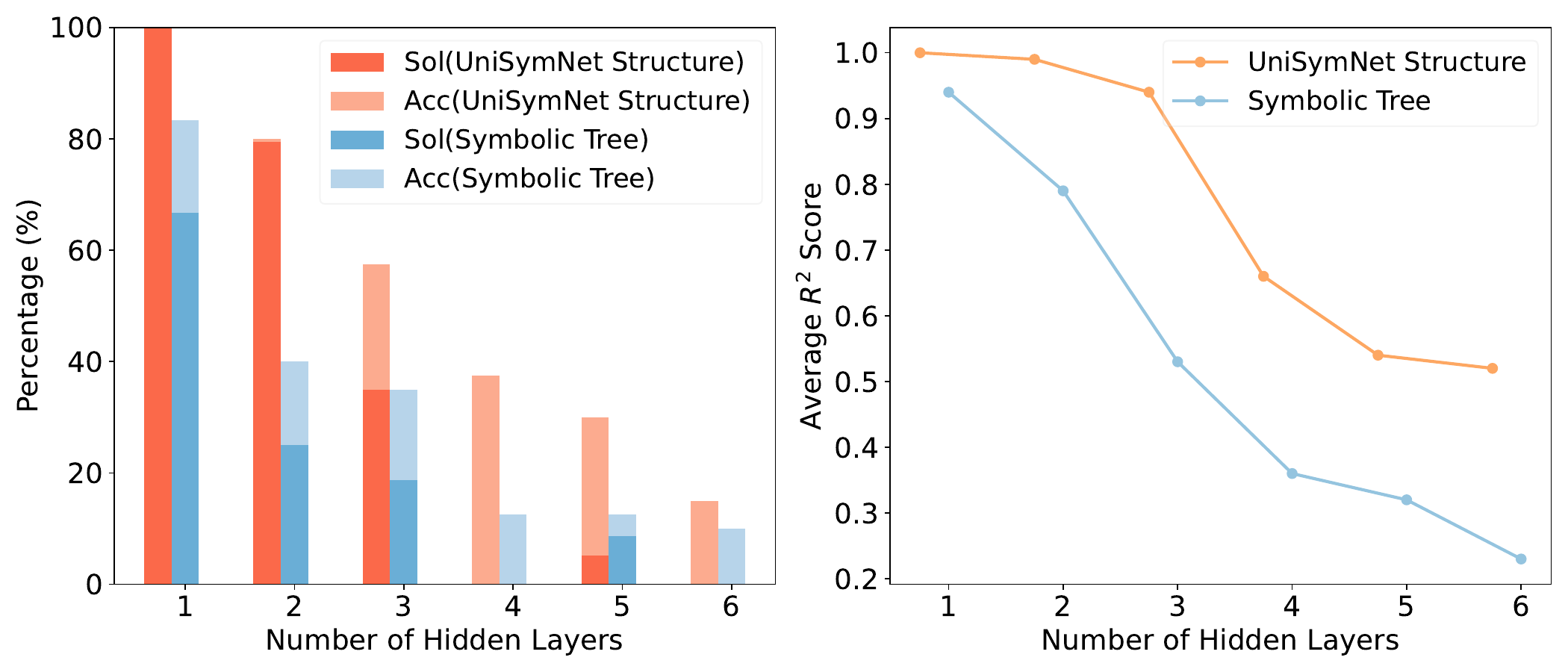}
    \caption{\textcolor{black}{Performance comparison of two encoding methods as the number of hidden layers in the equation varies.\textcolor{black}{The bar chart on the left is defined as subfigure (a), while the line chart on the right is defined as subfigure (b).}}}
    \label{fig: Ablation(bfgs).pdf}
\end{figure}

In Fig. \ref{fig: Ablation(bfgs).pdf}, the x-axis represents the number of hidden layers $L$ in the UniSymNet corresponding to the equations in the test set. A higher $L$ value indicates a more complex equation. The subfigure (a) shows that our encoding method performs better than others in both metrics: the accuracy solution rate at \(R^2>0.99\) \textcolor{black}{(Acc)} and the symbolic solution rate \textcolor{black}{(Sol)}. The subfigure (b) shows \textcolor{black}{the performance of the two methods} on the average \(R^2\) score. Both subfigures reveal that \textcolor{black}{the performance of both models} declines as \(L\) increases, due to the decreasing accuracy of the Transformer’s predictions as the target equations become more complex. However, \emph{UniSymNet Structure method} consistently maintains better performance than \textcolor{black}{other} method, demonstrating the superiority of \emph{UniSymNet Structure method} encoding.

\begin{figure}[!h]
    \centering
    \includegraphics[width=\linewidth]{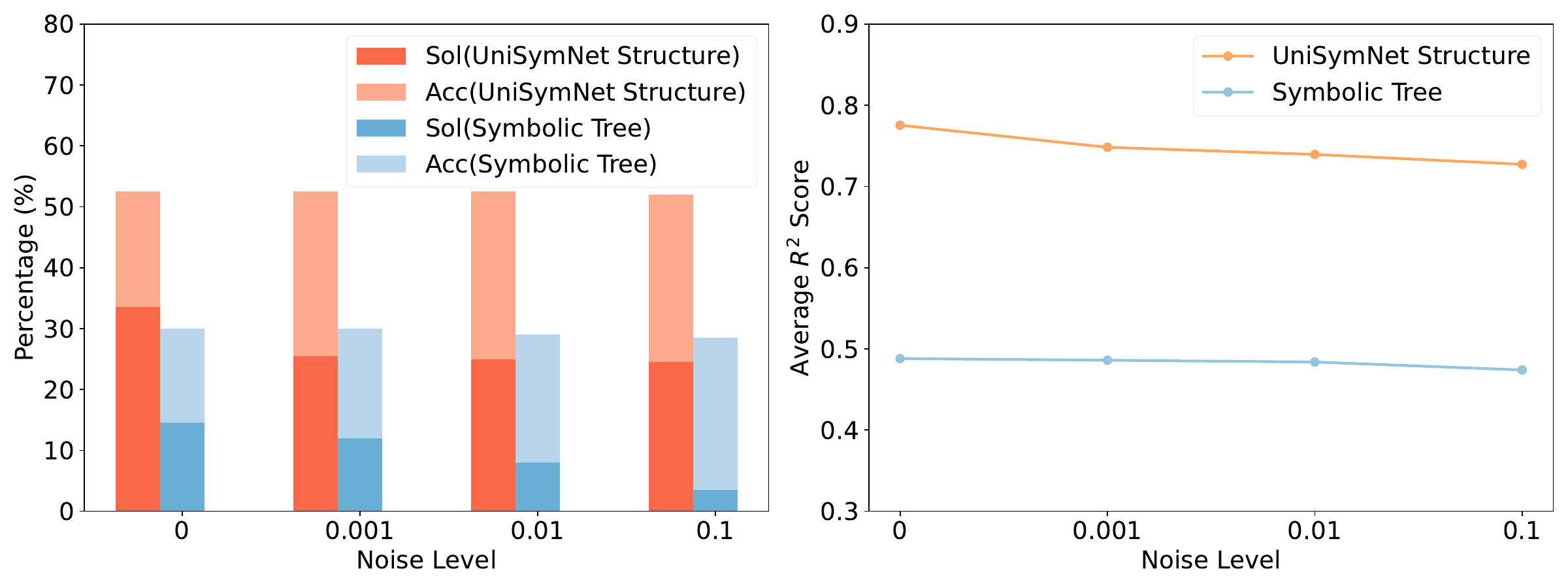}
    \caption{\textcolor{black}{Comparison of different encoding methods across different noise level. The bar chart on the left is defined as subfigure (a), while the line chart on the right is defined as subfigure (b).}}
    \label{fig:Ablation_noise_bfgs}
\end{figure}
\textcolor{black}{To evaluate the noise robustness of UniSymNet Structure encoding, we tested different noise levels and compared the two encoding methods. Subfigure (a) of Fig. \ref{fig:Ablation_noise_bfgs} reports the symbolic solution rate (Sol) and accuracy (Acc) across noise levels. While the fitting accuracy remains relatively stable, the symbolic solution rate exhibits a noticeable decline as noise increases. Notably, UniSymNet Structure maintains a higher symbolic solution rate under noisy conditions, whereas Symbolic Tree experiences a sharper drop. Subfigure (b) of Fig. \ref{fig:Ablation_noise_bfgs} shows the average $R^2$ score, demonstrating that both encoding methods retain strong robustness, with UniSymNet Structure consistently achieving superior performance.}

\subsubsection{Ablation Study on Pruning Strategy}
\label{sec: Ablation Study on pruning strategy}
\textcolor{black}{In Section \ref{sec: Differentiable Network Optimization}, we mentioned that the pruning strategy can be incorporated into the training process in a controlled manner.} To investigate its impact on experimental performance \textcolor{black}{of \emph{Differentiable Network Optimization method}}, we designed three comparison methods. The \textcolor{black}{Default} method sets the UniSymNet structure to a default form ($L=3, m=2$). The \textcolor{black}{DNO-NP} method \textcolor{black}{employs} a Transformer to guide the architecture of UniSymNet and doesn't add \textcolor{black}{a} pruning strategy. The \textcolor{black}{DNO-P} method uses a Transformer to guide the architecture and \textcolor{black}{add} the pruning strategy. 

\begin{figure}[!h]
\centering
\includegraphics[width=0.5\textwidth]{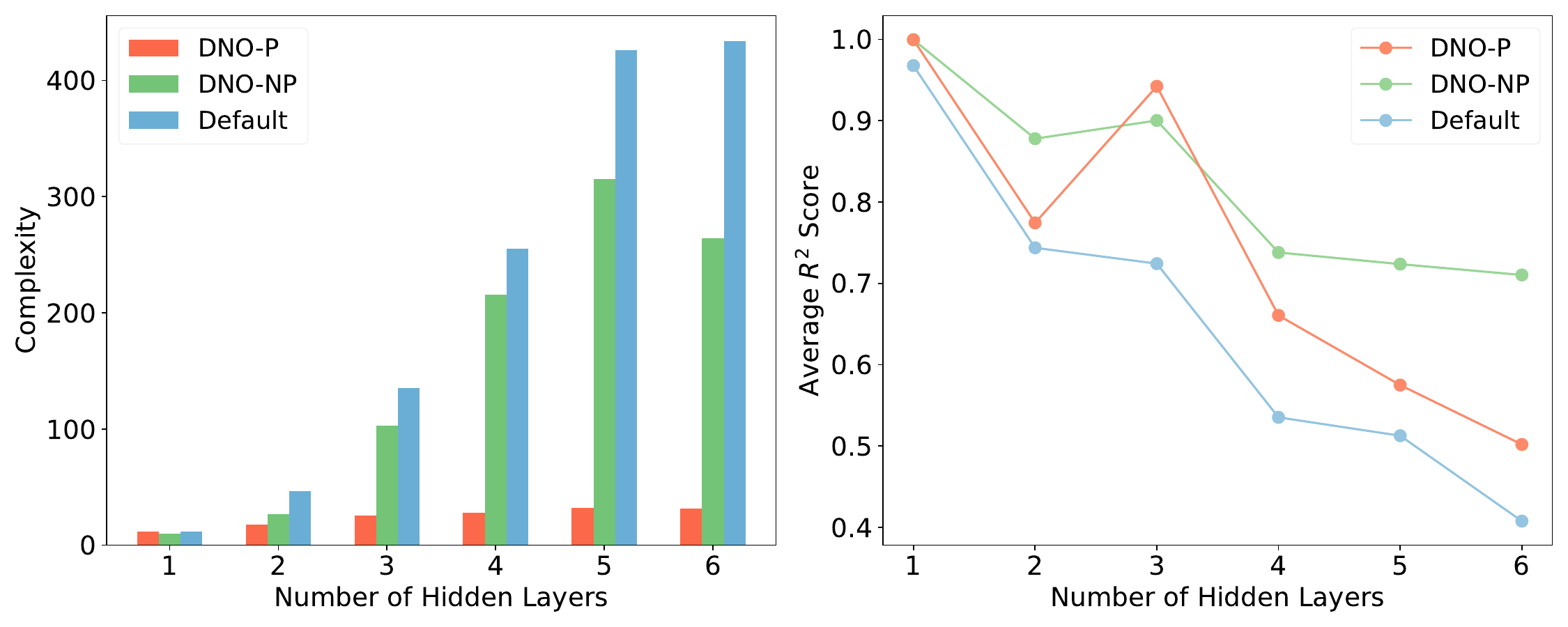}
\caption{\textcolor{black}{Performance comparison of three methods as the number of hidden layers in the equation varies.\textcolor{black}{The bar chart on the left is defined as subfigure (a), while the line chart on the right is defined as subfigure (b).}}}
\label{fig: Ablation(guidance).pdf}
\end{figure}

In Fig. \ref{fig: Ablation(guidance).pdf}, the subfigure (a) compares expression complexity for three methods as the number of hidden layers $L$ increases. We observe that introducing \textcolor{black}{the DNO-P setting (with pruning)} significantly reduces complexity \textcolor{black}{than the other settings. This occurs because only the DNO-P method enforces sparse connectivity, whereas the other two remain essentially fully connected.} The subfigure (b) shows the \textcolor{black}{DNO-NP} method achieves the best fitting performance, followed by the \textcolor{black}{DNO-P} method, while the \textcolor{black}{Default} method performs the worst. Overall, architectural guidance is effective, while pruning guidance doesn't improve fitting performance. This phenomenon can be attributed to two factors. First, the structure predicted by the Transformer is not always absolutely accurate. Since guidance is applied only at the architectural level, the network retains the ability to self-adjust and refine its internal structure during training. Second, the application of pruning strategies may interfere with gradient propagation, affecting final performance. However, the pruning strategy significantly reduces expression complexity.

\subsubsection{Ablation Study on Parameter Optimization}
\label{sec: Ablation Study on Parameter Optimization}

\begin{table*}[!t]
\color{black}
\centering
\caption{\textcolor{black}{Comparison of different optimization strategies across different noise level. The best results with highest Symbolic Solution Rate and Average $R^2$ Score are highlighted.}}
\label{tab:Ablation_noise_optim}
\scalebox{0.88}{
\begin{tabular}{@{}lcccccccccccc@{}}
\toprule
Optimazition Method & \multicolumn{4}{c}{SFO} & \multicolumn{4}{c}{DNO-P}& \multicolumn{4}{c}{DNO-NP} \\
\midrule
Noise Level& 0 & 0.001 & 0.01 & 0.1 & 0 & 0.001 & 0.01 & 0.1 & 0 & 0.001 & 0.01 & 0.1 \\
\cmidrule(lr){1-1}\cmidrule(lr){2-5} \cmidrule(lr){6-9}\cmidrule(lr){10-13}
Symbolic Solution Rate& \textbf{0.3350} & \textbf{0.2450} & \textbf{0.2450} & \textbf{0.2350} & 0.1515 & 0.1480 & 0.1472 & 0.1371 & 0.1224 & 0.1050 & 0.1150 & 0.1100 \\
Average $R^2$ Score & 0.7754 & 0.7592 & 0.7369  & \textbf{0.7386} & 0.7421 & 0.7286 & 0.7278 & 0.7242 & \textbf{0.8216} & \textbf{0.7724} & \textbf{0.7680}  & 0.7343 \\
\bottomrule
\end{tabular}}
\end{table*}

When applying our method, the choice between the \emph{Symbolic Function Optimization method} and the \emph{Differentiable Network Optimization method} depends on the specific objective. As noted in Section \ref{sec: Ablation Study on pruning strategy}, \emph{Differentiable Network Optimization method} itself can be carried out with or without \textcolor{black}{a} pruning strategy. \textcolor{black}{We now turn to a discussion on how to choose among these three methods.}

Because the pruned UniSymNet structure \({S}\) corresponds \textcolor{black}{directly} to the symbolic functional skeleton \(f_s\), their complexities coincide. In other words, pruning produces expression complexity similar to that of the Symbolic Function Optimization method, while skipping pruning leads to higher complexity than both alternatives. Therefore, we focus on symbolic accuracy and fitting performance, excluding complexity from comparison. For simplification, we denote the \emph{Symbolic Function Optimization method} as SFO, the \emph{Differentiable Network Optimization method} with pruning strategy as DNO-P, and the one without pruning strategy as DNO-NP.

\textcolor{black}{In the absence of noise,} Table \ref{tab:Ablation_noise_optim} shows that SFO achieves the highest symbolic solution rate, \textcolor{black}{whereas} DNO-NP shows the best fitting performance. \textcolor{black}{This raises an important question}: why does DNO-P's structure \( S \) exhibit a one-to-one correspondence with SFO's \( f_s \), yet SFO's Symbolic Solution Rate is \textcolor{black}{markedly lower} to DNO-P's? There are two reasons. \textcolor{black}{First}, the BFGS algorithm is more appropriate for models with relatively few parameters, in contrast to neural networks, which typically involve large-scale parameter optimization. \textcolor{black}{Second}, redundant parameters tend to emerge in symbolic networks. For example, symbolic networks are more likely to produce expressions in nested forms such as $w_1(w_0x_0 + b_0)$, even when simpler equivalent forms like $c_0x_0 + c_1$ exist.

\textcolor{black}{In the presence of noise, the results in Table \ref{tab:Ablation_noise_optim} demonstrate that SFO exhibits the strongest robustness, with both its symbolic solution rate and $R^2$ score degrading only marginally as the noise level increases from 0 to 0.1. By contrast, DNO-P and DNO-NP show greater sensitivity, with their fitting accuracy deteriorating more noticeably at higher noise levels. These findings indicate that although DNO-NP can achieve competitive accuracy under noise-free conditions, SFO proves to be substantially more reliable when handling noisy data, making it the most noise-robust strategy among the three.}

Therefore, each method is suited to distinct objectives:
\begin{enumerate}[label=(\roman*)]
\item  If a reliable and automatically differentiable surrogate model is the primary goal, DNO-NP serves as an excellent choice.  
\item If discovering underlying relationships or scientific discoveries within the data is desired, SFO is preferable. \textcolor{black}{Its strong noise robustness further enhances its reliability, making it one of the most suitable methods for SR tasks.}
\item If a simplified and automatically differentiable surrogate model is the primary goal, DNO-P provides a lower complexity surrogate compared to DNO-NP, though with a slight reduction in fitting accuracy.  
\end{enumerate}
\textcolor{black}{It is worth noting that, in contrast to the DNO method, the SFO method is not affected by the domain alterations introduced by the $\Psi$ transformation. A detailed discussion of this issue is provided in Section \ref{sec: domain implications}. As for the subsequent experiments, the choice of optimization methods for different datasets will be detailed in Appendix \ref{appendix: Hyperparameter Settings}.}

\subsection{In-domain Performance}
\label{sec: In-domain performance}
The sets of problems in the Standard Benchmarks have at most two variables and simple ground-truth equations. We tested our method on these benchmarks and compared it with the in-domain baselines through twenty independent runs. \textcolor{black}{In Appendix \ref{appendix: Representative Examples}, Table \ref{tab: Representative examples in Standard_benchmarks.} presents several representative examples of equations discovered from the Standard Benchmark datasets.}

\subsubsection{Fitting Performance and Complexity}
\label{sec: $R^2$ Score and Complexity}

A model that is too simple may fail to capture the underlying relationships within the data, resulting in poor fitting performance. However, an excessively complex model may fit the data well, but it also risks overfitting and lacks interpretability. Achieving a good balance requires a model with relatively low complexity and high fitting accuracy. Thus, we will present a joint comparison of fitting performance and complexity in this section.

\begin{figure}[!ht]
    \centering
    \includegraphics[width=\linewidth]{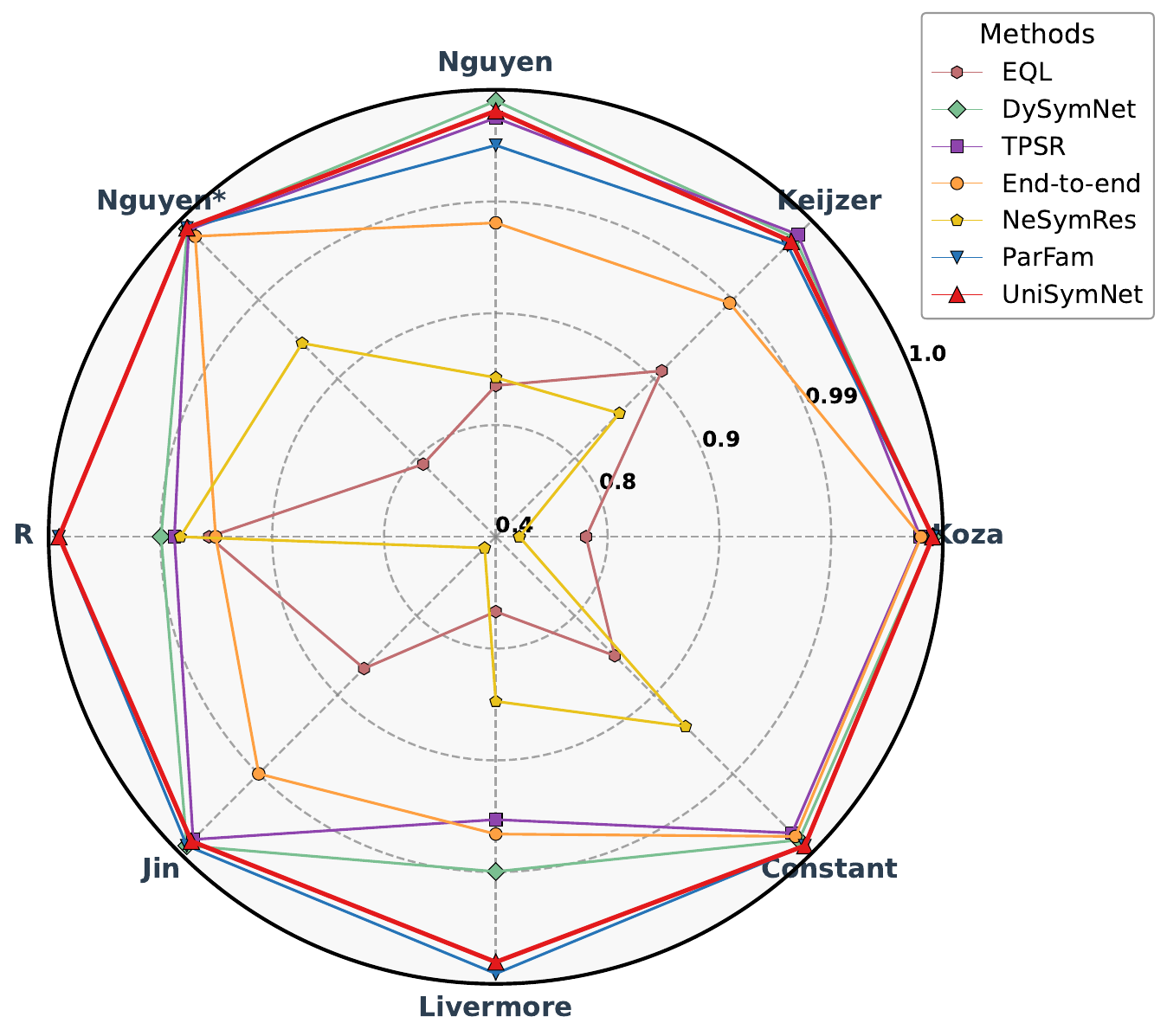}
    \caption{\textcolor{black}{Average $R^2$ score comparison of in-domain methods on Standard Benchmarks.}}
    \label{fig: Average $R^2$ score comparison of in-domain methods.}
\end{figure}

As a first step, we compare the fitting performance separately. Fig. \ref{fig: Average $R^2$ score comparison of in-domain methods.} presents the fitting performance comparison of in-domain methods on the Standard Benchmarks. To enhance clarity, the rings in Fig. \ref{fig: Average $R^2$ score comparison of in-domain methods.} are intentionally spaced unevenly. Based solely on the average $R^2$ score across each dataset, UniSymNet, \textcolor{black}{ParFam}, DySymNet, and TPSR exhibit competitive performance, among which UniSymNet consistently demonstrates superior fitting capability.

\begin{table*}[hb]
\centering
\caption{Fitting accuracy ($R^2>0.99$) rate, complexity, and symbolic solution rate comparison of in-domain methods. \textcolor{black}{We highlight the results with the highest fitting accuracy. Among these, the one with the lowest complexity is additionally marked. For the symbolic solution rate, the result with the highest value is indicated.}}
\label{tab: Complexity comparison of in-domain methods.}
\renewcommand{\arraystretch}{0.8} 
\scalebox{0.9}{
\begin{tabular}{cccccccccc}
\toprule
\textbf{Benchmark} & \textbf{Metrics} & \textbf{EQL} & \textbf{NeSymRes} & \textbf{End-to-End} & \textbf{DySymNet} & \textbf{TPSR} & \textcolor{black}{\textbf{ParFam}} &\textbf{UniSymNet} & \textbf{Ground-Truth} \\
\toprule
\multirow{3}{*}{\textbf{Koza}} & \emph{Com} & 5.00 & 9.00 & 16.85 & 14.00 & 16.83 &\textcolor{black}{23.80}& \textbf{13.50} & 11.00 \\
\cmidrule(lr){2-10}
& \emph{Acc} & 0.0000 & 0.0000 & \textbf{1.0000} & \textbf{1.0000} & \textbf{1.0000} &\textcolor{black}{\textbf{1.0000}}& \textbf{1.0000} & - \\
\cmidrule(lr){2-10}
 & \emph{Sol} & 0.0000 & 0.0000 & 0.0000 & 0.0000 & 0.0000 &\textcolor{black}{0.0000}& \textbf{0.5000} & - \\
\midrule
\multirow{3}{*}{\textbf{Keijzer}} & \emph{Com} & 17.93 & 9.45 & 22.45 & 17.20  & 18.67 &\textcolor{black}{74.37}& \textcolor{black}{\textbf{10.16}} & 10.83 \\
\cmidrule(lr){2-10}
 & \emph{Acc} & 0.4583 & 0.5833 & 0.6708  & 0.9667 & \textbf{1.0000} &\textcolor{black}{0.9762} & \textbf{1.0000} & - \\
\cmidrule(lr){2-10}
 & \emph{Sol} & 0.0000 & 0.0833 & 0.0083 & 0.0000 & 0.1667 &\textcolor{black}{0.3333}& \textbf{0.4167} & - \\
\midrule
\multirow{3}{*}{\textbf{Nguyen}} & \emph{Com} & 15.85 & 9.08 & 22.75 & 24.20 & 15.86 &\textcolor{black}{\textbf{57.83}} & \textcolor{black}{9.25} & 9.20    \\
\cmidrule(lr){2-10}
& \emph{Acc} & 0.2958 & 0.4167 & 0.8458 & 0.8333  & \textbf{0.9167}&\textcolor{black}{\textbf{0.9916}} & 0.9167 &  -\\
\cmidrule(lr){2-10}
 & \emph{Sol} & 0.0000 & 0.1667 & 0.0000 & 0.0000 & 0.0833 &\textcolor{black}{0.1667}& \textcolor{black}{\textbf{0.3333}} & - \\
\midrule
\multirow{3}{*}{\textbf{Nguyen*}} & \emph{Com} & 17.70 & 8.07 & 16.12 & 16.00 & 16.80 &\textcolor{black}{43.91}& \textbf{11.80}  & 9.58 \\
\cmidrule(lr){2-10}
 & \emph{Acc} & 0.5500 & 0.6000 & \textbf{1.0000} & \textbf{1.0000}  & \textbf{1.0000}&\textcolor{black}{\textbf{1.0000}} & \textbf{1.0000} & - \\
 \cmidrule(lr){2-10}
& \emph{Sol} & 0.0000 & 0.2000 & 0.0000 & 0.0000 & 0.0000 &\textcolor{black}{0.1000}& \textbf{0.4000} & - \\
\midrule
\multirow{3}{*}{\textbf{R}} & \emph{Com} & 9.00 & 12.30 & 21.83 & 29.70 & 16.87 &\textcolor{black}{23.92} & \textcolor{black}{\textbf{20.00}}  & 18.67 \\
\cmidrule(lr){2-10}
 & \emph{Acc} & 0.3167 & 0.3333 & 0.6667 & 0.6833 & 0.6833&\textcolor{black}{\textbf{1.0000}} & \textbf{1.0000} & - \\
 \cmidrule(lr){2-10}
& \emph{Sol} & 0.0000 & 0.0000 & 0.0000 & 0.0000 & 0.0000 &\textcolor{black}{0.0000}& \textbf{0.3333} & - \\
\midrule
\multirow{2}{*}{\textbf{Jin}} & \emph{Com} & 17.11 & 8.95 & 28.19 & 59.78 & 19.78&\textcolor{black}{100.92} &  \textcolor{black}{\textbf{15.67}}  & 13.00 \\
\cmidrule(lr){2-10}
 & \emph{Acc} & 0.0167 & 0.1667 & 0.9667 &0.9833  & \textbf{1.0000} &\textcolor{black}{\textbf{1.0000}}& \textbf{1.0000}& - \\
 \cmidrule(lr){2-10}
& \emph{Sol} & 0.0000 & 0.0000 & 0.0000 & 0.0000 & 0.0583&\textcolor{black}{0.3333} & \textbf{0.5000} & - \\
\midrule
\multirow{3}{*}{\textbf{Livermore}} & \emph{Com} & 14.72 & 8.09 & 20.95 & 23.30 & 14.55 &\textcolor{black}{29.39} & \textcolor{black}{\textbf{11.50}} & 9.73 \\
\cmidrule(lr){2-10}
 & \emph{Acc} & 0.2682 & 0.2727 & 0.6614 & 0.7568 & 0.8182 &\textcolor{black}{\textbf{1.0000}}& \textbf{1.0000} & -  \\
 \cmidrule(lr){2-10}
& \emph{Sol} & 0.0000 & 0.0455 & 0.0045 & 0.0000 & 0.0500 &\textcolor{black}{0.2727}& \textbf{0.4091} & - \\
\midrule
\multirow{3}{*}{\textbf{Constant}} & \emph{Com} & 21.47 & 7.43 & 16.30 & 17.11  & {14.50} &\textcolor{black}{28.33}& \textcolor{black}{\textbf{9.63}}  & 7.63 \\
\cmidrule(lr){2-10}
 & \emph{Acc} & 0.5313 & 0.5000 & \textbf{1.0000} & 0.8333 & {0.8750} &\textcolor{black}{\textbf{1.0000}}&  \textbf{1.0000} & - \\
 \cmidrule(lr){2-10}
 & \emph{Sol} & 0.0000 & 0.2500 & 0.0063 & 0.0000 & 0.1250 &\textcolor{black}{0.3750}& \textbf{0.7500} & - \\
\bottomrule
\end{tabular}}
\end{table*}
And to enable a more detailed joint comparison, Table \ref{tab: Complexity comparison of in-domain methods.} provides a comprehensive analysis of each method in terms of fitting accuracy and expression complexity. The fitting accuracy refers to the proportion of 20 independent runs per dataset that satisfy $R^2>0.99$. And the "Ground-truth" column in Table \ref{tab: Complexity comparison of in-domain methods.} represents the true average complexity for each dataset. As shown, UniSymNet strikes a good balance between fitting accuracy and complexity, achieving perfect fitting accuracy with the lowest complexity on most datasets. 


\subsubsection{Symbolic Solution Rate}
\label{sec: Symbolic Solution Rate}

If a predicted expression meets the criteria of Definition \ref{def: Symbolic Solution}, it uncovers the true law behind the data, yielding \(R^2 = 1\) and thus perfect fitting and extrapolation performance. Since this is the ideal outcome for discoveries in the natural sciences, we compare these methods by their symbolic solution rates on the Standard Benchmarks in Table \ref{tab: Complexity comparison of in-domain methods.}. The symbolic solution rate refers to the proportion of 20 independent runs per dataset that satisfy the Definition \ref{def: Symbolic Solution}. According to the results, the three methods: UniSymNet, \textcolor{black}{ParFam}, NeSymRes, exhibit relatively higher solution rates. Among them, UniSymNet consistently achieves the highest symbolic solution rate.

\subsubsection{Computational Efficiency}
\label{sec: Computational Efficiency}
We also consider the time cost of solving the problems. Therefore, we compare the average test time of the in-domain methods on the Standard Benchmarks. 
\begin{figure}[!ht]
    \centering
    \includegraphics[width=
    \linewidth]{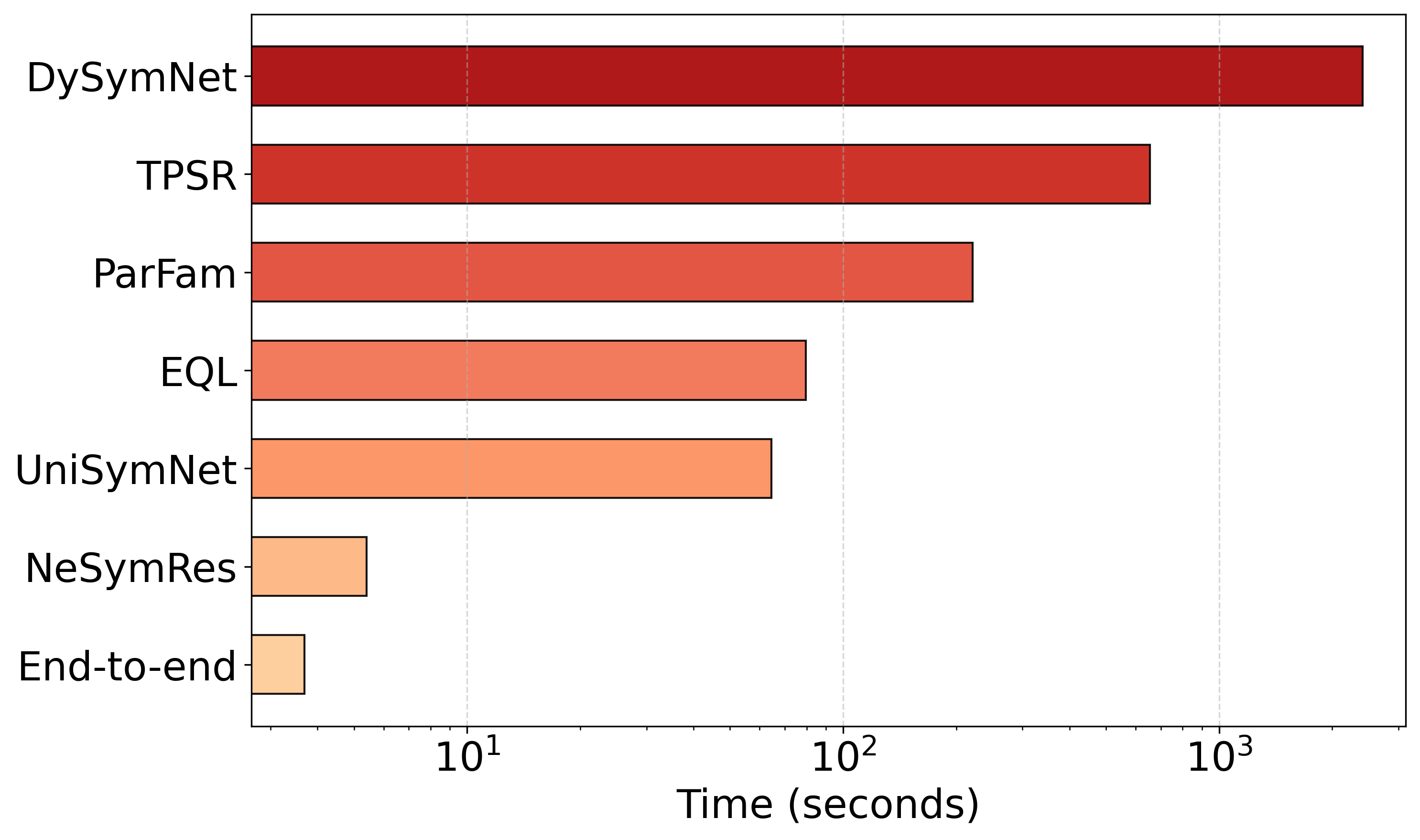}
    \caption{\textcolor{black}{Average test time comparison of in-domain methods.}}
    \label{fig: Test time comparison of in-domain methods.}
\end{figure}
Fig. \ref{fig: Test time comparison of in-domain methods.} compares the test time of various in-domain methods. The learning-from-experience methods: End-to-end, NeSymRes\textcolor{black}{, and} UniSymNet achieve the shortest test time, benefiting from the strong guidance provided by pre-trained models, and other methods generally exhibit longer test time.

\subsubsection{Extrapolation Ability}
\label{sec: Extrapolation Ability}
\begin{figure}[!h]
  \centering
  \begin{minipage}{0.23\textwidth}
  \centering
    \includegraphics[width=\textwidth]{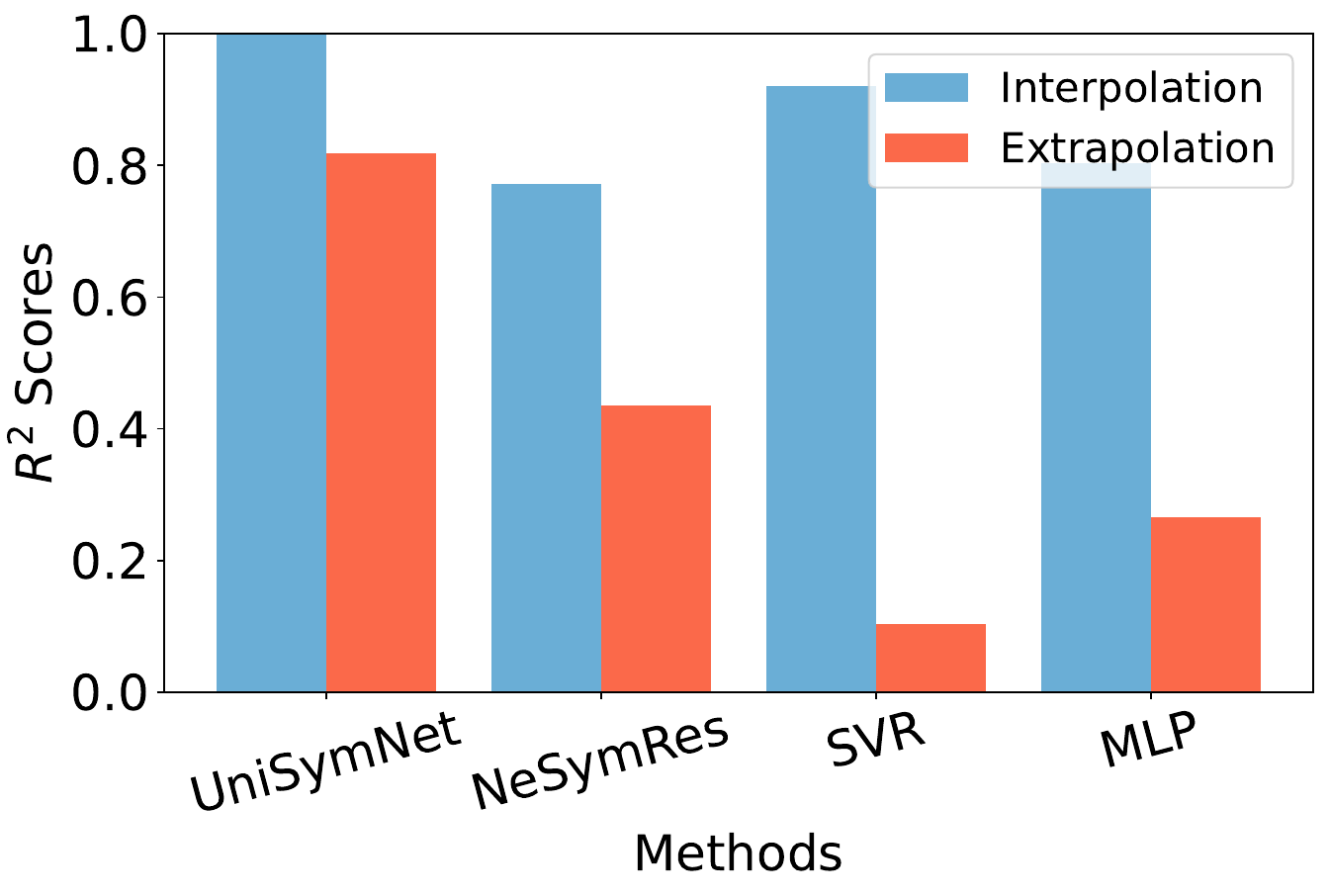} \\
    Keijzer
  \end{minipage}
  \begin{minipage}{0.23\textwidth}
  \centering
    \includegraphics[width=\textwidth]{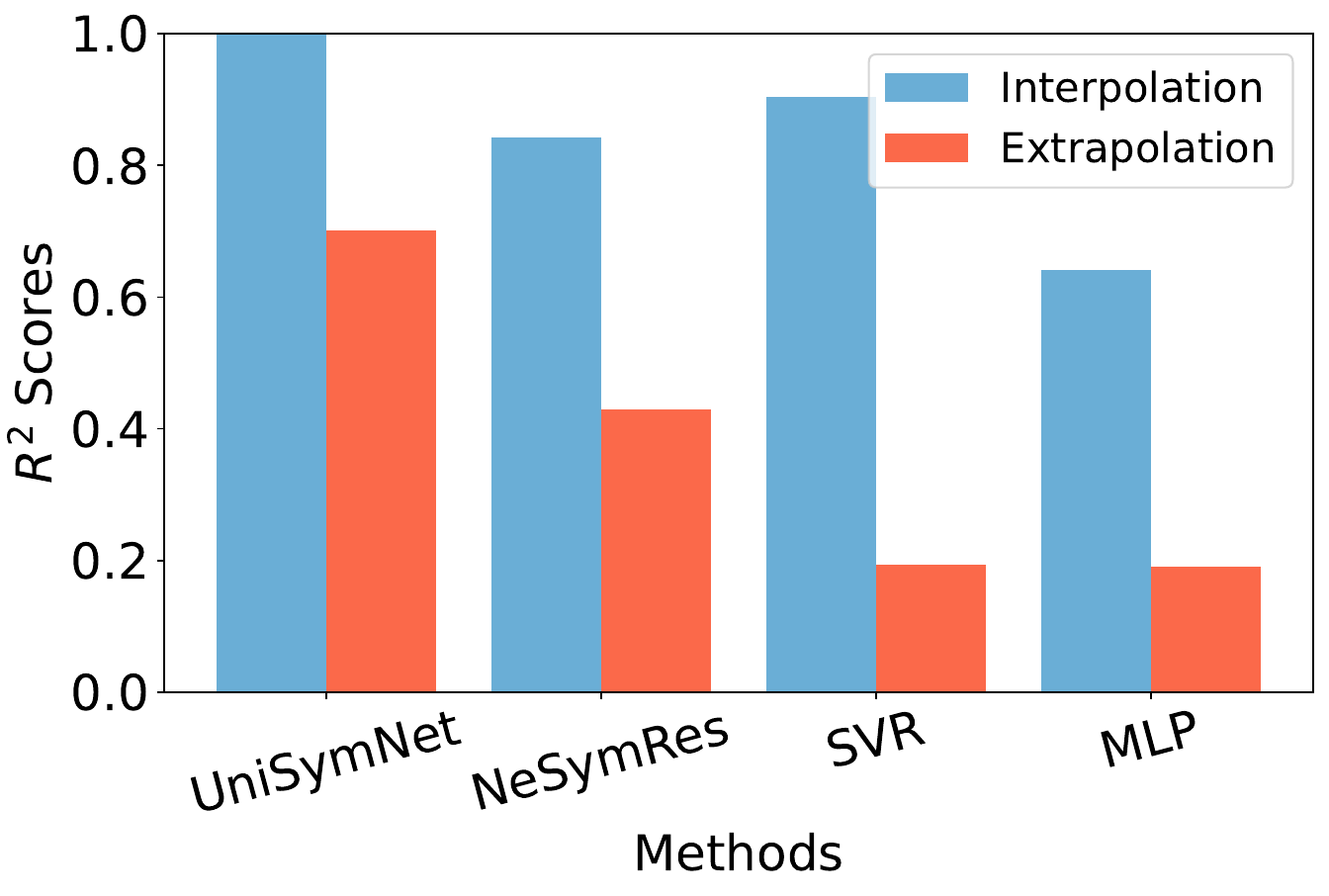}\\
    Nguyen
  \end{minipage}

  
  \begin{minipage}{0.23\textwidth}
  \centering
    \includegraphics[width=\textwidth]{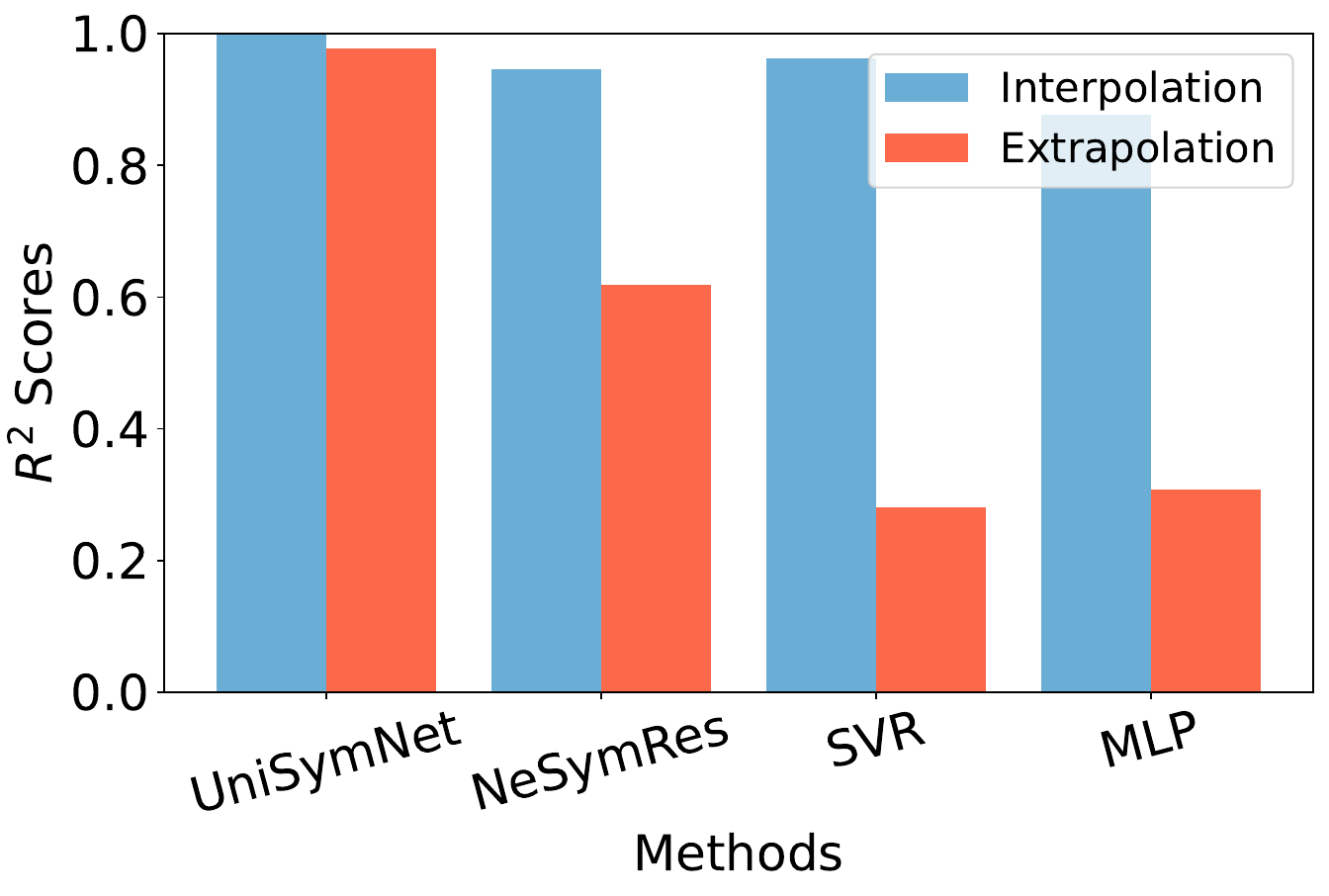}\\
    Constant
  \end{minipage}
  \begin{minipage}{0.23\textwidth}
  \centering
    \includegraphics[width=\textwidth]{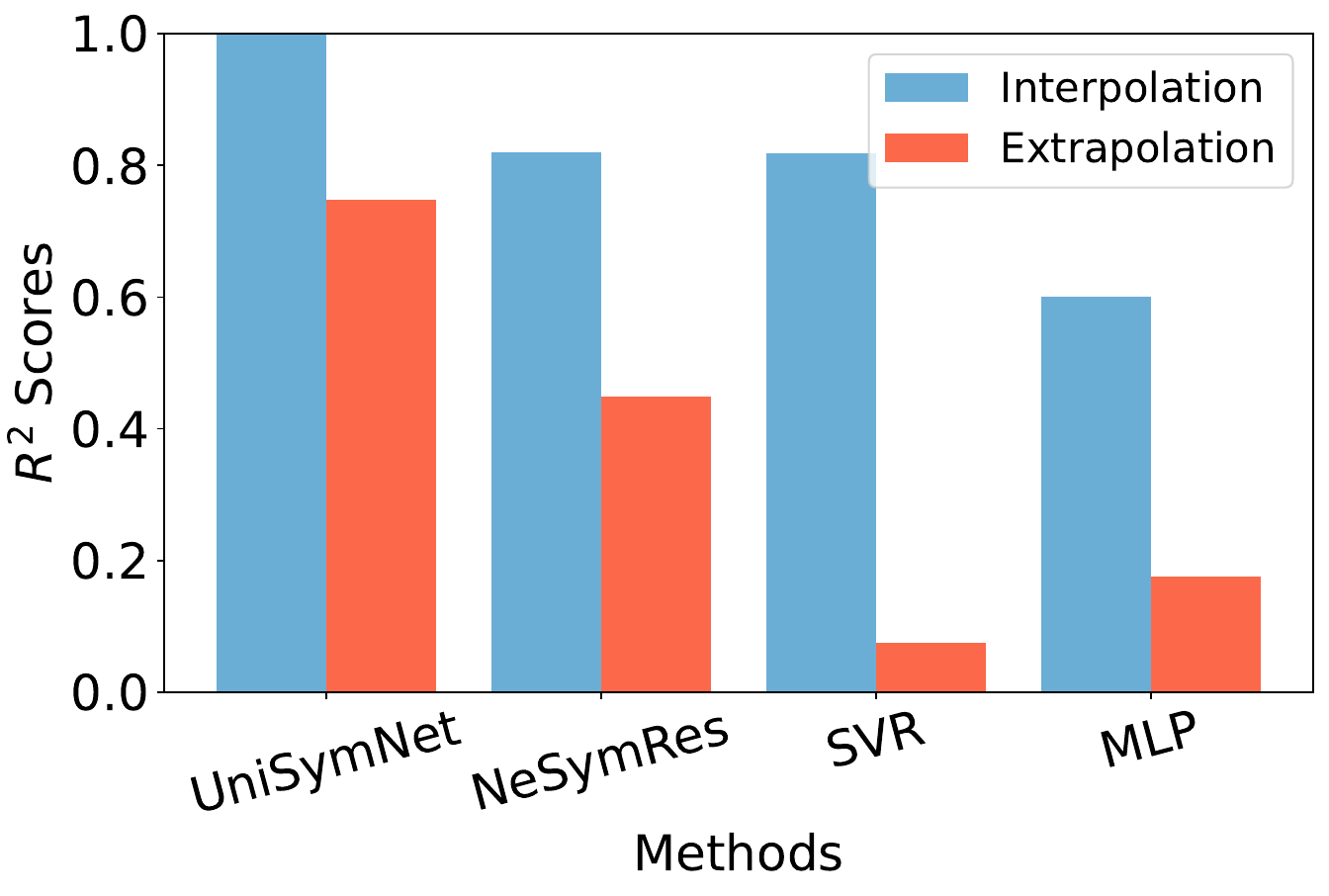}\\
    Livermore
  \end{minipage}

  \caption{Extrapolation ability comparison.} 
  
  \label{fig: Extrapolation Experiment: part 1}
\end{figure}
To compare the extrapolation capabilities of symbolic networks and standard neural networks, we conducted evaluations on the Standard Benchmarks datasets. Our comparison included the Transformer-based SR methods NeSymRes and UniSymNet, as well as the machine learning methods SVR and MLP. The training datasets used in our experiments are the same as in the previous section, while the extrapolation test sets span intervals twice the length of the training ranges. Fig. \ref{fig: Extrapolation Experiment: part 1} presents a qualitative \textcolor{black}{comparison of extrapolation performance} on four datasets, \textcolor{black}{while the remaining results are provided in Appendix \ref{appendix: Extrapolation Experiments}.} UniSymNet outperforms the other methods in interpolation and extrapolation $R^2$ score, demonstrating superior extrapolation ability. 

\begin{figure}[!h]
    \centering
    \begin{minipage}[t]{0.212\textwidth}
        \centering
        \includegraphics[width=\textwidth]{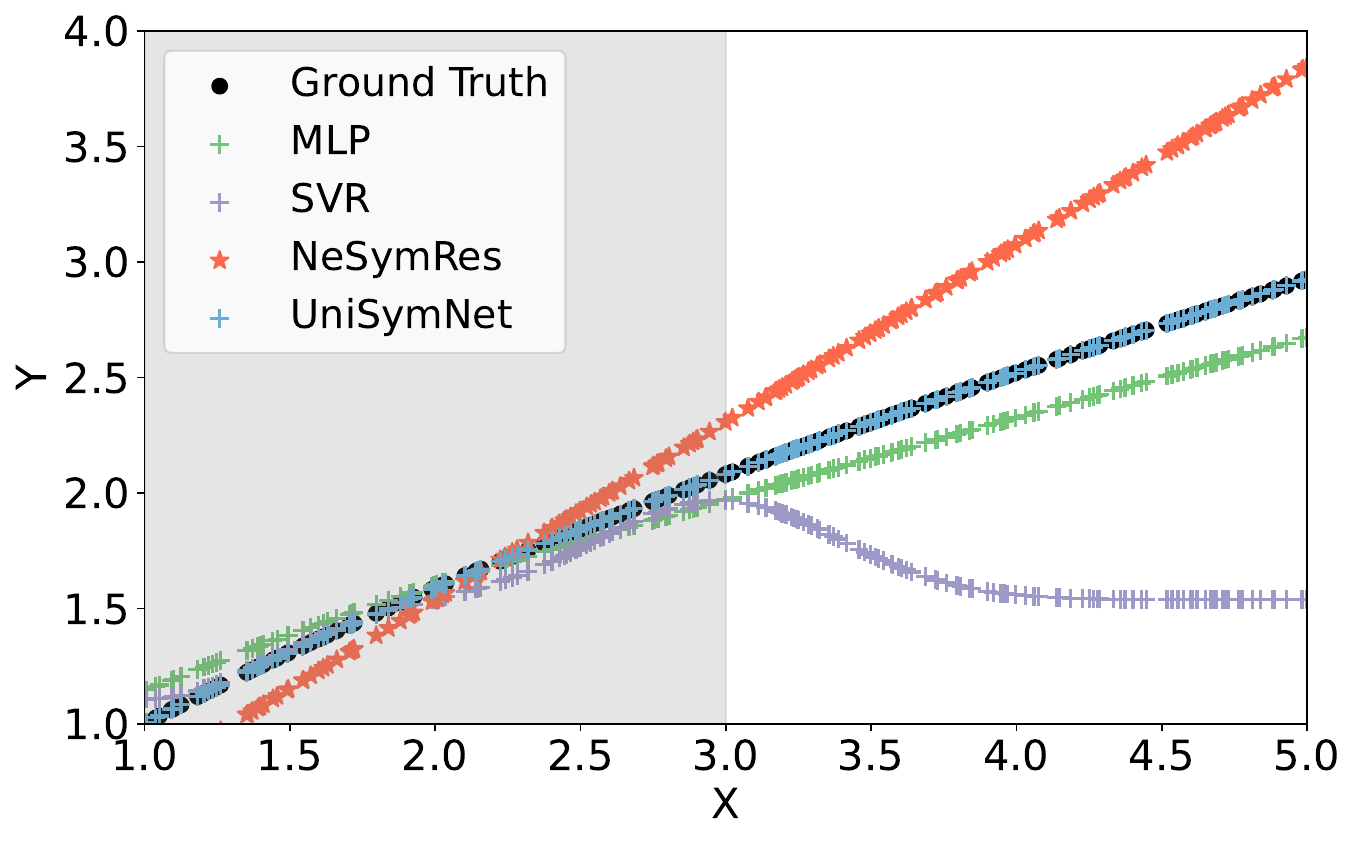}
    \end{minipage}
    \begin{minipage}[t]{0.225\textwidth}
        \centering
        \includegraphics[width=\textwidth]{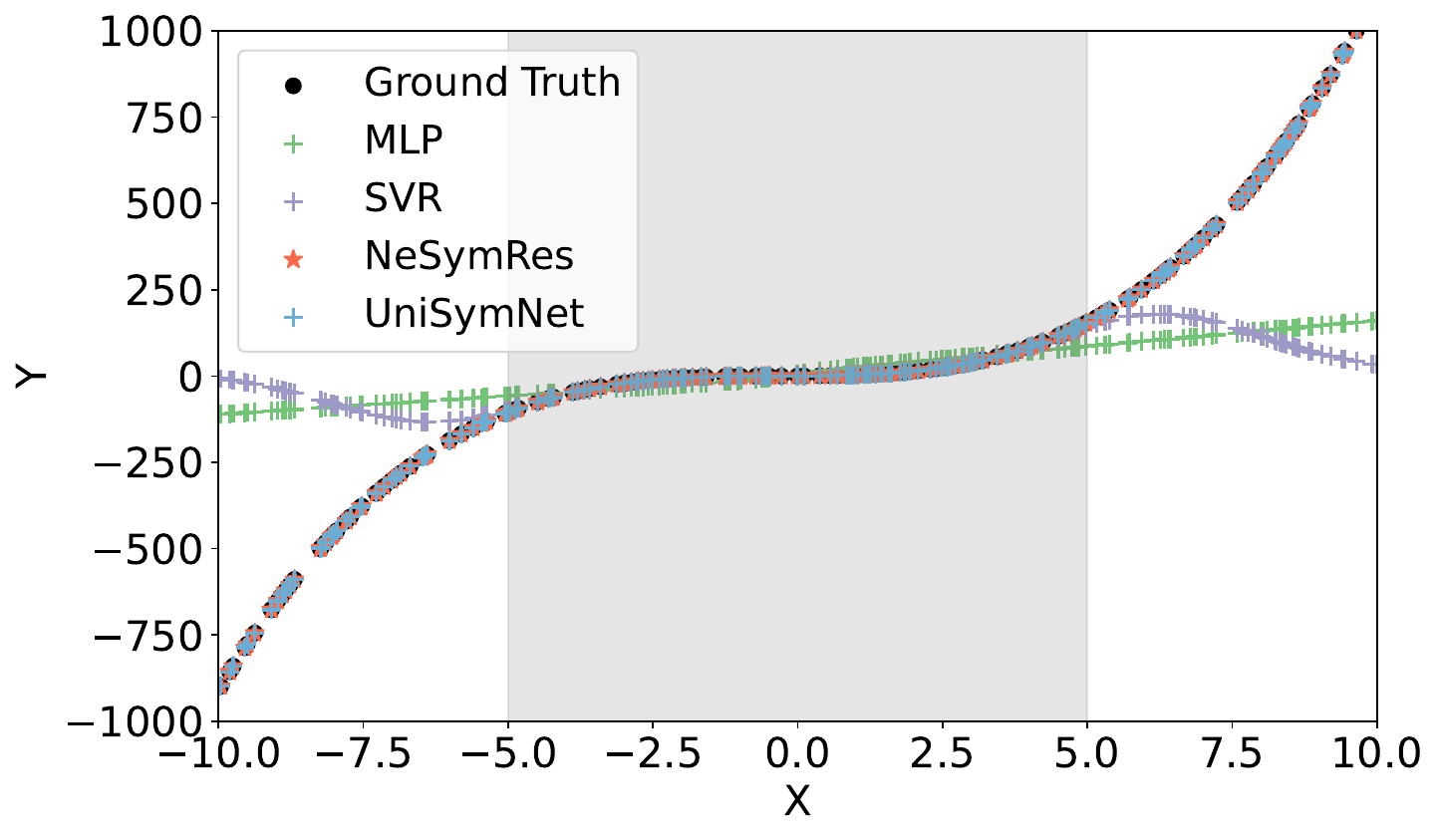}
    \end{minipage}
    \caption{Extrapolation ability comparison on the Livermore-13 and Nguyen-1 problems.}
    \label{fig: Extrapolation Experiment: part 3}
\end{figure}

To provide a more detailed illustration, we conducted an in-depth evaluation of the specific Livermore-13 and Nguyen-1 problems. Fig. \ref{fig: Extrapolation Experiment: part 3} presents extrapolation ability comparison on the ground-truth problems Livermore-13 ($x^{\frac{2}{3}}$) and Nguyen-1 ($x^3 + x^2 + x$). \textcolor{black}{Although} all four methods effectively fit the training data, SR methods outperform traditional machine learning methods (MLP and SVR) in the out-of-domain performance. On the Livermore-13 and Nguyen-1 problems, our method closely fits the ground truth, accurately capturing the underlying patterns.

\begin{table*}[!h]
\centering
\caption{Overall performance of UniSymNet on Feynman and Strogatz Datasets.}
\label{tab: Overall performance of UniSymNet on Feynman and Strogatz.}
\scalebox{1}{
\begin{tabular}{@{}lcccccccc@{}}
\toprule
DataSet& \multicolumn{4}{c}{Feynman} & \multicolumn{4}{c}{Strogatz} \\
\midrule
noise level & 0.0 & 0.001 & 0.01 & 0.1 & 0.0 & 0.001 & 0.01 & 0.1 \\
\cmidrule(lr){1-1}\cmidrule(lr){2-5} \cmidrule(lr){6-9}
Symbolic Solution Rate & 0.5690 & 0.5086 & 0.4483 & 0.2845 & 0.2857 & 0.1429 & 0.1429 & 0.1429 \\
Average $R^2$ Score   & 0.9791& 0.9708& 0.9681&\textcolor{black}{0.9596}& 0.9522 & 0.8781 & 0.8581 & 0.9113 \\
Average Complexity & \textcolor{black}{20.19}&  20.44& 19.07&\textcolor{black}{21.24}&  22.21& 20.50& 19.57 & 21.79\\
\bottomrule
\end{tabular}}
\end{table*}

\subsection{Out-of-domain Performance}
\label{sec: Out-of-domain performance}
We evaluated our method using the ground-truth problems in SRBench, including 130 datasets from Feynman and ODE-Strogatz. \textcolor{black}{Each dataset was evaluated under four noise levels:} 0, 0.001, 0.01, and 0.1. Table \ref{tab: Overall performance of UniSymNet on Feynman and Strogatz.} presents the detailed results of UniSymNet. And we compared the performance of different methods across three metrics: symbolic solution rate, fitting performance, and expression complexity. \textcolor{black}{Moreover, in Appendix \ref{appendix: Representative Examples}, Table \ref{tab: Representative examples in SRBench.} presents several representative examples of equations discovered from the SRBench datasets.}

\subsubsection{Symbolic Solution Rate}
\label{sec: Symbolic Solution Rate}

\begin{figure}[!ht]
    \centering
    \includegraphics[width=\linewidth]{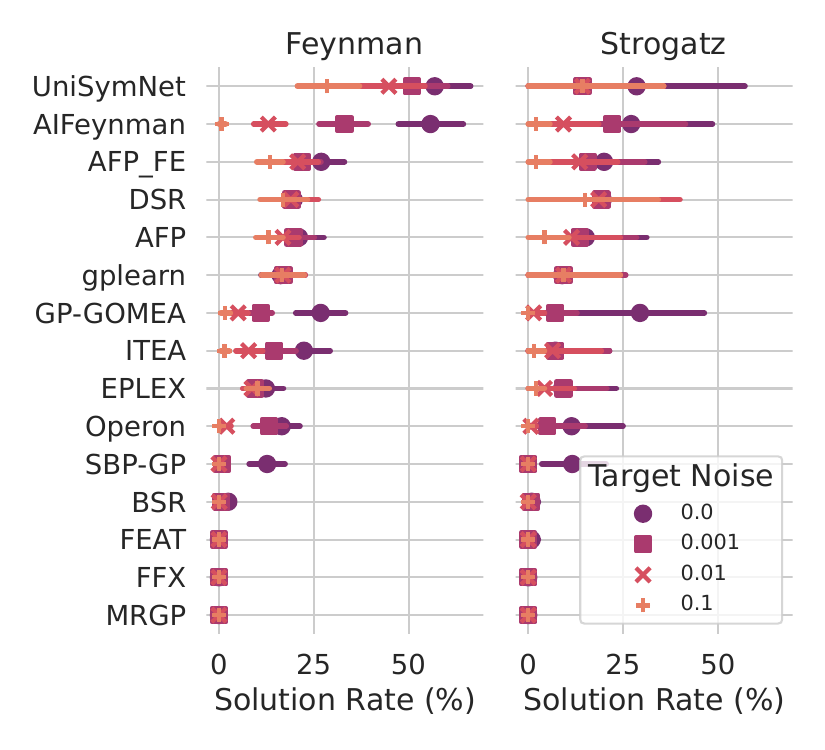}
    \caption{Symbolic solution rate comparison of 14 out-of-domain methods on SRBench under different noise levels.}
    \label{fig: solution_rate.pdf} 
\end{figure}

Fig. \ref{fig: solution_rate.pdf} shows a comparison of the symbolic solution rate between UniSymNet and 14 out-of-domain methods on ground-truth problems under different noise levels. \textcolor{black}{For the Feynman datasets, the results indicate that UniSymNet consistently outperforms all other methods across all noise levels. For the ODE-Strogatz datasets, our method achieves the highest symbolic solution rate in the noise-free setting, while showing a relatively larger performance drop as the noise level increases.} Overall, UniSymNet shows a strong capability to recover true symbolic expressions.

\subsubsection{Fitting Performance and Complexity}
\label{sec: Fitting Performance and Complexity}

\begin{figure}[!ht]
    \centering
    \includegraphics[width=\linewidth]{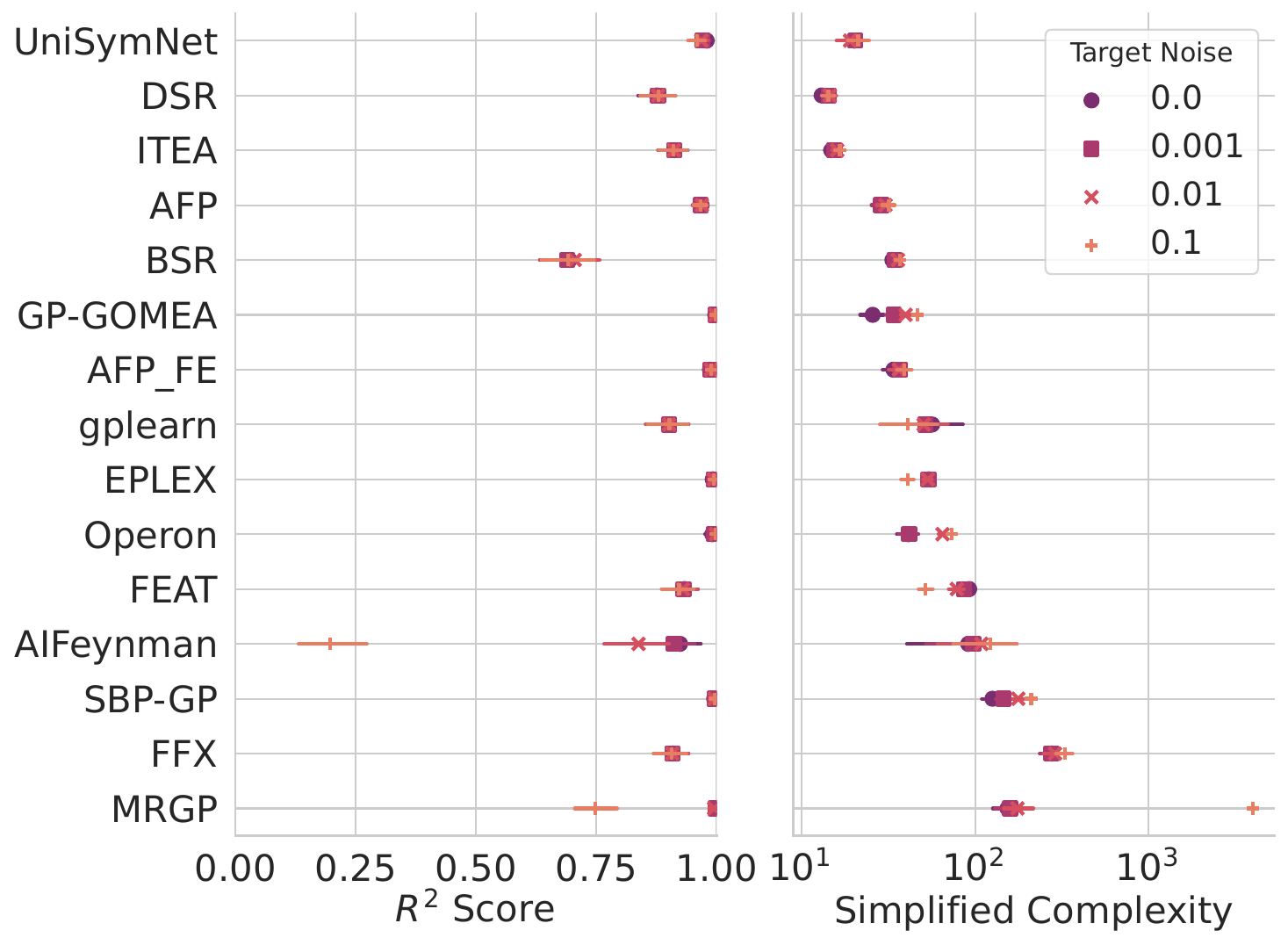}
    \caption{\textcolor{black}{Performance comparison of 14 out-of-domain methods on the Feynman datasets under different noise levels.}}
    \label{fig: Out-of-domain_Feynman.pdf} 
\end{figure}
Fig. \ref{fig: Out-of-domain_Feynman.pdf} \textcolor{black}{presents} fitting performance and complexity comparisons on Feynman datasets. Our method ranks third \textcolor{black}{in terms of} expression complexity, while achieving an average $R^2$ score of \textbf{0.9786} on the noise-free Feynman datasets. 
\begin{figure}[!ht]
    \centering
    \includegraphics[width=\linewidth]{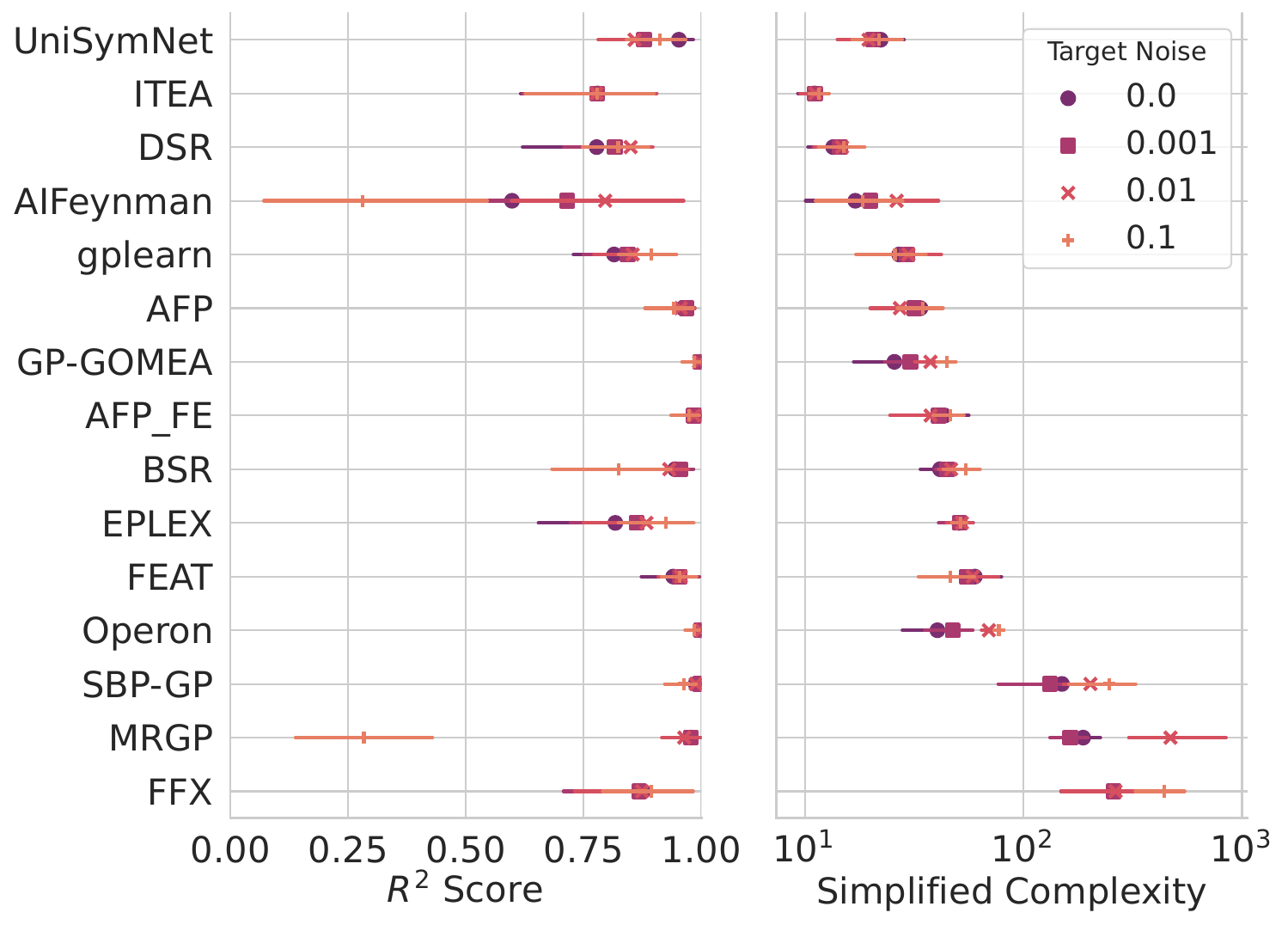}
    \caption{\textcolor{black}{Performance comparison of 14 out-of-domain methods on the ODE-Strogatz datasets under different noise levels.}}
    \label{fig: Out-of-domain_Strogatz.pdf}
\end{figure}
Fig. \ref{fig: Out-of-domain_Strogatz.pdf} shows fitting performance and complexity comparisons on ODE-Strogatz datasets. UniSymNet achieved an average $R^2$ score of \textbf{0.9522} on the noise-free dataset, while also maintaining a relatively low expression complexity. These results \textcolor{black}{indicate} that, given a similar level of expression complexity, our method provides better fitting performance.

Overall, UniSymNet shows an excellent symbolic solution rate \textcolor{black}{while} maintaining high fitting accuracy with low complexity. Our method performs better on Feynman than on Strogatz because the value range distribution of \textcolor{black}{the} training data is more aligned with that of Feynman. As the noise increases, the symbolic accuracy and fitting performance of UniSymNet and other methods decrease accordingly. \textcolor{black}{However, as shown in Fig. \ref{fig: Out-of-domain_Feynman.pdf} and Fig. \ref{fig: Out-of-domain_Strogatz.pdf}, our method remains stable across varying noise levels, consistent with the findings in Section \ref{sec: Ablation Studies}, which highlights its greater robustness relative to other methods.}

\subsection{\textcolor{black}{Convergence Behavior}}
\label{sec: Convergence Behavior}
\textcolor{black}{For the outer optimization, we trained a Transformer model to predict UniSymNet structures. As illustrated in Fig. \ref{fig:loss_transformer}, both the training and validation losses exhibit a rapid decline during the initial epochs and gradually plateau thereafter. By approximately the 50th epoch, the training loss has stabilized, while the validation loss fluctuates within a narrow margin, indicating that the model has effectively converged.} 

\begin{figure}[!h]
    \centering
    \includegraphics[width=\linewidth]{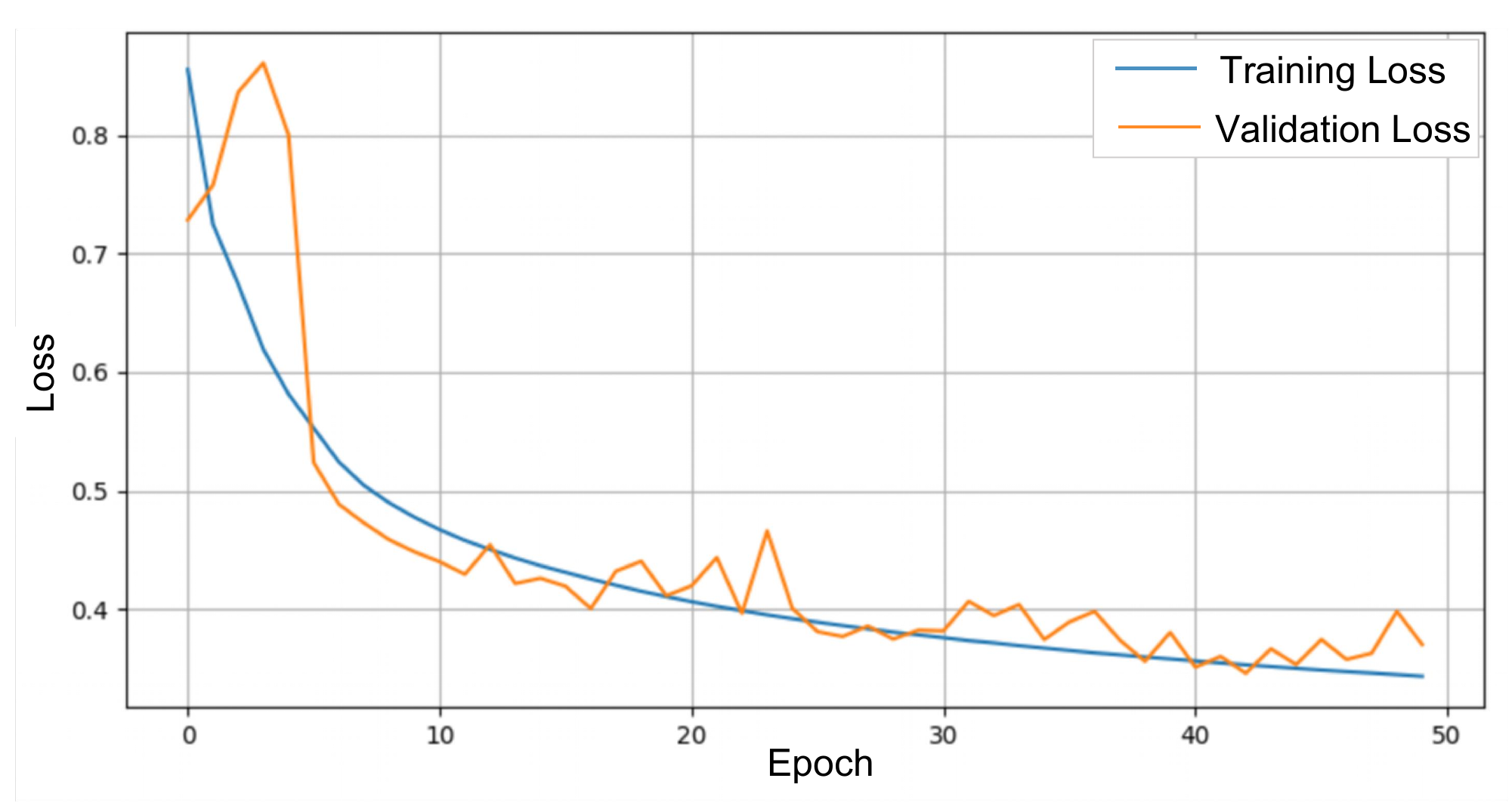}
    \caption{\textcolor{black}{Outer optimation convergence curve.}}
    \label{fig:loss_transformer}
\end{figure}

\textcolor{black}{For the inner optimization, we employed two methods, DNO and SFO. Appendix \ref{appendix: Convergence Behavior} presents representative convergence curves for a subset of equations. While not all cases are shown, the presented results are representative of the typical convergence behavior observed across our experiments.}

\section{Discussion}
\label{sec: Discussion}

\subsection{\textcolor{black}{$\Psi$ Transformation}}
\label{sec: transformation}
\textcolor{black}{In Section \ref{sec: Structural Space Formalization}, we formally introduce the definition of the $\Psi$ Transformation, and subsequently provide a detailed discussion of its advantages, its domain implications, and its limitations.}

\subsubsection{\textcolor{black}{Advantages}}
\label{sec: transformation advantages}
\textcolor{black}{The advantages in expressivity and complexity of this representation $\Psi$ have been discussed in Section \ref{sec: Structural Space Formalization}. Moreover, the representation $\Psi$ is beneficial when encoding network structures into label sequences. Since all activation functions in the network are unary operators, this unified representation ensures that the dimensions of operator inputs ($y^{(l)}$ in Equation \ref{eq: UniSymNet forward}) and outputs ($z^{(l)}$ in Equation \ref{eq: UniSymNet forward}) are identical. Consequently, all operators follow a uniform computational pattern. As a result, there is no need for operator-specific input handling during encoding, and the construction and parsing of label sequences become more straightforward.}

\subsubsection{\textcolor{black}{Domain Implications}}
\label{sec: domain implications}
\textcolor{black}{In Equation \ref{eq:binary_unification_constrained}, if the $\Psi$ transformation is applied in the order of $\exp(\ln(\cdot))$, it alters the domain of the original expression. For example, $\Psi(x_1 \times x_2) \triangleq \exp(\ln x_1 + \ln x_2)$. The implications of domain alteration on our method constitute an issue worthy of discussion.}

\textcolor{black}{It is important to emphasize that when generating the training data $\{(x_i,y_i)\}_i$, we did not generate them through the forward propagation of UniSymNet. Instead, the data were generated directly from the given function $f$. It means that for any input–output pair $\{(x_i,y_i)\}_i$, the label sequence used for Transformer training corresponds to the true equation structure, which is only transformed into the UniSymNet structure for representation. In other words, the $\Psi$ transformation influences only the representation of the equation for Transformer training and structure inference, without altering the correspondence between the data and the underlying function.}

\textcolor{black}{With this premise clarified, let us consider the impact of $\Psi$ on parameter optimization.}
\begin{itemize}
    \color{black}
    \item \textbf{Differentiable Network Optimization: }The DNO \\method, as it follows the forward propagation defined in Equation \ref{eq: UniSymNet forward}, is strongly affected by the domain restrictions of the $\ln$ operator, which often leads to numerical instability during training. In Section \ref{sec: Differentiable Network Optimization}, we introduce a regularized logarithm $ln_{reg}(x)$ and a penalty term in the loss function to discourage negative inputs to the $\ln$ operator. If the input data contains values in the positive domain, DNO generally works well; however, if the inputs lie entirely in the negative domain, DNO training will fail.
    \item \textbf{Symbolic Function Optimization: }Since the $\Psi$ transformation affects only the representation but not the data–function mapping, the network structure $S$ predicted by the Transformer can always be converted into equations $f_s$ without nested ${\ln, \exp}$ operations. SFO method merges nested $\exp$ and $\ln$ operators into a nonlinear binary operator during skeleton transformation, which helps avoid numerical instability. Consequently, in the SFO method, the domain restrictions introduced by $\Psi$ have no effect.
\end{itemize}
\textcolor{black}{Overall, when dealing with expressions that are valid over $\mathbb{R}$, the DNO method fails if the test data lie entirely in the negative domain, in which case the SFO method should be used for optimization. If the test data contain positive values, DNO tends to focus on the positive region to recover the target expression. Alternatively, the SFO method can be directly adopted as a substitute.}

\subsubsection{\textcolor{black}{Limitations}}
\label{sec: transformation limitations}
\textcolor{black}{
This representation \(\Psi\) is not entirely beneficial.
First, for equations involving functions such as \(\arcsin\), \(\arccos\), and \(\arctan\), it cannot provide an exact representation, although our method can identify suitable alternative expressions. Second, for some functions, the domain may change after applying the \(\Psi\) transformation, which can affect the performance of the DFO optimization method. In such cases, the SFO method can serve as an alternative, as discussed in Section \ref{sec: domain implications}}.

\subsection{Symbolic Trees vs. Unified Symbolic Networks}
\label{sec: Symbolic Trees vs. Symbolic Networks}
Symbolic trees represent mathematical expressions hierarchically, leaves represent operands, internal nodes denote operators, and edges define operational dependencies. However, UniSymNet represents mathematical expressions through forward propagation: input nodes encode variables, hidden layers apply operations, and edges carry learnable weights. \textcolor{black}{We now turn to a discussion of the differences between them.}

\textbf{Linear Transformation:} Symbolic Networks replace the $\{+,-\}$ operators with linear transformations, implying that the $\{+,-\}$ operators are no longer confined to binary operators but can achieve multivariate interactions, thereby adapting to high-dimensional feature spaces. The linear transformations in symbolic networks inherently constitute affine mappings in vector spaces\textcolor{black}{; therefore,} the expressive capacity of symbolic networks forms a superset encompassing all linear operations achievable by symbolic trees.  
For the $i$-th node in Equation \ref{eq: UniSymNet forward}, linear transformation can be written as:
\begin{equation}
    \label{eq: node forward}
     y_i^{(l)} = \sum_{j=1}^{d_{l-1}} W_{ij}^{(l)}z_j^{(l-1)} + b_i^{(l)}, \quad \forall i=1,2,\cdots,d_l,
\end{equation}
However, for the linear transformation in Equation \ref{eq: node forward}, symbolic tree representation requires constructing a binary tree structure for each $i$-th node. Its minimum depth grows as:
\[
\text{Depth}_{\text{tree}} = \left\lceil \log_2 \left( \sum_{j=1}^{d_{l-1}} \delta_{ij} \right) \right\rceil + 1,
\]
where $\delta_{ij} \in \{0,1\}$ indicates whether the connection $W_{ij}$ exists. 
\tikzset{
    inputnode/.style={circle, draw, fill=orange!10, minimum size=17pt, inner sep=0pt, font=\small},
    middlenode/.style={circle, draw, fill=blue!10, minimum size=17pt, inner sep=0pt, font=\small},
    same_module/.style={rectangle, draw, fill=green!10, minimum size=17pt, inner sep=0pt, font=\small},
    hiddennode/.style={circle, draw, fill=blue!10, minimum size=17pt, inner sep=0pt, font=\small},
    outputnode/.style={circle, draw, fill=green!10, minimum size=17pt, inner sep=0pt, font=\small},
    defnode/.style={minimum size=17pt, inner sep=0pt, font=\small},
    signal/.style={thick, ->, >=latex, color=gray!50} 
}
For example, Fig. \ref{fig: SymTree vs. SymNet} presents two different representations of the equation \( \sin(x_0 + x_1) + \cos(x_1) \). 
\begin{figure}[!h]
    \centering
    \begin{minipage}{0.18\textwidth}
        \centering
          \begin{tikzpicture}[
            edge from parent path={(\tikzparentnode.south) -- (\tikzchildnode.north)},
            every node/.style={draw, minimum size=6mm, inner sep=0pt},
            operator/.style={circle, fill=blue!10},
            leaf/.style={rectangle, fill=orange!10, text width=1.75cm, align=center, minimum height=8mm},
            var/.style={circle, fill=orange!10},
            level 1/.style={sibling distance=15mm, level distance=10mm}, 
            level 2/.style={sibling distance=15mm, level distance=10mm}, 
            level 3/.style={sibling distance=15mm, level distance=10mm}  
        ]
        \node [operator] {$+$}
          child {node [operator] {$\sin$}
            child {node [operator] {$+$}
              child {node [var] {$x_0$}}
              child {node [var] {$x_1$}}}}
          child {node [operator] {$\cos$}
            child {node [var] {$x_1$}}};
        \end{tikzpicture}
    \end{minipage}
    \hfill
    \begin{minipage}{0.20\textwidth}
    \centering
    \begin{tikzpicture}[shorten >=1pt,->, node distance=1.5cm and 1cm]
        \foreach \i/\name in {0/$x_1$, 1/$x_0$}
            \node[inputnode] (I\i) at (0, \i+1.5){\name} ;
       
        \foreach \i/\name in {0/$\sin$, 1/$\cos$}
            \node[middlenode] (H2\i) at (1.5, \i+1.5){\name} ;
        
        \node[hiddennode] (1) at (3, 2) {$id$};
        \draw[thick, ->, >=latex,  red!100] (I0) -- (H20); 
        \draw[thick, ->, >=latex,  red!100] (I1) -- (H21); 
        \draw[thick, ->, >=latex,  red!100] (I1) -- (H20); 
        \draw[thick, ->, >=latex,  red!100] (H20) -- (1);  
        \draw[thick, ->, >=latex,  red!100] (H21) -- (1); 
        \foreach \i in {0}
            \foreach \j in {1}
                \draw[signal] (I\i) -- (H2\j);
    \end{tikzpicture}
\end{minipage}
\caption{\textcolor{black}{Symbolic Tree vs. UniSymNet. For equation $ \sin(x_0 + x_1) + \cos(x_1)$, the depth of the symbolic tree is 3, whereas the number of hidden layers in our symbolic network is only 2.}}
\label{fig: SymTree vs. SymNet}
\end{figure}

\begin{figure}[!h]
    
    \begin{minipage}{0.15\textwidth}
          \begin{tikzpicture}[
            edge from parent path={(\tikzparentnode.south) -- (\tikzchildnode.north)},
            every node/.style={draw, minimum size=6mm, inner sep=0pt},
            operator/.style={circle, fill=blue!10},
            leaf/.style={rectangle, fill=green!10},
            var/.style={circle, fill=orange!10},
            level 1/.style={sibling distance=20mm, level distance=10mm}, 
            level 2/.style={sibling distance=10mm, level distance=12mm}, 
            level 3/.style={sibling distance=10mm, level distance=10mm}  
        ]
        \node [operator] {$+$}
          child {node [operator] {$op_1$}
            child {node [leaf] {$g(x_0)$}}
              }
          child {node [operator] {$op_2$}
            child {node [leaf] {$g(x_0)$}}
            child{node [var] {$x_0$}}
            };
        \end{tikzpicture}
    \end{minipage}
    \hfill
    \begin{minipage}{0.25\textwidth}
    \begin{tikzpicture}[shorten >=1pt,->, node distance=1.5cm and 1cm]
        \foreach \i/\name in {0/$x_0$}
            \node[inputnode] (I\i) at (0, \i+2){\name} ;
        \foreach \i/\name in {0/$id$}
            \node[middlenode] (H1\i) at (1.5, \i+1.5){\name} ;
        \foreach \i/\name in {1/$g(\cdot)$}
            \node[same_module] (H1\i) at (1.5, \i+1.5){\name} ;
        \foreach \i/\name in {0/$op_2$, 1/$op_1$}
            \node[middlenode] (H2\i) at (3, \i+1.5){\name} ;
        
        \node[hiddennode] (1) at (4, 2) {$id$};
        \draw[thick, ->, >=latex,  red!100] (I0) -- (H10); 
        \draw[thick, ->, >=latex,  red!100] (I0) -- (H11);
        \draw[thick, ->, >=latex,  red!100] (H11) -- (H21);  
        \draw[thick, ->, >=latex,  red!100] (H10) -- (H20); 
        \draw[thick, ->, >=latex,  red!100] (H11) -- (H20); 
        \draw[thick, ->, >=latex,  red!100] (H20) -- (1);  
        \draw[thick, ->, >=latex,  red!100] (H21) -- (1);

        \foreach \i in {0}
            \foreach \j in {1}
                \draw[signal] (H1\i) -- (H2\j);
      
    \end{tikzpicture}
\end{minipage}
\caption{\textcolor{black}{Symbolic Tree vs. UniSymNet. UniSymNet can achieve cross-layer connections of identical modules through the id operator, whereas a symbolic tree cannot.}}
\label{fig: Symbolic Tree vs. Unified Symbolic Network (id)}
\end{figure}
$\mathbf{id}$ \textbf{Operators:} In symbolic networks, the identity operator \(id\), which performs an identity transformation, facilitates cross-layer connections for identical modules, thereby eliminating redundant subtrees. In Fig. \ref{fig: Symbolic Tree vs. Unified Symbolic Network (id)}, we use \( g(\cdot) \) to represent the same module. As the complexity of \( g(\cdot) \) increases, the advantage of the \( id \) operator becomes increasingly significant.

\textbf{$\Psi$ Transformation:} Symbolic trees and EQL-type symbolic networks typically use standard binary operators. In the UniSymNet, we represent binary operators using nested unary operators, i.d. $\Psi$ transformation. A comprehensive discussion of this is given in Section \ref{sec: transformation}.

\subsection{Statistical Comparison of Complexity}

Since the combinations of operators in equations are highly diverse, it is theoretically challenging to compare the complexity of symbolic tree representations and the UniSymNet representations. Therefore, we conduct an empirical comparison of their complexities under controlled conditions.

\textcolor{black}{The structural complexity of symbolic trees and UniSymNet representations is evaluated using two fundamental measures:}
\begin{itemize}
    \item \textbf{Depth ($L$)}: For symbolic trees, $L$ denotes the longest path from the root node to any leaf node. In UniSymNet, it corresponds to the number of hidden layers.
    
    \item \textbf{Node count($N$)}: In symbolic trees, this enumerates all operators and operands within the expression. For UniSymNet, it counts the total nodes across all layers.
\end{itemize}
The overall complexity is positively correlated with the two measures mentioned above\textcolor{black}{; therefore,} we define complexity $C = {LN}$. Based on a given frequency distribution of operators, we compare the complexity of the two representations over 100,000 generated expressions. 
\begin{figure}[ht]
\centering
\includegraphics[width=0.48\textwidth]{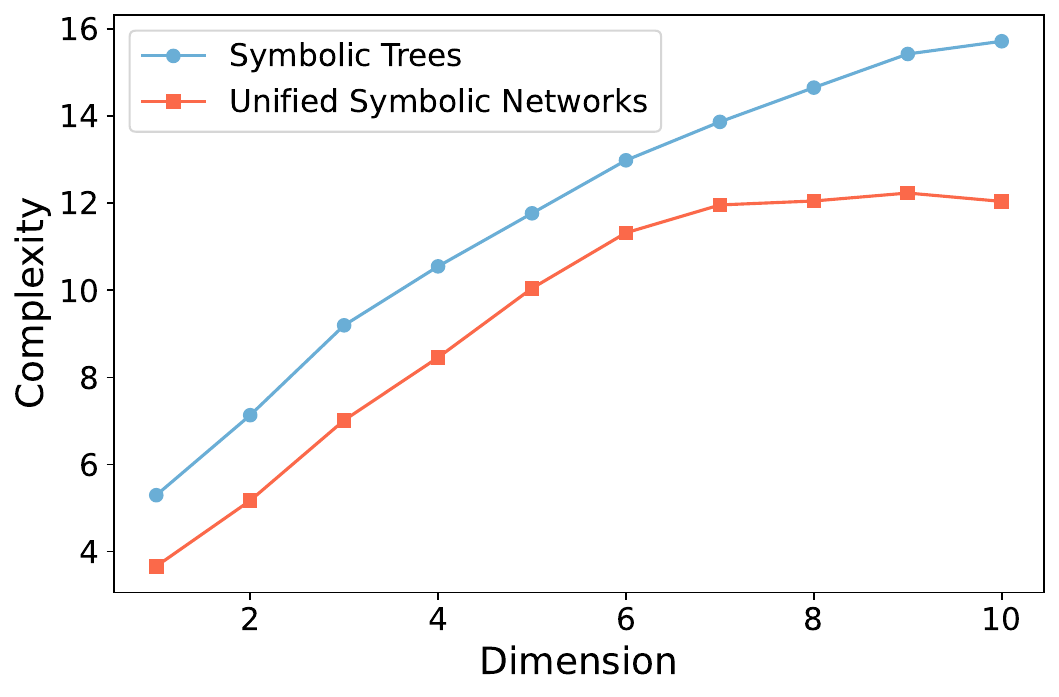}
\caption{Complexity comparison of two representation methods as the dimension varies.}
\label{fig: Complexity comparison}
\end{figure}
Fig. \ref{fig: Complexity comparison} presents the complexity comparison between the two representations. As the dimensionality increases, the complexity gap between UniSymNet and symbolic trees becomes more pronounced, highlighting the advantage of UniSymNet.
\section{Conclusion}
\label{sec: conclusion}

In conclusion, we propose a novel \textbf{Uni}fied \textbf{Sym}bolic \textbf{Net}work (UniSymNet), where a pre-trained Transformer model guides structure discovering, and then \textcolor{black}{adopts} objective-specific optimization strategies to learn the parameters of the symbolic network. Through experiments on both low-dimensional Standard Benchmarks and high-dimensional SRBench, we demonstrate the outstanding performance of UniSymNet, which not only maintains high fitting accuracy, symbolic solution rate\textcolor{black}{, }and extrapolation ability, but also ensures a relatively low equation complexity.

\textcolor{black}{In this work, UniSymNet is applied to Symbolic Regression tasks, but the framework is capable of addressing a broader range of problems. Its interpretable and flexible symbolic network design, together with the expressive $\Psi$ representation, provides a general mechanism for bridging symbolic structures and neural representations. Future research will explore extending UniSymNet to broader areas of scientific modeling. In particular, enriching the operator set with spatial and temporal differential operators such as $\partial_x$, $\partial_y$, and $\partial_t$, UniSymNet can be adapted for the discovery of governing equations. Furthermore, integrating UniSymNet with data-driven solvers would enable it to discover governing equations and solve them efficiently within a unified framework. These directions would equip UniSymNet with the capability to tackle complex modeling and simulation tasks, thereby advancing automated scientific discovery in domains governed by physical laws.}

\section*{CRediT authorship contribution statement}
\textbf{Xinxin Li:} Conceptualization, Formal analysis, Investigation, Methodology, Software, Visualization, Writing-original draft. \textbf{Juan Zhang:} Formal analysis, Investigation, Methodology, Resources, Supervision, Validation. \textbf{Da Li:} Formal analysis, Investigation, Methodology, Software. \textbf{Xingyu Liu:} Investigation, Software. \textbf{Jin Xu:} Formal analysis, Investigation. \textbf{Junping Yin:} Supervision, Project administration.


\section*{Acknowledgments}
This work was supported by the National Natural Science Foundation of China, No. 12201024.
\appendix
\section*{Appendix}
\section{Proof of Theorems}
\label{appendix: proof of Theorem}
In the UniSymNet, we transform the nonlinear binary operators \(\{\times, \div, \mathrm{pow}\}\) into nested unary operators through Equation \ref{eq: unified trans}. We demonstrate that in Theorem \ref{thm:exact-neural-rep}, under specific conditions, applying the mapping $\Psi$ can reduce both the depth and the number of nodes in a symbolic network, thereby lowering its overall complexity. \textcolor{black}{First, we introduce a lemma that establishes an important inequality, which will be used in the proof of Theorem \ref{thm:exact-neural-rep}.}
\begin{lemma}
\label{lemma: N_1,N_2}
For integers \( K \geq 2 \) and \( d \geq 2 \), the following inequality holds:  
\[
K\left(\sum_{i=1}^{\lceil \log_2{d} \rceil} \left\lceil \frac{d}{2^i} \right\rceil +1\right) \geq d + K +1.
\]    
\end{lemma}
\begin{proof}
  Let \( S(d) = \sum_{i=1}^{\lceil \log_2{d} \rceil} \left\lceil \frac{d}{2^i} \right\rceil \), then:
\[
S(d) \geq \sum_{i=1}^{\lceil \log_2{d} \rceil} \frac{d}{2^i} = d \cdot \left(1 - \frac{1}{2^{\lceil \log_2{d} \rceil}}\right).
\]  
Since \( 2^{\lceil \log_2{d} \rceil} \geq d \), we have \( \frac{1}{2^{\lceil \log_2{d} \rceil}} \leq \frac{1}{d} \), yielding:  
\begin{equation}
\label{eq: neq}
S(d) \geq d \left(1 - \frac{1}{d}\right) = d - 1.
\end{equation}
Substitute \( S(d) \geq d - 1 \) into the left-hand side (LHS) of the original inequality:  
\[
K \cdot S(d) + 1 \geq K(d - 1) + 1 = Kd - K + 1.
\]  
To prove \( Kd - K + 1 \geq d + K + 1 \), simplify:  
\[
Kd - K + 1 - d - K - 1 \geq 0 \quad \Rightarrow \quad d(K - 1) - 2K \geq 0.
\]  
Rearranging gives:  
$$
    d \geq \frac{2K}{K - 1}=2+\frac{2}{K-1}.
$$
Because $2<\frac{2K}{K-1}\le4$, if $d\ge 4$, the inequality holds. Then, we just need to verify the condition of $2\le d<4,K\ge2$:

\begin{enumerate}[label=(\roman*)]
    \item $d=2$: $S(d)=2, KS(d)=2K=K+K\ge 2+K$;
    \item $d=3$: $S(d)=3, KS(d)=3K>2K\ge 2+K.$
\end{enumerate}
Therefore, for all integers \( K \geq 2 \) and \( d \geq 2 \), the inequality \( K\left(\sum_{i=1}^{\lceil \log_2{d} \rceil} \left\lceil \frac{d}{2^i} \right\rceil \right) + 1 \geq d + K + 1 \) holds.  
\end{proof}

Second, we start to prove Theorem \ref{thm:exact-neural-rep}.
\begin{proof}
    Consider \(P(\mathbf{x})\) in $\mathbb{R}^d$ with $d\ge2$, 
    $$
    P(\mathbf{x}) = \sum_{k=1}^K c_k \prod_{i=1}^d x_i^{m_{ki}}, \quad X = (x_1, \ldots, x_d)
    $$
    where $m_{ki} \in \mathbb{R}$ denotes the exponent of variable $x_i$ in the $k$-th monomial.
    Applying the unified representation $\Psi$ defined in Equation \ref{eq: unified trans}, we obtain:
    \begin{equation}
    \label{eq: poly transformation}
    \begin{aligned}
    \Psi(P(\mathbf{x})) &= \Psi\left(\sum_{k=1}^K c_k \prod_{i=1}^d x_i^{m_{ki}}\right) \\
                 &= \sum_{k=1}^K c_k \exp\left(\sum_{i=1}^d m_{ki} \ln x_i\right)
    \end{aligned}
    \end{equation}
    The $\Psi$ enables representation through the following UniSymNet architecture:
    \begin{itemize}
        \item \textbf{Layer 1 ($\mathbf{ln}$ operators)}: Compute $\{\ln x_i\}_{i=1}^d$ using $d$ nodes.
        \item \textbf{Layer 2 ($\mathbf{exp}$ operators)}: Compute $\exp\Bigl(\sum_{i=1}^d m_{ki}\ln x_i\Bigr)$ using $K$ nodes.
        \item \textbf{Layer 3 ($\mathbf{id}$ operators)}: if $K=1$, this layer is not necessary. Otherwise, a $\mathrm{id}$ node is needed to apply the linear combination. 
    \end{itemize} 
    Consequently, Any symbolic function \(P(\mathbf{x})\) satisfying the above conditions can be exactly represented by a UniSymNet with depth \(L_1\) and at most \(N_1\) nodes, where
    $$
\begin{cases}
    L_1=2, N_1 = d+1,      & \text{if } K=1, \\
    L_1=3 ,N_2=d+K+1,     & \text{else } K\ge 2.
\end{cases}
$$
If we use normal binary operators in EQL-type networks , the representation of $P(\mathbf{x})$ through the following  architecture:
\begin{itemize}
    \item \textbf{Section 1 ($\mathbf{(\cdot)^k}$ operators)}: given the condition: $$\exists (k_0, i_0),\ s.t. \ m_{k_0 i_0} \neq 1,$$ at least one additional layer with power-function activations is necessary to construct non-unit exponents. if $K=1$, there are $d$ nodes in the layer. However, if $K\ge 2$, since the depth required for each item may be inconsistent, such as $x_1x_2$and $x_1^2x_2$, exploring the number of nodes is complex. Therefore, if $K\ge 2$, we will ignore the number of nodes required for this section.
    \item \textbf{Section 2 ($\mathbf{\times}$ operators)}:  \( \forall k, \forall i,\, m_{ki} \neq 0 \), the monomial \( \prod_{i=1}^d x_i^{m_{ki}} \) inherently involves multiplicative interactions across all \( d \) variables. Since the multiplication operator \( \times \) is strictly binary, constructing such a monomial requires a network depth of \( \lceil \log_2{d} \rceil \), with each layer \( i \) containing \( K \lceil \frac{d}{2^i} \rceil \) nodes to iteratively merge pairs of terms.  
    \item \textbf{Section 3 ($\mathbf{id}$ operators)}:  if $K=1$, this layer is not necessary. Otherwise, a $\mathrm{id}$ node is needed to apply the linear combination. 
\end{itemize}
Consequently, Any symbolic function \(P(\mathbf{x})\) satisfying the above conditions can be exactly represented by a EQL-type network with depth \(L_2\) and at least \(N_2\) nodes, where
$$
\begin{cases}
    L_2=\lceil \log_2{d} \rceil + 1, N_2 = \sum_{i=1}^{\lceil\log_2 d\rceil} \Bigl\lceil \tfrac{d}{2^i}\Bigr\rceil + d,      & \text{if } K=1, \\
    L_2=\lceil \log_2{d} \rceil + 2 ,N_2 \ge \sum_{i=1}^{\lceil\log_2 d\rceil} \Bigl\lceil \tfrac{d}{2^i}\Bigr\rceil + 1,     & \text{else } K\ge 2.
\end{cases}
$$
When $K=1$, if $d\ge 2$, then $L_2=\lceil \log_2{d} \rceil+1 \ge 2=L_1$. According to Equation \ref{eq: neq}, we have $N_2\ge d+d-1\ge d+1=N_1$; When $K\ge 2$, if $d\ge 2$, then $L_2=\lceil \log_2{d} \rceil+2 \ge 3=L_1$. On the basis of Lemma \ref{lemma: N_1,N_2}, we have $N_2\ge N_1$. That is to say, when the above conditions are met, when representing the same $P(\mathbf{x})$, the depth and the number of nodes of UniSymNet are less than those of EQL.
    
\end{proof}

Finally, we will prove Theorem \ref{thm:composition-generalization}. 
\begin{proof}
By replacing $x_i$ in Theorem \ref{thm:exact-neural-rep} with the function $g_i(x)$, this is essentially treating it as the first layer input of a neural network. 

Let $P(\mathbf{x})$ require a UniSymNet with depth $L_{10}$ and node count $N_{10}$. Then the total depth is $L_1 = L_{10} + \max\{L_{1i}\}$, and the total node count is $N_1 = N_{10} + \sum_{i} N_{1i}$. For the EQL-type symbolic network, let $P(\mathbf{x})$ require a depth of $L_{20}$ and a node count of $N_{20}$. Then, the total depth is $L_2 = L_{20} + \max\{L_{2i}\}$, and the total node count is $N_2 = N_{20} + \sum_{i} N_{2i}$. 

Assuming the hypothesis in Theorem 1 holds, then $L_{10} \leq L_{20}$ and $N_{10} \leq N_{20}$. Therefore,
\begin{equation}
\begin{aligned}
    L_1 &= L_{10} + \max\limits_{i}\{L_{1i}\} \leq L_{20} + \max\limits_{i}\{L_{2i}\} = L_2, \\
    N_1 &= N_{10} + \sum_{i} N_{1i} \leq N_{20} + \sum_{i} N_{2i} = N_2.
\end{aligned}
\end{equation}

\end{proof}

\section{Algorithms for Label Encoding}
\label{appendix: Algorithms for Label Encoding}
To implement label encoding, it is crucial to convert function expressions in string form into a network structure. For this purpose, we employ the following Algorithm \ref{alg:uninet_structure}.

\color{black}{
\begin{algorithm}[!h]
\color{black}
\caption{\textcolor{black}{Structure Identification for UniSymNet}}
\label{alg:uninet_structure}
\begin{algorithmic}[1]
    \STATE{\bfseries Input:} Symbolic expression $expr$, dimension $d$
    \STATE{Initialize} UniSymNet's structure list $\mathcal{L}$
    \STATE $\mathcal{L} \gets \left[ \left[ \text{id} \right] \right]$ 
    \STATE $(\mathcal{L}_{op}, \mathcal{L}_{expr}) \gets   \textsc{SymbolicDecompose}(expr)$ 

    \STATE Append $\mathcal{L}_{op}$ to $\mathcal{L}$ 
    \WHILE{$\mathcal{L}_{expr} \neq []$}
        \STATE $\mathcal{T}_{op} \gets [],\ \mathcal{T}_{expr} \gets []$ 
        \FOR{$\sigma \in L_{expr}$} 
            \STATE $(\tau_{op}, \tau_{expr}) \gets \textsc{SymbolicDecompose}(\sigma)$
            \STATE Append $\tau_{op}$ to $\mathcal{T}_{op}$, Append $\tau_{expr}$ to $\mathcal{T}_{expr}$
        \ENDFOR
        \STATE Append $\mathcal{T}_{op}$ to $\mathcal{L}$ 
        \STATE $\mathcal{L}_{op} \gets \textsc{Flatten}(\mathcal{T}_{op}),\ \mathcal{L}_{expr} \gets \textsc{Flatten}(\mathcal{T}_{expr})$
    \ENDWHILE
    \STATE{\bfseries Output:}  $\mathcal{L}$
\end{algorithmic}
\end{algorithm}
}

\textcolor{black}{For the \textsc{SymbolicDecompose} algorithm in Algorithm \ref{alg:uninet_structure}, we provide a detailed explanation in Algorithm \ref{alg:string-to-subtree}. In Algorithm \ref{alg:string-to-subtree}, the \textsc{SymbolicDecompose} algorithm employs \textsc{GetOuterOp} to extract the outermost operator and \textsc{RemoveOuterOp} to return the operand associated with that operator. \emph{Marking $e_i$ as a bias or variable term} means that bias indicates the presence of additive or subtractive constant terms, whereas variable indicates the occurrence of $x_i$ variables.}

\color{black}{
\begin{algorithm}[!h]
\color{black}
\caption{\textsc{\textcolor{black}{SymbolicDecompose}}}
\label{alg:string-to-subtree}
\begin{algorithmic}[1]
\STATE{\bfseries Input:} Expression string $expr$
\STATE{\bfseries Initialize:}
\STATE Basic unary operators $\mathcal{O}_u \gets \{\sin, \cos, \exp, \log\}$
\STATE Binary operators $\mathcal{O}_b \gets \{\times, \div, \mathrm{pow}\}$ 
\STATE Special operators $\mathcal{O}_s \gets \{\mathrm{abs}, \mathrm{cosh}, \mathrm{sinh}\}$ 
\STATE $\mathcal{L}_{op}\gets [~]$, \quad $\mathcal{L}_{expr} \gets [~]$
\STATE Split $expr$ into parallel subexpressions $\{e_1,\dots,e_n\}$
\FOR{$e_i$ in $\{e_1,\dots,e_n\}$}
    \STATE Remove leading coefficients from $e_i$
    \STATE $f \gets \textsc{GetOuterOp}(e_i)$
    \IF{$f \in \mathcal{O}_s$}
        \IF{$f = \mathrm{abs}$}
            \STATE $f\gets \log$, $e_i \gets \Psi(e_i)$ 
        \ELSE
            \STATE $f\gets \exp,\exp$
        \ENDIF
    \ELSIF{$f \in \mathcal{O}_b$}
        \STATE $f\gets \exp$, $e_i \gets \Psi(e_i)$
    \ELSIF{$f \in \mathcal{O}_u$}
        \STATE $f\gets f$
    \ELSE
        \STATE Mark $e_i$ as bias or variable term

    \ENDIF
    \STATE Append $f$ to $\mathcal{L}_{op}$
    \STATE $e_i \gets \textsc{RemoveOuterOp}(e_i,f)$
    \STATE Append $e_i$ to $\mathcal{L}_{expr}$
\ENDFOR
\STATE{\bfseries Output:}  $\mathcal{L}_{op}, \mathcal{L}_{expr}$
\end{algorithmic}
\end{algorithm}
}

\textcolor{black}{To better illustrate this process, we provide an example. Consider the expression $\frac{x_0^2}{x_1}$. The iterative calls of \textsc{SymbolicDecompose} produce the following substructure:
\begin{equation*}
\begin{aligned}
\text{Step=0: }&\mathcal{L}^{(0)}_{\mathrm{op}} = [\mathrm{id}] \\
\text{Step=1: }&\mathcal{L}^{(1)}_{\mathrm{op}} = [\exp], 
\ \mathcal{L}^{(1)}_{\mathrm{expr}} = [ 2\log x_0 - \log x_1 ] \\
\text{Step=2: }&\mathcal{T}_{op}: [[log, log]],\ \mathcal{T}_{expr} [[x_0, x_1]],\\
&\mathcal{L}^{(2)}_{\mathrm{op}} = [\log, \log], 
\ \mathcal{L}^{(2)}_{\mathrm{expr}} = [ x_0, x_1 ] \\
\text{Step=3: }&  \mathcal{T}_{op}: [[x_0], [x_1]],\ \mathcal{T}_{expr}:[[~]],[~]], \\
&\mathcal{L}^{(3)}_{\mathrm{op}}= [x_0, x_1], \ \mathcal{L}^{(3)}_{\mathrm{expr}} = [~]
\end{aligned}
\end{equation*}
which yields the final result:
\[
\mathcal{L} = \big[ [\mathrm{id}],\ [\exp],\ [[\log,\log]],\ [[x_0],[x_1]] \big],
\]
The network's connectivity is defined by the operator list $\mathcal{L}$, which is read from left to right, with the output structure S being its reverse. Each sublist in $\mathcal{L}$ represents a layer: \([ [x_0], [x_1] ]\) indicates that two separate inputs feed into two $\log$ nodes in the next layer; \([ [\log, \log] ]\) means that The outputs from two $\log$ nodes are both connected to a single $\exp$ node; and \([\exp]\) represents the connection to the final output node.
}

{\color{black}
\section{Description of the Testing Data}
\label{appendix: Description of the Testing Data}
\subsection{Artificial DataSet in Ablation Studies}
\label{appendix: Artificial DataSet in Ablation Studies}
In the ablation study section, we randomly generated 200 equations with the constraint that the input dimension $d\leq 4,\ L\le 6$. Fig. \ref{fig: 3D_scatter} depict the relationship among the input dimension, label length, and the number of hidden layers within our dataset.

\begin{figure}[!h]
    \centering
        \includegraphics[width=\linewidth]{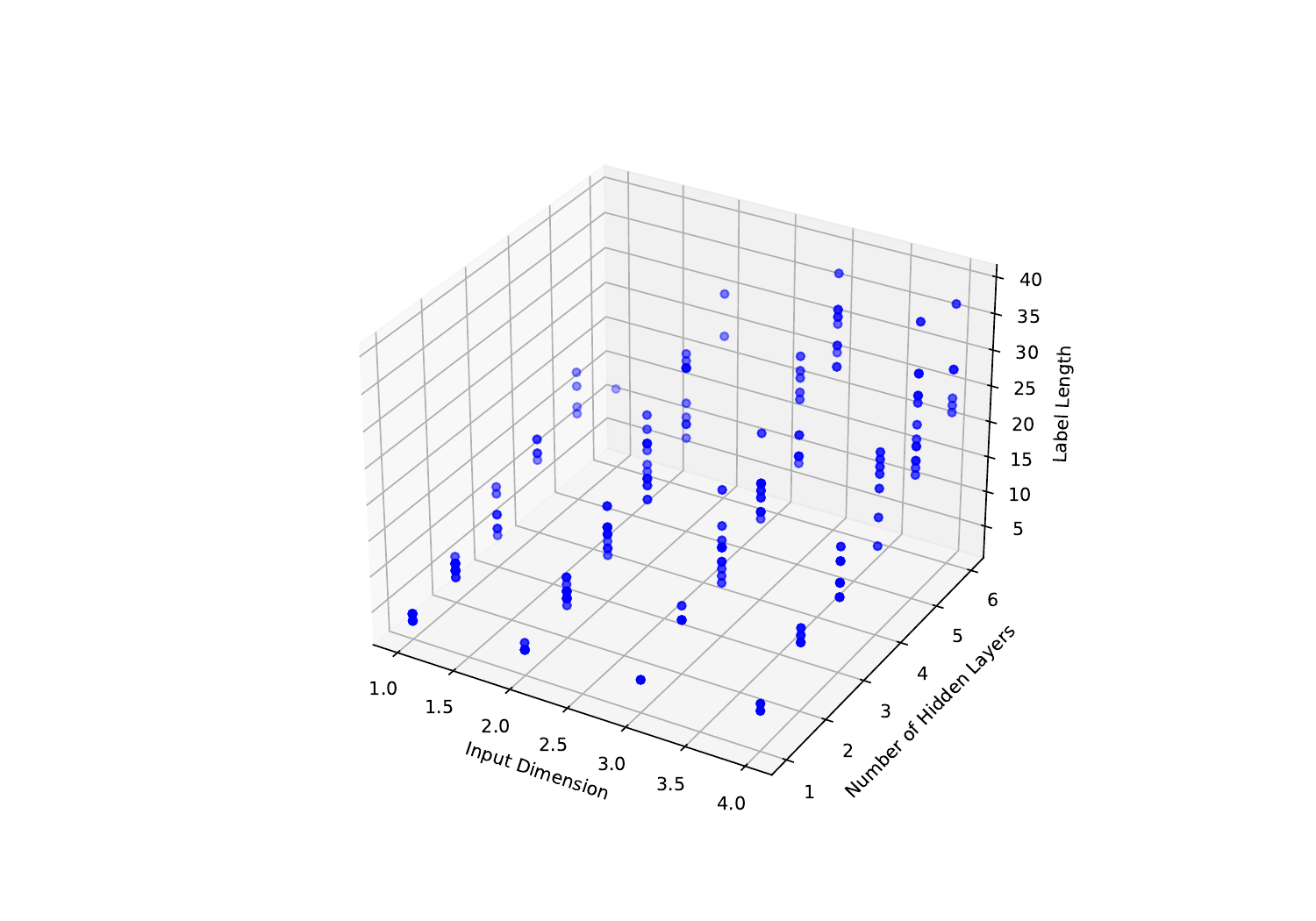}
        \caption{3D scatter plot of artificial dataset.}
        \label{fig: 3D_scatter}    
\end{figure}

\subsection{Standard Benchmarks}
\label{appendix: Standard Benchmarks}

\textcolor{black}{The Standard Benchmarks consist of eight datasets, each defined by equations with at most two input variables \(( d \leq 2 )\). Training data are obtained by evaluating these equations on randomly sampled inputs, with a fixed seed to ensure reproducibility (see Table \ref{Standard Benchmarks.}).} In contrast, testing sets consist of 100 points sampled from the same intervals but with a different random seed, ensuring an independent evaluation of model performance. For a fair comparison, the dataset used in all methods was generated under the same seed.

\subsection{SRBench}
\label{appendix: SRBench}
\textcolor{black}{The SRBench benchmarks consist of two datasets, Feynman and Strogatz, each defined by equations with input dimensions ranging from 2 to 10 (\( 2 \leq d \leq 10 \)).} The Feynman dataset comprises a total of 119 equations sourced from Feynman Lectures\citep{udrescu2020ai} on \textcolor{black}{the} Physics database series. The Strogatz dataset contains 14 SR problems sourced from the
ODE-Strogatz database.

\section{Comprehensive Overview of Baselines}
\label{Appendix: Comprehensive Overview of Baselines}

\subsection{In-domain Methods}
\label{appendix: In-domain Methods}

Symbolic regression methods can be divided into two main types: learning-from-scratch and learning-with-experience. The former learns each equation from scratch without any prior knowledge, while the latter uses pre-trained knowledge to guide the learning of new equations. Our in-domain baseline \textcolor{black}{includes} six methods, with EQL and DySymNet being learning-from-scratch methods, while the remaining three are learning-with-experience methods. Table \ref{tab: Comprehensive Overview of Baselines.} provides a detailed description of the in-domain methods.

\subsection{Out-of-domain Methods}
\label{appendix: Out-of-domain Methods}
We present concise descriptions of the 14 out-of-domain methods employed by SRBench, as shown in Table \ref{tab: Comprehensive Overview of Baselines.}. These 14 baseline methods include both (GP)-based methods and (DL)-based methods. The final experimental results of these methods are taken from SRBench \citep{la2021contemporary}.

\section{Computing Infrastructure}
\label{appendix: Computing Infrastructure}
The experiments in this work were executed on an Intel(R) Xeon(R) Platinum 8468H CPU @ 2.10GHz, 1.5T RAM equipped with two NVIDIA H800 GPUs 80GB. 
}

\section{Hyperparameter Settings}
\label{appendix: Hyperparameter Settings}
\subsection{\textcolor{black}{Initialization, Loss Functions in SFO method}}
\label{appendix: Initialization, Loss Functions}
\textcolor{black}{We compare four initialization methods on the Standard Benchmarks.}
\begin{itemize}
    \color{black}
    \item \textbf{Random: }Randomly samples from a normal distribution.
    \item \textbf{Grid search: }Systematically samples values from a predefined grid, ensuring exhaustive coverage at high computational cost.
    \item \textbf{Latin hypercube sampling: }Uses a space-filling strategy to sample parameter ranges without redundancy, achieving better coverage with fewer samples.
    \item \textbf{Metaheuristics: }Employs an evolutionary strategy with selection, crossover, mutation, and replacement.
\end{itemize}

\textcolor{black}{As shown in Table \ref{tab: Comparison of different initialization methods.}, the fitting accuracy, symbolic solution rate, and expression complexity vary only slightly across initialization methods. However, grid search yields lower accuracy on datasets such as Nguyen, Constant, and Livermore, likely due to locally suboptimal starting points. More importantly, all three methods significantly reduce testing time compared with random initialization, indicating that while initialization does not affect model capacity, it can accelerate convergence by providing better starting points. Overall, \textbf{Latin hypercube sampling} proves to be the most efficient choice, offering the highest optimization efficiency with comparable performance.}

\textcolor{black}{We adopt MSE as the primary loss function and further investigate the impact of alternative loss functions, specifically Huber loss and Quantile loss. Their definitions are given below:}
\begin{itemize}
    \color{black}
    \item \textbf{Huber Loss: }Given the residual $d = y - \hat{y}$ and a threshold parameter $\delta > 0$, the Huber Loss is defined as:
    $$
    L_{\delta}(d) =
    \begin{cases}
    \frac{1}{2} d^2, & \text{if } |d| \leq \delta \\
    \delta \cdot \big(|d| - \tfrac{1}{2}\delta \big), & \text{if } |d| > \delta
    \end{cases}
    $$
    The overall loss function (averaged over all samples) is given by:
    $$
    \mathcal{L}_{Huber}(y, \hat{y}) = \frac{1}{N} \sum_{i=1}^{N} L_{\delta}(d_i).
    $$The hyperparameter was set to $\delta=10$ in our experiments.
    \item \textbf{Quantile Loss: }Given a quantile parameter $\tau \in (0,1)$ and the residual $d = y - \hat{y}$, 
    the Quantile Loss is defined as:
    $$
    L_{\tau}(d) =
    \begin{cases}
    \tau \cdot d, & \text{if } d \geq 0 \\
    (\tau - 1) \cdot d, & \text{if } d < 0
    \end{cases}
    $$
    
    The overall loss function is given by:
    $$
    \mathcal{L}_{Quantile} = \frac{1}{N} \sum_{i=1}^{N} L_{\tau}(d_i).
    $$The hyperparameter was set to $\tau=0.5$ in our experiments.
\end{itemize}
\textcolor{black}{The results are presented in Table \ref{tab:loss_comparison_delta}. We observed that replacing MSE with alternative loss functions resulted in a marginal reduction in the $R^2$ score. This outcome is expected, as using MSE as the loss function is effectively equivalent to directly optimizing for $R^2$, thereby yielding the highest fitting accuracy. Nevertheless, when applying the Quantile loss, the true equation $sin(x) + sin(y^2)$ was correctly recovered in the Nguyen dataset. It indicates that robust loss functions can facilitate the discovery of the exact underlying equations.}

\subsection{Sensitivity Analysis in SFO method}
\label{appendix: Sensitivity Analysis in SFO method}
\textcolor{black}{To investigate the effects of the temperature schedule in Gumbel-Softmax and the entropy coefficient, we systematically compared three schedules: constant, linear, and exponential, and three entropy values: 0.005, 0.01, and 0.05. Results on the Nguyen dataset (Table \ref{tab: Sensitivity analysis of the parameters}) indicate that the symbolic solution rate is considerably more sensitive to these hyperparameters than the average $R^2$. For $R^2$, we provide confidence intervals to reflect variability, whereas the solution rate is reported as the proportion of correctly recovered equations. Among all settings, the exponential temperature schedule combined with an entropy coefficient of 0.005 delivers the most favorable trade-off between symbolic discovery and fitting accuracy. 
}

{\color{black}
\subsection{Total Hyperparameter Settings}
\label{appendix: Total Hyperparameter Settings}
For our method, the hyperparameters used in experiments are summarized in Table \ref{tab:Hyperparameters Setting.}, including parameters in the data generation, pre-training, beam search,  and inner optimization process. \textcolor{black}{Regarding the choice of optimization strategies across different datasets, we adopt SFO for the low-dimensional datasets in the Standard Benchmarks, where the target equations are relatively simple; this ensures low complexity expressions and a high symbolic solution rate. In the SRBench datasets, where some equations involve up to ten variables and are inherently more complex, the ODE dataset still relies entirely on SFO since it contains at most two variables. In contrast, a subset of the Feynman equations is optimized using DNO-NP, as listed in Table \ref{tab: dnonp_examples}.} \textcolor{black}{Overall, SFO proves more suitable for Symbolic Regression, and DNO-NP is only employed when the fitting performance of SFO is unsatisfactory.}

For other methods: EQL\footnote{https://github.com/martius-lab/EQL.git}, NeSymRes\footnote{https://github.com/SymposiumOrganization/NeuralSymbolicRegression\\ThatScales.git}, End-to-end\footnote{https://github.com/facebookresearch/symbolicregression.git}, DySymNet\footnote{https://github.com/AILWQ/DySymNet.git}, TPSR\footnote{https://github.com/deep-symbolic-mathematics/TPSR.git}, \textcolor{black}{ParFam \footnote{\textcolor{black}{https://github.com/Philipp238/parfam}}} and 14 out-of-methods\footnote{https://github.com/cavalab/srbench.git}, we used the standard hyperparameters provided in the open-source implementations. 

\section{Detailed Results}
\label{appendix: Detailed Results}
\subsection{\textcolor{black}{Representative Examples}}
\label{appendix: Representative Examples}
 \textcolor{black}{To better illustrate representative examples of the discovered equations from the benchmark datasets, we divide the presentation into two parts: Standard Benchmarks (Table \ref{tab: Representative examples in Standard_benchmarks.}) and SRBench (Table \ref{tab: Representative examples in SRBench.}). In the tables, Sol indicates whether the symbolic solution criterion is satisfied: T denotes that the definition of symbolic solution is met, while F denotes that it is not.}
\subsection{Extrapolation Experiments}
\label{appendix: Extrapolation Experiments}
\begin{figure}[!h]
  \centering
  \begin{minipage}{0.23\textwidth}
  \centering
    \includegraphics[width=\textwidth]{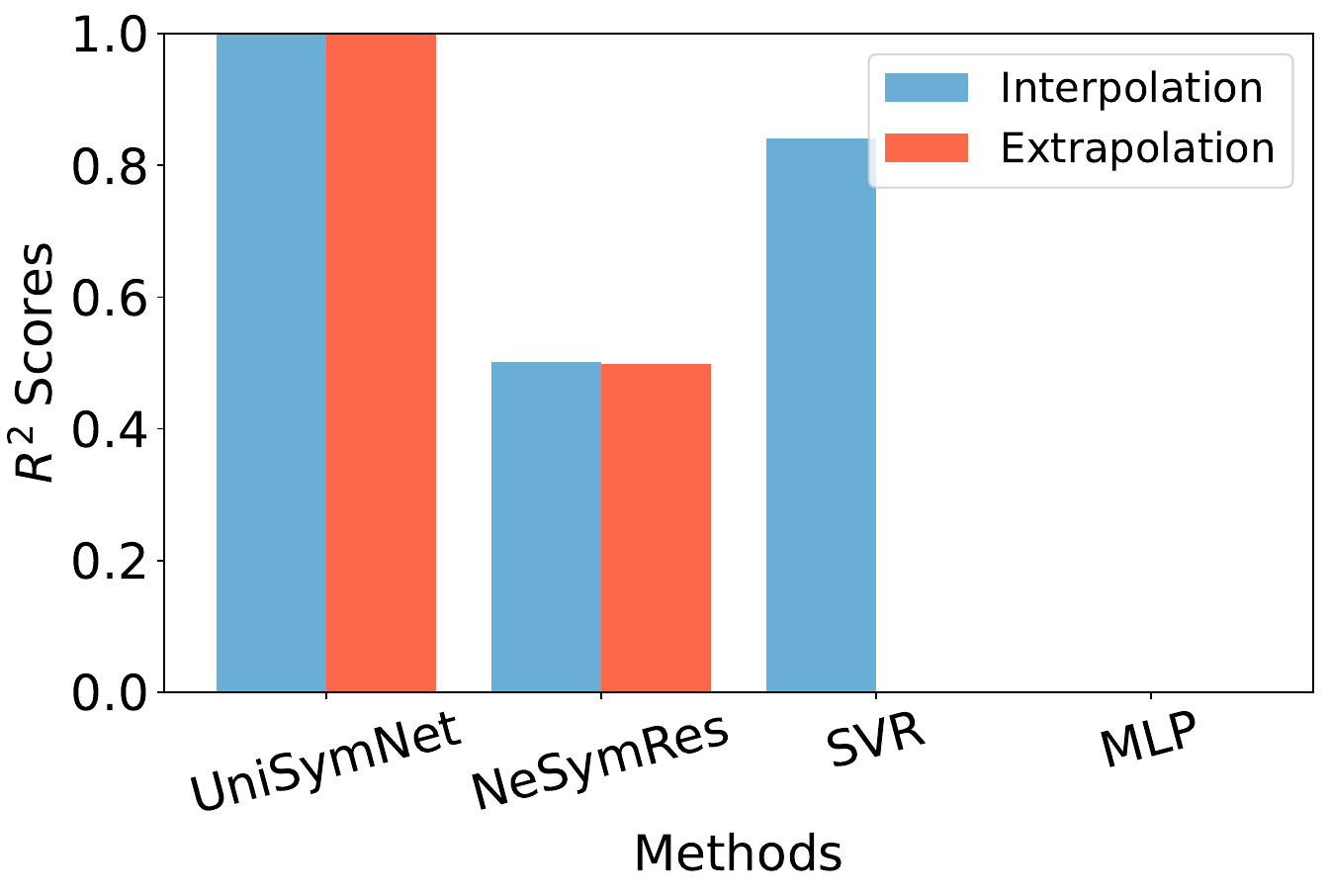} \\
    Koza
  \end{minipage}
  \begin{minipage}{0.23\textwidth}
  \centering
    \includegraphics[width=\textwidth]{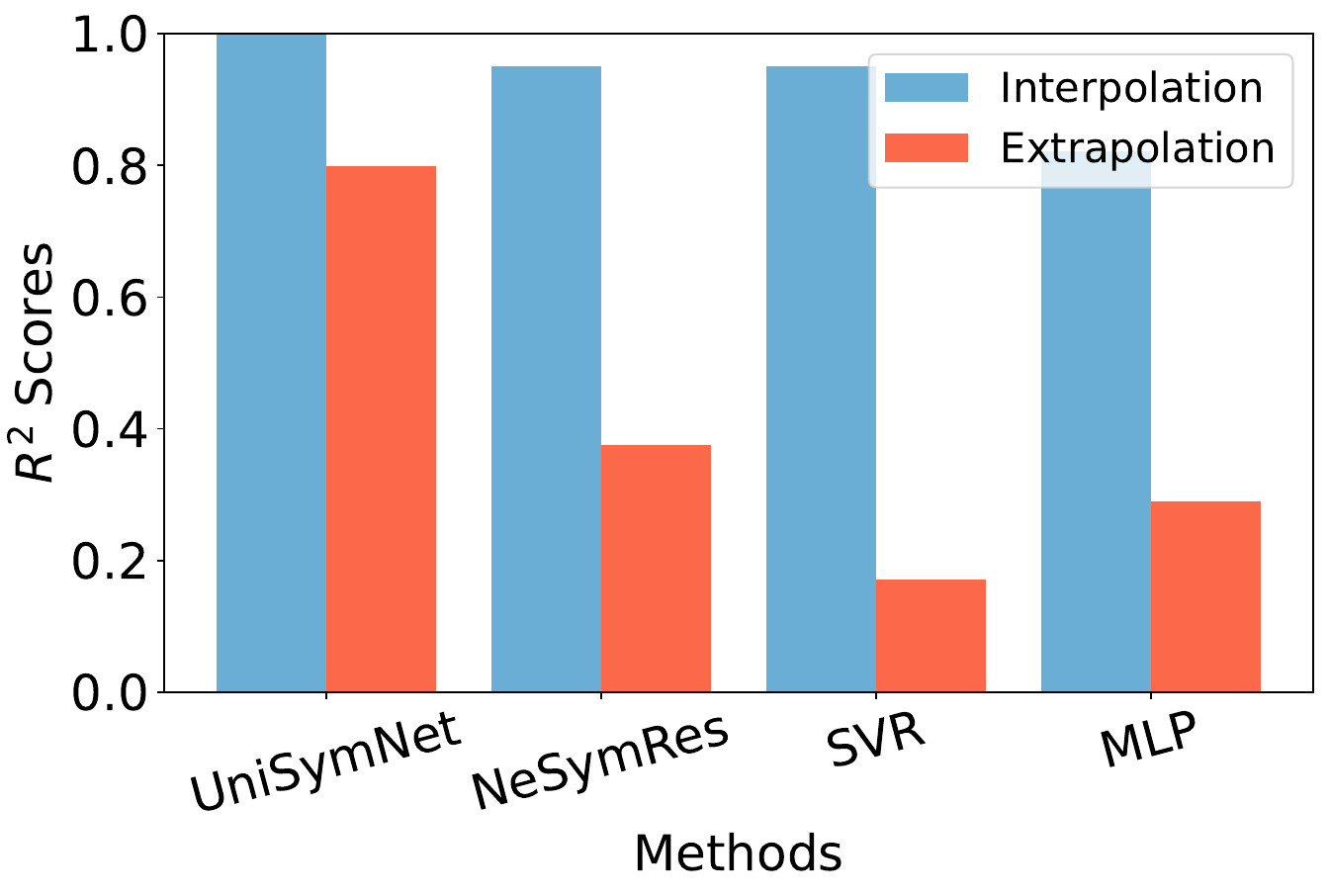}\\
    Nguyen*
  \end{minipage}

  
  \begin{minipage}{0.23\textwidth}
  \centering
    \includegraphics[width=\textwidth]{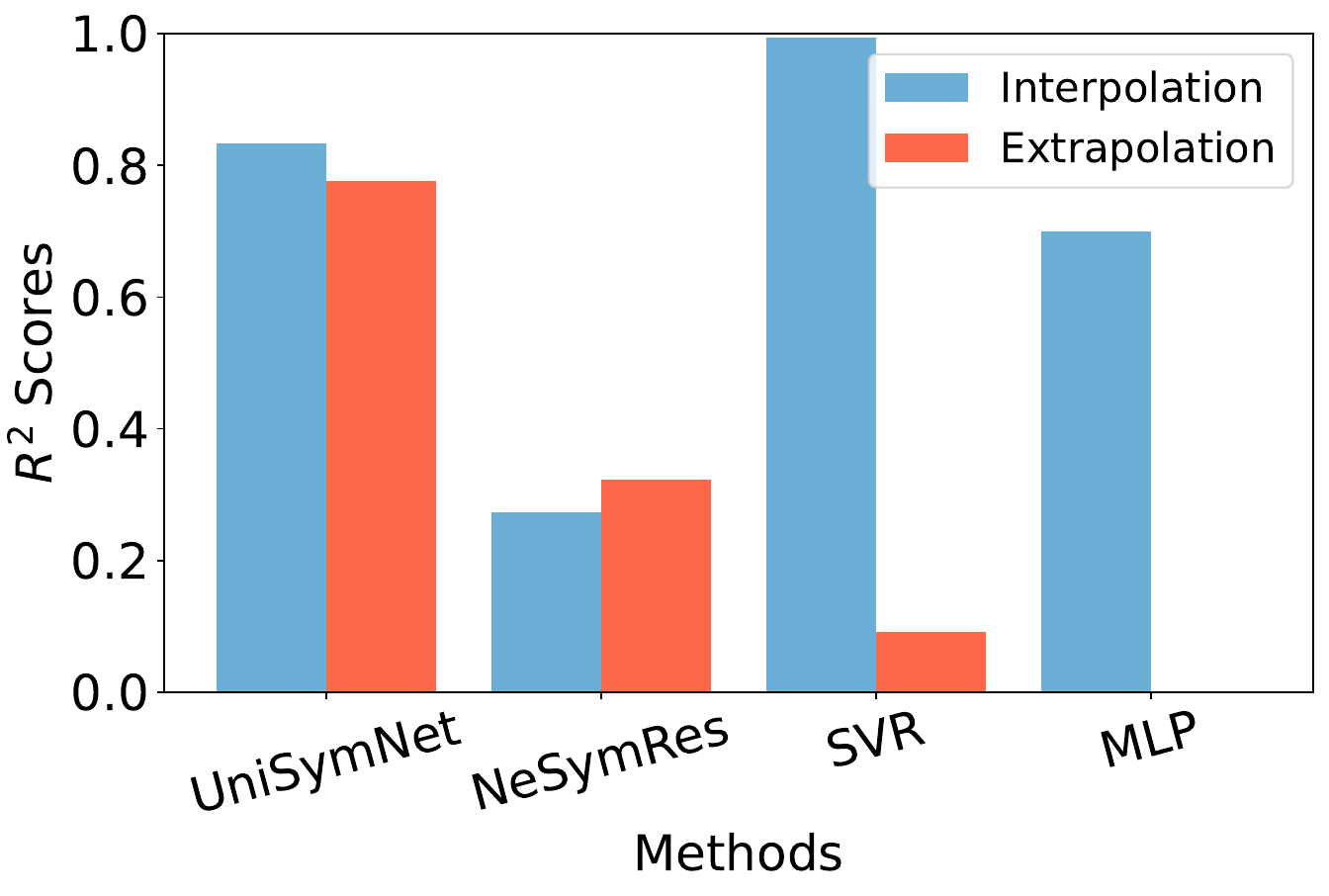}\\
    Jin
  \end{minipage}
  \begin{minipage}{0.23\textwidth}
  \centering
    \includegraphics[width=\textwidth]{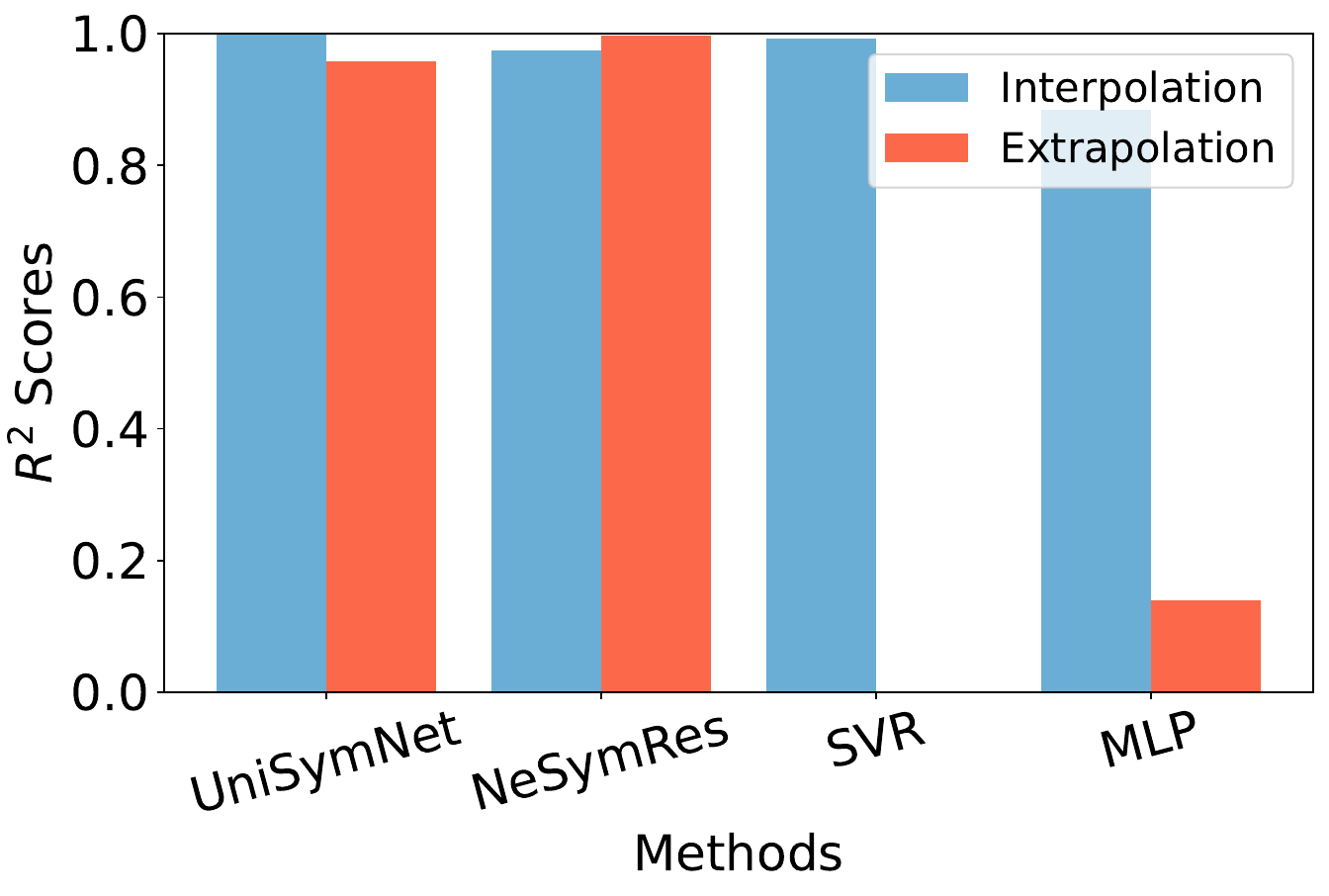}\\
    R
  \end{minipage}

  \caption{Extrapolation ability comparison.} 
  
  \label{fig: Extrapolation ability comparison: part 2.}
\end{figure}

In Fig. \ref{fig: Extrapolation Experiment: part 1}, we present a comparison of the extrapolation capabilities of in-domain methods on the Nguyen, Keijzer, Livermore, and Constant datasets. As shown in Fig. \ref{fig: Extrapolation ability comparison: part 2.}, our method, UniSymNet, shows superior extrapolation performance compared to black-box models on the Koza, Nguyen\*, and Jin datasets.
}
\subsection{\textcolor{black}{Convergence Behavior}}
\label{appendix: Convergence Behavior}
\textcolor{black}{For the inner optimization, we adopt two methods, DNO and SFO.}

\textcolor{black}{\textbf{DNO method:} In the ablation study with 200 test equations, four representative cases were randomly selected to illustrate the loss evolution during training. Fig. \ref{fig:loss_dno_np} presents the results of the DNO-NP method without pruning, whereas Fig. \ref{fig:loss_dno_p} shows the DNO-P method with pruning applied. Both methods achieve convergence within 100 training epochs, with DNO-P exhibiting a comparatively smoother convergence behavior.}
\begin{figure}[!h]
    \centering
    \includegraphics[width=0.95\linewidth]{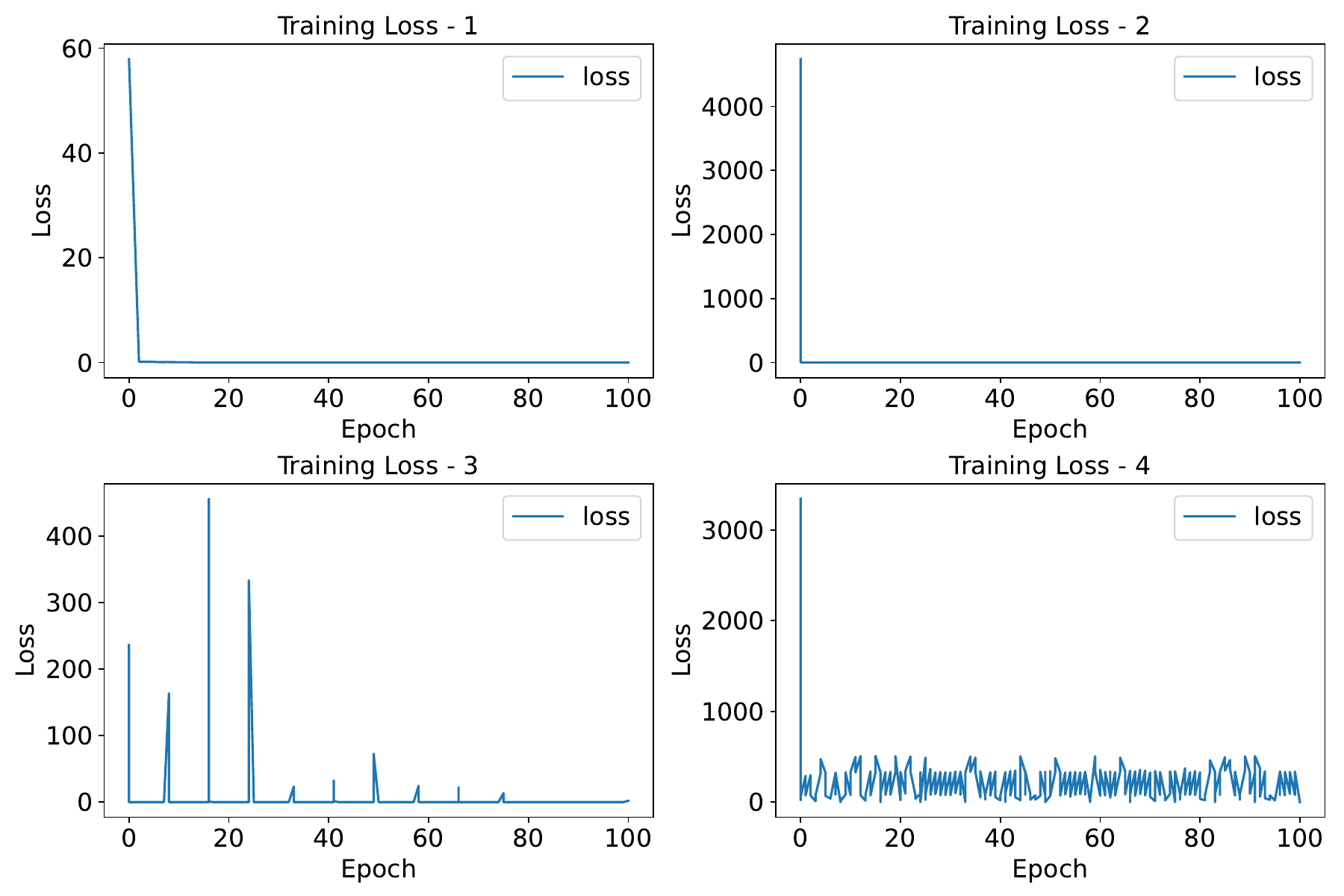}
    \caption{\textcolor{black}{Inner optimation convergence curve(DNO-NP).}}
    \label{fig:loss_dno_np}
\end{figure}
\begin{figure}[!h]
    \centering
    \includegraphics[width=0.95\linewidth]{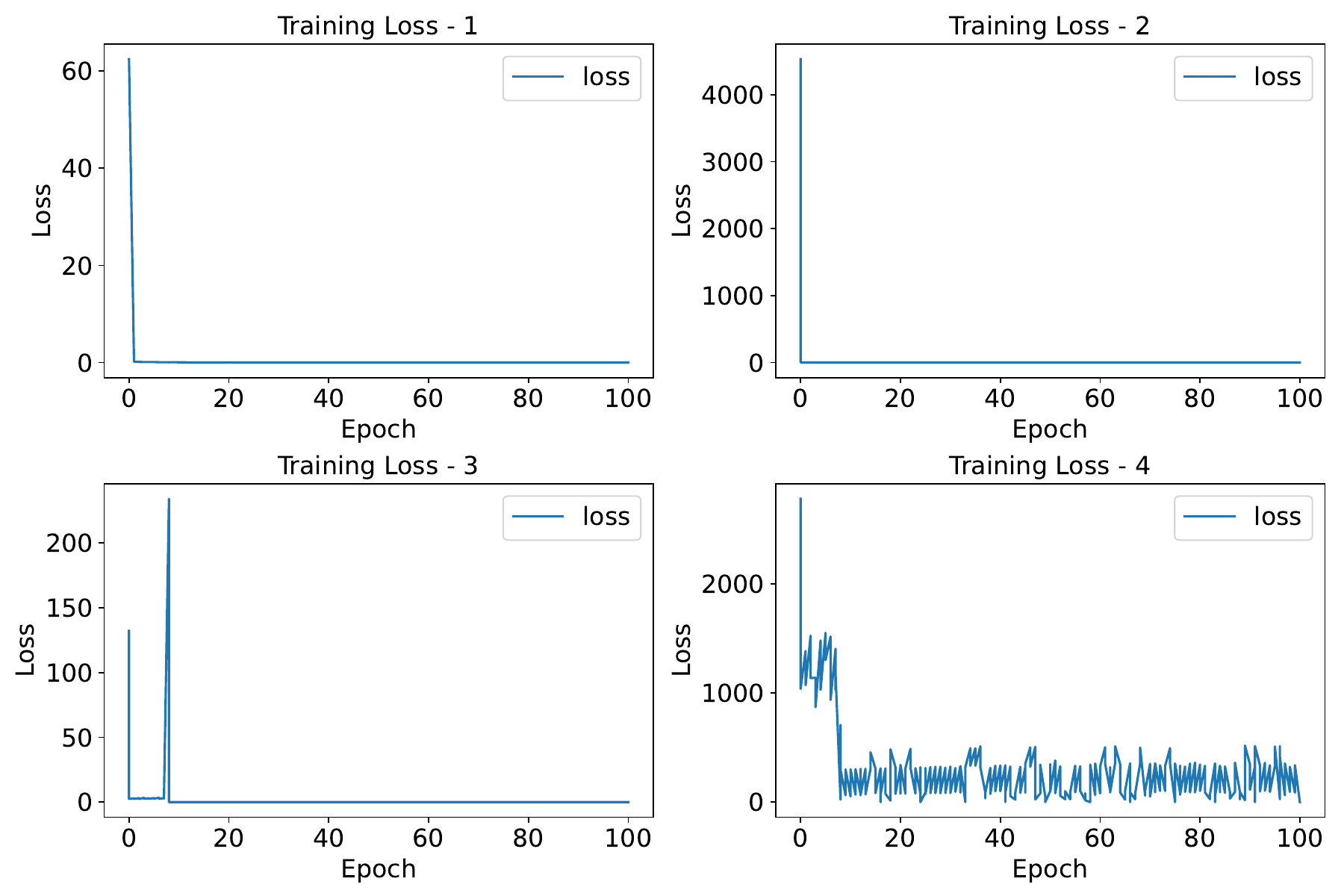}
    \caption{\textcolor{black}{Inner optimation convergence curve(DNO-P).}}
    \label{fig:loss_dno_p}
\end{figure}

\textcolor{black}{\textbf{SFO method:} As an illustrative case, we consider the equation $\frac{x_1}{2.3x_2+4.6}+\sin x_0$. Fig. \ref{fig:loss_sfo} shows the evolution of the reward under the SFO optimization strategy, where the model essentially converges by around step 50. Fig. \ref{fig:loss_bfgs} further provides a more fine-grained view of the BFGS optimization process at steps 20, 40, 100, and 200, demonstrating how the loss progressively decreases to lower values at an accelerating rate.}
 \begin{figure}[!h]
    \centering
    \includegraphics[width=0.85\linewidth]{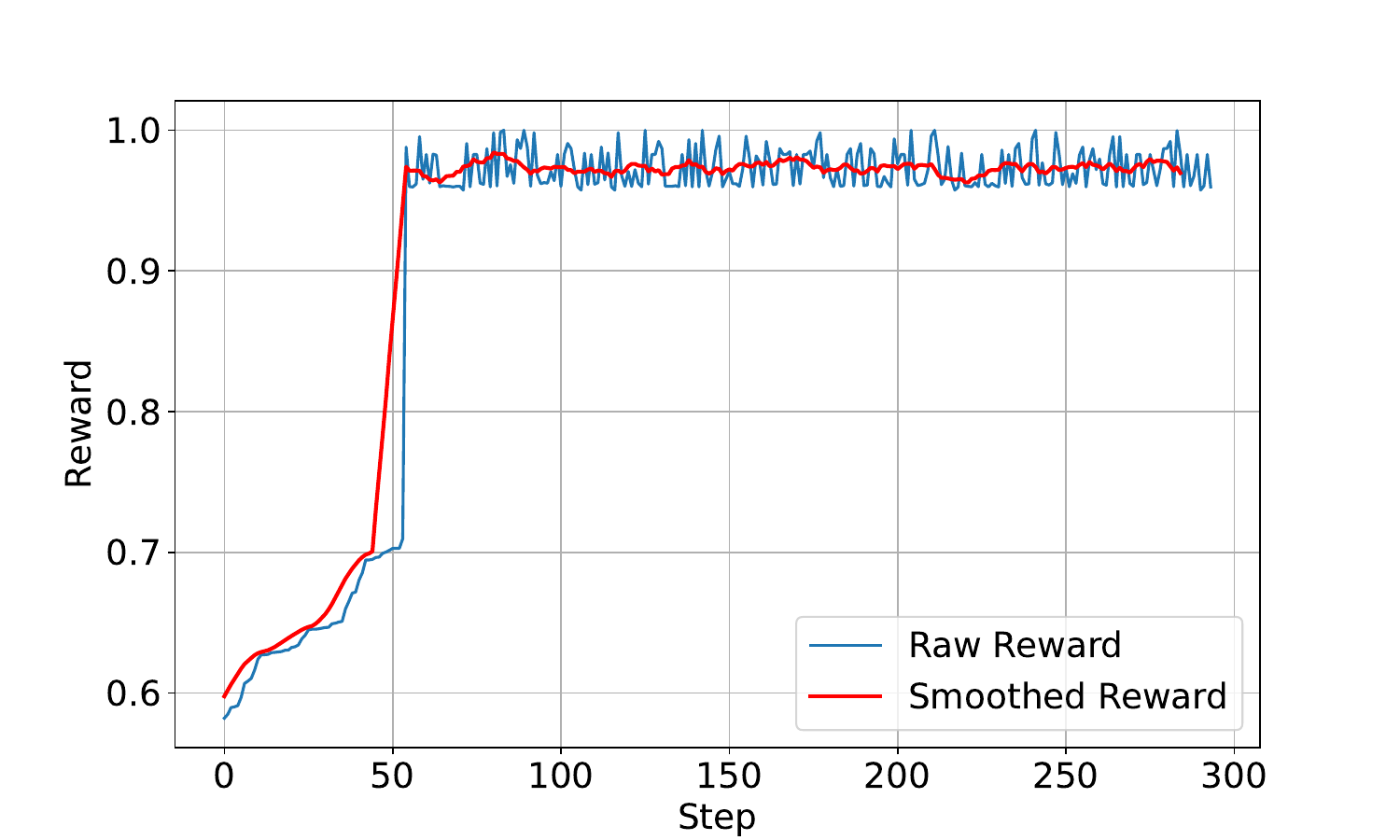}
    \caption{\textcolor{black}{Inner optimation convergence curve(SFO).}}
    \label{fig:loss_sfo}
\end{figure}
\begin{figure}[!h]
    \centering
    \includegraphics[width=0.95\linewidth]{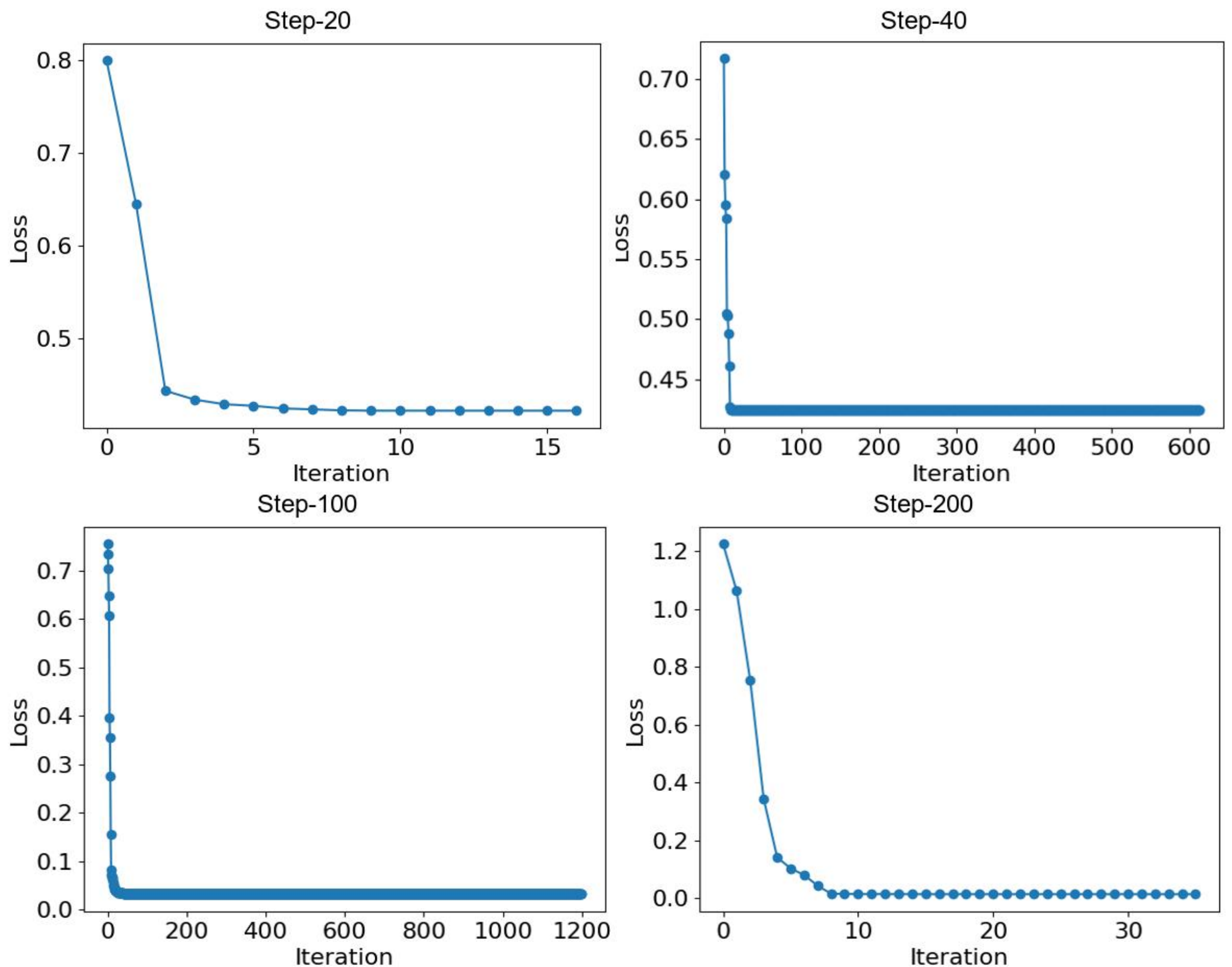}
    \caption{\textcolor{black}{BFGS optimation convergence curve.}}
    \label{fig:loss_bfgs}
\end{figure}

\begin{table*}[!h]
\centering
\caption{Standard Benchmarks.}
\label{Standard Benchmarks.}
\scalebox{1}{
\begin{tabular}{@{}lllll@{}}
\toprule
Benchmark&Name & Expression&Name & Expression \\ 
\toprule
\multirow{6}{*}{\textbf{Nguyen}}&Nguyen-1 & $x^3 + x^2 + x$ & Nguyen-2 & $x^4 + x^3 + x^2 + x$ \\
&Nguyen-3 & $x^5 + x^4 + x^3 + x^2 + x$ &Nguyen-4 & $x^6 + x^5 + x^4 + x^3 + x^2 + x$ \\
&Nguyen-5 & $\sin(x^2) \cos(x) - 1$ & Nguyen-6 & $\sin(x) + \sin(x + x^2)$ \\
&Nguyen-7 & $\log(x+1) + \log(x^2+1)$ & Nguyen-8 & $\sqrt{x}$ \\
&Nguyen-9 & $\sin(x) + \sin(y^2)$ & Nguyen-10 & $2\sin(x) \cos(y)$ \\
&Nguyen-11 & $x^y$ &Nguyen-12 & $x^4 - x^3 + \frac{1}{2}y^2 - y$ \\
\midrule
\multirow{3}{*}{\textbf{Nguyen*}}&Nguyen-1c & $3.39x^3 + 2.12x^2 + 1.78x$ &Nguyen-5c & $\sin(x^2) \cos(x) - 0.75$ \\
&Nguyen-7c & $\log(x+1.4) + \log(x^2+1.3)$&Nguyen-8c & $\sqrt{1.23x}$ \\
&Nguyen-10c & $\sin(1.5x) \cos(0.5y)$& &\\
\midrule
\multirow{4}{*}{\textbf{Constant}}&Constant-1 & $3.39x^3 + 2.12x^2 + 1.78x$ &Constant-2 & $\sin(x^2) \cos(x) - 0.75$ \\
&Constant-3 & $\sin(1.5x) \cos(0.5y)$ & Constant-4 & $2.7xy$ \\
&Constant-5 & $\sqrt{1.23x}$ &Constant-6 & $x^{0.426}$\\
&Constant-7 & $2\sin(1.3x)\cos(y)$ & Constant-8 & $\log(x+1.4) + \log(x^2+1.3)$ \\
\midrule
\multirow{6}{*}{\textbf{Keijzer}}&Keijzer-3 & $0.3 \cdot x \cdot \sin(2\pi x)$ &Keijzer-4 & $x^3 \exp(-x) \cos(x) \sin(x) (\sin(x^2) \cos(x)- 1)$ \\
&Keijzer-6 & $\frac{x(x+1)}{2}$ &Keijzer-7 & $\ln x$ \\
&Keijzer-8 & $\sqrt{x}$&Keijzer-9 & $\log(x + \sqrt{x^2 + 1})$ \\
&Keijzer-10 & $x^y$ &Keijzer-11 & $xy + \sin((x-1)(y-1))$ \\
&Keijzer-12 & $x^4 - x^3 + \frac{y^2}{2} - y$ &Keijzer-13 & $6\sin(x)\cos(y)$ \\
&Keijzer-14 & $\frac{8}{2 + x^2 + y^2}$ &Keijzer-15 & $\frac{x^3}{5} + \frac{y^3}{2} -y - x$\\
\midrule
\multirow{11}{*}{\textbf{Livermore}}&Livermore-1 & $\frac{1}{3} + x + \sin(x^2)$ &Livermore-2 & $\sin(x^2)\cos(x) - 2$ \\
&Livermore-3 & $\sin(x^3)\cos(x^2) - 1$&Livermore-4 & $\log(x+1) + \log(x^2+1) + \log(x)$\\
&Livermore-5 & $x^4 - x^3 + x^2 - y$ &Livermore-6 & $4x^4 + 3x^3 + 2x^2 + x$\\
&Livermore-7 & $\sinh(x)$ &Livermore-8 & $\cosh(x)$ \\
&Livermore-9 & $x^9 + x^8 + x^7 + x^6 + x^5 + x^4 + x^3 + x^2 + x$ &Livermore-10 & $6\sin(x)\cos(y)$ \\
&Livermore-11 & $\frac{x^2 + y^2}{x + y}$ & Livermore-12 & ${\frac{x^4}{y^4}}$ \\
&Livermore-13 & $x^{\frac{2}{3}}$ &Livermore-14 & $x^3 + x^2 + x + \sin(x) + \sin(x^2)$ \\
&Livermore-15 & $x^{\frac{1}{5}}$ &Livermore-16 & $x^{\frac{2}{5}}$ \\
&Livermore-17 & $4\sin(x)\cos(y)$&Livermore-18 & $\sin(x^2)\cos(x) - 5$ \\
&Livermore-19 & $x^5 + x^4 + x^2 + x$ &Livermore-20 & $\exp(-x^2)$ \\
&Livermore-21 & $x^8 + x^7 + x^6 + x^5 + x^4 + x^3 + x^2 + x$&Livermore-22 & $\exp(-0.5x^2)$\\
\midrule
\multirow{3}{*}{\textbf{Jin}}&Jin-1 & $2.5x^4 - 1.3x^3 + 0.5y^2 - 1.7y$&Jin-2 & $8.0x^2 + 8.0y^3 - 15.0$\\
&Jin-3 & $0.2x^3 + 0.5y^3 - 1.2y - 0.5x$ &Jin-4 & $1.5\exp(x) + 5.0\cos(y)$\\
&Jin-5 & $6.0\sin(x)\cos(y)$ &Jin-6 & $1.35xy + 5.5\sin(x - 1.0)(y - 1.0)$\\
\midrule
\multirow{2}{*}{\textbf{R}}&R-1 & $\frac{(x+1)^3}{x^2 -x + 1}$ &R-2 & $\frac{x^5 - 3x^3 + 1}{x^2 + 1}$ \\
&R-3 & $\frac{x^6 + x^5}{x^4 + x^3 + x^2 + x + 1}$ &&\\
\midrule
\multirow{1}{*}{\textbf{Koza}}&Koza-2 & $x^5 - 2x^3 + x$ &Koza-4 & $x^6 - 2x^4 + x^2$\\
\bottomrule
\end{tabular}}
\end{table*}

\begin{table*}[!ht]
    \centering
    \caption{Comprehensive Overview of Baselines.}
    \label{tab: Comprehensive Overview of Baselines.}
    \begin{tabular}{c p{12.5cm}}
        \toprule
        \multicolumn{2}{c}{\textbf{In-domain Methods}}\\
        \toprule
        \textbf{Method} & \textbf{Description} \\
        \midrule
        EQL & A method\citep{sahoo2018learning} that directly trained a fully connected neural network without any guidance on changing the symbolic network framework. \\
        \midrule
        DySymNet &A RL-based method\citep{li2024neural} uses a controller for sampling various architectures of symbolic network guided by policy gradients.  \\
        \midrule
        NeSymRes & A Transformer-based method\citep{biggio2021neural} that used a pre-trained model to obtain the equation's skeleton, treating operators (e.g., '+', '-', 'exp', 'log') as tokens. \\
        \midrule
        End-to-end & A Transformer-based method\citep{kamienny2022end} that didn't rely on skeleton structures, using a hybrid symbolic-numeric vocabulary that encoded both symbolic operators and variables, as well as numeric constants. \\
        \midrule
        TPSR & A Transformer-based method\citep{shojaee2023transformer} for Symbolic Regression that incorporates Monte Carlo Tree Search planning algorithm into the transformer
        decoding process.\\ 
        \midrule
        \textcolor{black}{ParFam} &\textcolor{black}{A novel algorithm\citep{scholl2025parfam} that leverages the inherent structural patterns of physical formulas to transform the original discrete optimization problem into a continuous one, making the search process more efficient and tractable.}\\
        \toprule
        \multicolumn{2}{c}{\textbf{Out-of-domain Methods}} \\
        \toprule
        \textbf{Method} & \textbf{Description} \\
        \midrule
        BSR & A method \citep{jin2019bayesian} generates models as a linear combination of symbolic trees using a Markov Chain Monte Carlo algorithm. \\
        \midrule
        AIFeynman & A divide-and-conquer method \citep{udrescu2020ai} uses neural networks to identify modularities and decompose problems into simpler subproblems. \\
         \midrule
        AFP & A method \citep{schmidt2010age} incorporates model age as an objective to prevent premature convergence and bloat. \\
        \midrule
        AFP$\_$FE & A method combines AFP with Eureqa’s method for fitness estimation. \\
        \midrule
        EPLEX & A parent selection method \citep{la2016epsilon} filters models through randomized subsets of training cases, rewarding those that excel in challenging data regions.\\
        \midrule
        FEAT& A method \citep{lalearning} focuses on discovering simple solutions that generalize well by storing solutions with accuracy-complexity trade-offs.\\
        \midrule
        FFX & A method \citep{mcconaghy2011ffx} initializes a population of equations, selects the Pareto optimal set, and returns a single linear model by treating the population as features.\\
        \midrule
        GP(gplearn) & A method initializes a random population of models and iteratively applies selection, mutation, and crossover operations.\\
        \midrule
        GP-GOMEA & A method\citep{virgolin2021improving} that adapts recombination by modeling interdependencies and recombining components to preserve collaboration.\\
        \midrule
        DSR & A method \citep{petersen2020deep} uses RL to train a generative RNN model of symbolic expressions, employing a risk-seeking policy gradient to bias the model towards exact expressions.\\
        \midrule
        ITEA & A method \citep{de2020discovery} where each model is an affine combination of interaction-transformation expressions, comprising a unary transformation function and a polynomial interaction function.\\
        \midrule
        MRGP & A method\citep{arnaldo2014multiple} where the entire program trace, including each subfunction of the model, is decomposed into features and used to train a Lasso model.\\
        \midrule
        Operon& A method \citep{kommenda2020parameter} integrates nonlinear least squares constant optimization through the Levenberg-Marquardt algorithm within a GP framework.\\
        \midrule
        SBP-GP & A method \citep{virgolin2019linear} refines SBP-based recombination by dynamically adjusting outputs via affine transformations.\\
        \bottomrule
    \end{tabular}
\end{table*}
\begin{table*}[!]
\color{black}
\centering
\caption{\textcolor{black}{Comparison of different initialization methods. The results with the highest accuracy and symbolic solution rate, as well as the shortest testing time, in each dataset are highlighted.}}
\label{tab: Comparison of different initialization methods.}
\renewcommand{\arraystretch}{0.8} 
\scalebox{1}{
\begin{tabular}{cccccc}
\toprule
\textbf{Benchmark} & \textbf{Metrics} & \textbf{Random} & \textbf{Grid search} & \textbf{Latin hypercube} & \textbf{Evolutionary} \\
\toprule
\multirow{4}{*}{\textbf{Koza}} & \emph{Com} & 13.50 & 13.00 & \textbf{13.50} & 14.50 \\
\cmidrule(lr){2-6}
& \emph{Acc} & 1.0000 & 1.0000 & \textbf{1.0000} & 1.0000  \\
\cmidrule(lr){2-6}
& \emph{Sol} & 0.5000 & 0.5000 & \textbf{0.5000} & 0.5000 \\
\cmidrule(lr){2-6}
& \emph{Time} & 81.03 & 37.54 & \textbf{37.23} & 46.15 \\
\midrule
\multirow{4}{*}{\textbf{Keijzer}} & \emph{Com} & 10.00 & \textbf{10.50} & 10.16 & 10.25 \\
\cmidrule(lr){2-6}
 & \emph{Acc} & 1.0000 & \textbf{1.0000} & 1.0000  & 1.0000 \\
\cmidrule(lr){2-6}
 & \emph{Sol} & 0.4167 & \textbf{0.4167} & 0.4167 & 0.4167 \\
 \cmidrule(lr){2-6}
& \emph{Time} & 78.45 & \textbf{72.69} & 74.91 & 75.98 \\
\midrule
\multirow{4}{*}{\textbf{Nguyen}} & \emph{Com} & 13.08 & 13.17 & \textbf{11.50} & 13.17\\
\cmidrule(lr){2-6}
& \emph{Acc} & 0.9167 & 0.8333 & \textbf{0.9167} & 0.9167\\
\cmidrule(lr){2-6}
 & \emph{Sol} & 0.2500 & 0.1667 & \textbf{0.2500} & 0.2500\\
 \cmidrule(lr){2-6}
& \emph{Time} & 87.28 & 75.35 & \textbf{71.22} & 72.01 \\
\midrule
\multirow{4}{*}{\textbf{Nguyen*}} & \emph{Com} & 11.80 & 9.40 & \textbf{11.80} & 11.80\\
\cmidrule(lr){2-6}
 & \emph{Acc} & 1.0000 & 0.8000 & \textbf{1.0000} & 1.0000\\
 \cmidrule(lr){2-6}
& \emph{Sol} & 0.4000 & 0.2000 & \textbf{0.4000} & 0.4000\\
\cmidrule(lr){2-6}
& \emph{Time} & 75.38 & 54.82 & \textbf{50.40} & 52.01\\
\midrule
\multirow{4}{*}{\textbf{R}} & \emph{Com} & 19.33 & 19.33 & 20.00 & 19.33 \\
\cmidrule(lr){2-6}
 & \emph{Acc} & 1.0000 & 1.0000 & 1.0000 & 1.0000\\
 \cmidrule(lr){2-6}
& \emph{Sol} & 0.3333 & 0.3333 & 0.3333 & 0.3333 \\
\cmidrule(lr){2-6}
& \emph{Time} & 100.00 & 100.00 & 100.00 & 100.00 \\
\midrule
\multirow{2}{*}{\textbf{Jin}} & \emph{Com} & 16.50 & \textbf{16.50} & 15.67 & 16.83\\
\cmidrule(lr){2-6}
 & \emph{Acc} & 1.0000 & \textbf{1.0000} & 1.0000 & 1.0000\\
 \cmidrule(lr){2-6}
& \emph{Sol} & 0.5000 & \textbf{0.5000} & 0.5000 & 0.5000\\
\cmidrule(lr){2-6}
& \emph{Time} & 88.79 & \textbf{69.14} & 69.79 & 72.01 \\
\midrule
\multirow{4}{*}{\textbf{Livermore}} & \emph{Com} &11.00 & 10.86 & 11.50 & \textbf{10.82}\\
\cmidrule(lr){2-6}
 & \emph{Acc} & 1.0000 & 0.09091 & 1.0000 & \textbf{1.0000}\\
 \cmidrule(lr){2-6}
& \emph{Sol} & 0.4091 & 0.3183 & 0.4091 & \textbf{0.4091}\\
\cmidrule(lr){2-6}
& \emph{Time} & 68.60 & 49.65 & 53.30 & \textbf{52.29}\\
\midrule
\multirow{4}{*}{\textbf{Constant}} & \emph{Com} & 9.75 & 8.63 & 9.63 & \textbf{9.63}\\
\cmidrule(lr){2-6}
 & \emph{Acc} & 1.0000 & 0.7500 & 1.0000 & \textbf{1.0000}\\
 \cmidrule(lr){2-6}
 & \emph{Sol} & 0.7500 & 0.5000 & 0.7500 & \textbf{0.7500}\\
 \cmidrule(lr){2-6}
& \emph{Time} & 66.11 & 62.72 & 59.16 & \textbf{57.46} \\
\bottomrule
\end{tabular}}
\end{table*}

\begin{table*}[!]
\color{black}
\centering
\caption{\textcolor{black}{Relative changes of Huber and Quantile losses with respect to MSE baseline.}}
\label{tab:loss_comparison_delta}
\renewcommand{\arraystretch}{0.9}
\scalebox{1}{
\begin{tabular}{c|ccc|ccc}
\toprule
\multirow{2}{*}{Benchmark} 
& \multicolumn{3}{c|}{Huber $\Delta$ (vs. MSE)} 
& \multicolumn{3}{c}{Quantile $\Delta$ (vs. MSE)} \\
\cmidrule(lr){2-4} \cmidrule(lr){5-7}
& $\Delta R^2$ & $\Delta$ Sol & $\Delta$ Com 
& $\Delta R^2$ & $\Delta$ Sol & $\Delta$ Com \\
\midrule
Koza       & 0 & 0 & $\uparrow\!0.50$ & 0 & 0 & $\downarrow\!1.50$ \\
Keijzer    & $\uparrow 2e-5$     & 0   & $\downarrow 0.17$    & $\downarrow 8e-4$   & 0 & 0 \\
Nguyen     & $\downarrow 2.02e-3$     & 0   & $\downarrow 1.92$    & $\downarrow 1.87e-3$  & $\uparrow 0.0833$  & $\downarrow 2.25$  \\
Nguyen*    & $\downarrow 2.50e-3$     & 0   & $\downarrow 2.20$    & $\downarrow 2.37e-3$  & 0  & $\downarrow 2.20$  \\
R          & $\uparrow 3e-5$& 0   & $\downarrow 3.00$& 0   & 0 & 0 \\
Jin        & $\downarrow 3.98e-3$   & 0   & $\downarrow 0.83$    & $\downarrow 3.36e-3$  & 0 & $\downarrow 0.33$\\
Livermore  & $\downarrow 2.42e-3$     & 0   & $\downarrow 0.64$    & $\downarrow 1.47e-3$   & 0 & $\downarrow 0.23$ \\
Constant   & $\downarrow 1e-5$&  0   & $\uparrow 0.75$& $\downarrow 8e-5$  & 0   & $\downarrow 0.75$ \\
\bottomrule
\end{tabular}}
\end{table*}
\begin{table*}[!h]
\color{black}
\centering
\caption{\textcolor{black}{Sensitivity analysis of the parameters. The results with the highest Average $R^2$ Score and Symbolic Solution Rate, as well as the lowest Average Complexity, are highlighted.}}
\label{tab: Sensitivity analysis of the parameters}
\scalebox{1.1}{
\begin{tabular}{@{}lccccccccc@{}}
\toprule
Temperature Scheduling& \multicolumn{3}{c}{Constant} & \multicolumn{3}{c}{Linear}& \multicolumn{3}{c}{Exponent} \\
\midrule
Entropy Coefficient & 0.005 & 0.01 & 0.05 & 0.005 & 0.01 & 0.05& 0.005 & 0.01 & 0.05 \\
\cmidrule(lr){1-1}\cmidrule(lr){2-4} \cmidrule(lr){5-7}\cmidrule(lr){8-10}
Average $R^2$ Score  
& 0.9976 & 0.9976 & 0.9976 
& 0.9976 & 0.9976 & 0.9976 
& \textbf{0.9989} & 0.9976 & 0.9976 \\
\scriptsize{$\pm$ 95\% CI.} 
& \scriptsize{5.21e-4} & \scriptsize{5.21e-4} & \scriptsize{5.21e-4} 
& \scriptsize{5.21e-4} & \scriptsize{5.20e-4} & \scriptsize{5.20e-4} 
& \scriptsize{5.18e-4} & \scriptsize{5.20e-4} & \scriptsize{5.21e-4} \\
Symbolic Solution Rate & 0.1667 & 0.1667 & 0.1667 & 0.2500 & 0.2458 & 0.2458 & \textbf{0.2500} & 0.2458 & 0.2458 \\
Average Complexity & 14.83 & 14.83 & 14.87 & 14.83 & 14.85 & 14.85 & \textbf{13.50} & 14.85 & 14.89 \\
\bottomrule
\end{tabular}}
\end{table*}

\begin{table*}[!h]
\centering
\caption{Hyperparameters Setting.}
\label{tab:Hyperparameters Setting.}
\begin{tabular}{@{}lccc@{}}
\toprule
\textbf{Hyperparameter} & \textbf{Symbol} &  \multicolumn{2}{c}{\textbf{Value}} \\ 
\toprule
\textbf{Data Generation} & && \\
\midrule
 &  &  $d\le4$ & $d>4$ \\
\midrule
Max Input Dimension & \( d_{\text{max}} \) & 4 & 10 \\
Max Binary Operators &\( b_{\text{max}} \) & $d+5$ & $d+1$ \\
Max Hidden Layers &\( L_{\text{max}} \) & 6 & 7 \\
Number of repetitions of operators &\( m \) & 5 & 7 \\
Number of Functions &- & 1,500,000 & 3,750,000 \\
Input Dimension Distribution  &- &   \{1:0.1,2:0.2,3:0.3,4:0.4\} & \{5:0.2,6:0.2,7:0.15,8:0.15,9:0.15,10:0.15\} \\

\midrule
Train Data Points &\( N \) & \multicolumn{2}{c}{200} \\
Max Unary Operators & \( u_{\text{max}} \)  & \multicolumn{2}{c}{$\min\{5,d+1\}$} \\
Binary Operator Sampling Frequency &- & \multicolumn{2}{c}{$\{\text{add}$: 1, $\text{sub}$: 1, $\text{mul}$: 1, $\text{pow}$: 1$\}$} \\
Unary Operator Sampling Frequency &- & \multicolumn{2}{c}{$\{\log$: 0.3, $\exp$: 1.1, $\sin$: 1.1, $\cos$: 1.1, $\text{sqrt}$: 3, 
$\text{inv}$: 5, $\text{abs}$: 1, $\text{pow2}$: 2$\}$} \\
Distribution of $(w,b)$ &$\mathcal{D}_{\text{aff}}$ & \multicolumn{2}{c}{$\{\text{sign} \sim \mathcal{U}(-1,1)\}$, $\{\text{mantissa} \sim \mathcal{U}(0,1)\}$, $\{\text{exponent} \sim \mathcal{U}(-1,1)\}$} \\
\midrule
\textbf{Pre-Training} & & \\
\midrule
Learning rate & \( \eta \) & \multicolumn{2}{c}{0.0001} \\
Learning rate decay &\( \gamma \) & \multicolumn{2}{c}{0.99} \\
Batch size & \( B \) & \multicolumn{2}{c}{512} \\
Embedding size & - & \multicolumn{2}{c}{512} \\
Training epochs  & $K$ & \multicolumn{2}{c}{20($d \leq 4$), 40($d > 4$)} \\
\midrule
\textbf{Beam Search} & & \\
\midrule
Beam size & - & \multicolumn{2}{c}{20} \\
Number of candidate sequences  & - & \multicolumn{2}{c}{5} \\
\midrule
\textbf{Differentable Network Training} & & \\
\midrule
Learning rate & \( \eta \) & \multicolumn{2}{c}{0.01} \\
Batch size & \( B \) & \multicolumn{2}{c}{64} \\
Weight initialization parameter & - & \multicolumn{2}{c}{2} \\
Training epochs & \( K \) & \multicolumn{2}{c}{$10L$} \\
Penlty threshold of $log_{reg}$ & \( \theta_{log} \) & \multicolumn{2}{c}{$0.0001$} \\
Penlty threshold of $exp_{reg}$ & \( \theta_{exp} \) & \multicolumn{2}{c}{$100$} \\
Penlty Epoch &- & \multicolumn{2}{c}{$10$} \\
\midrule
\textbf{Symbolic Function Optimization} & & \\
\midrule
Number of restarts & - & \multicolumn{2}{c}{10} \\
Number of points  & - & \multicolumn{2}{c}{200} \\
Random exponent & $\mathcal{V}$ & \multicolumn{2}{c}{$\subset$\{1, -1, 0.5, -0.5, 2, -2, 3, -3, 4, -4, 5, \textcolor{black}{*}\}} \\ 
\bottomrule
\end{tabular}
\end{table*}

\begin{table*}[!h]
\color{black}
\centering
\caption{\textcolor{black}{Representative examples of equations optimized with DNO-NP from the Feynman dataset.}}
\label{tab: dnonp_examples}
\scalebox{1.1}{
\begin{tabular}{@{}ll@{}}
\toprule
Method&Datasets\\ 
\toprule
\multirow{2}{*}{DNO-NP}&Feynman\_I\_18\_4, Feynman\_I\_40\_1, Feynman\_I\_50\_26,
Feynman\_III\_14\_14, Feynman\_test\_2, \\
&Feynman\_test\_4, Feynman\_test\_6, Feynman\_test\_7, Feynman\_test\_16,
Feynman\_test\_19\\
\bottomrule
\end{tabular}}
\end{table*}

\begin{table*}[!h]
\color{black}
\centering
\caption{\textcolor{black}{Representative examples of discovered equations in Standard Benchmarks.}}
\label{tab: Representative examples in Standard_benchmarks.}
\scalebox{1}{
\begin{tabular}{@{}llllcc@{}}
\toprule
Benchmark&Name & Ground Truth& Discovered Equation&$R^2$&Sol \\ 
\toprule
\multirow{2}{*}{\textbf{Koza}}&Koza-2 & $x^5 - 2x^3 + x$ & $x(x^2-1)^2$&1.0000&T\\
&Koza-4 & $x^6 - 2x^4 + x^2$ &  $4.2x(0.032x^2+x-0.692)^4+0.084$&0.9999&F\\
\midrule
\multirow{5}{*}{\textbf{Keijzer}}&Keijzer-3 & $0.3\sin(2\pi x)$& $0.3\sin(2\pi x)$ &1.0000&T \\
&Keijzer-6 & $\frac{x(x+1)}{2}$ &$\frac{x^2+x}{2}$&1.0000&T \\
&Keijzer-7 & $\log x$ &$\log x$&1.0000&T \\
&Keijzer-10 & $x^y$ &$0.087x^{1.78\sqrt{y}}\exp(1.326\sqrt{y})$&1.0000&F \\
&Keijzer-11 & $xy+\sin(x-1)(y-1)$ &$-0.002x^4y+14.8\sin(0.261x+0.274y+4.75)+7.96$ &0.9912&F\\
&Keijzer-13 & $6\sin x\cos y$ &$6\sin x\cos y$ &1.0000&T \\
\midrule
\multirow{6}{*}{\textbf{Nguyen}}&Nguyen-1 & $x^3 + x^2 + x$ &  $x(x^2+x+1)$ &1.0000&T\\
&Nguyen-2 & $x^4+x^3 + x^2 + x$ &  $0.826x+1.03x^2(x+0.49)^2+2.33$&0.9999&F \\
&Nguyen-5 & $\sin(x^2) \cos(x) - 1$ & $8.04\sin^2(0.365x)cos^2(0.824x)-1$ &0.9994&F\\
&Nguyen-7 & $\log(x+1) + \log(x^2+1)$ & $5.97\log(2.38x+1.64)+2.62\log(10.2x+4.01)$ &0.9999&F\\
&Nguyen-8 & $\sqrt{x}$  & $x^{0.5}$ &1.0000&T\\
&Nguyen-9 & $\sin(x) + \sin(y^2)$  & $\sin(x) + \sin(y^2)$  &1.0000&T\\
&Nguyen-10 & $2\sin x\cos y$  & $2\sin x\cos y$  &1.0000&T\\
\midrule
\multirow{4}{*}{\textbf{Nguyen*}}&Nguyen-1c & $3.39x^3 + 2.12x^2 + 1.78x$ & $1.34x+3.39(x+0.208)^3-0.031$&0.9999&F\\
&Nguyen-5c & $\sin(x^2)\cos x-0.75$ & $0.818\sin(1.22x^2)\cos^2(0.692x)-0.75$ &0.9999&F\\
&Nguyen-8c & $\sqrt{1.23x}$ & $1.11x^{0.5}$ &1.0000&T\\
&Nguyen-10c & $\sin(1.5x)\cos(0.5y)$ & $\sin(1.5x)\cos(0.5y)$ &1.0000&T\\
\midrule
\multirow{2}{*}{\textbf{R}}&R-1 & $\frac{(x+1)^3}{x^2-x+1}$ &  $\frac{(x+1)^3}{x^2-x+1}$ &1.0000&T\\
&R-2 & $\frac{x^5-3x^2+1}{x^2+1}$ &  $\frac{1.17(1-0.936x^2)^2}{x}$&0.9999&F\\
\midrule
\multirow{2}{*}{\textbf{Jin}}&Jin-2 & $8.0x^2+8.0y^3-15.0$ &  $8.0x^2+8.0y^3-15.0$&1.0000&T\\
&Jin-4 & $1.5\exp(x)+5.0\cos y$ &  $1.5\exp(x)+5.0\cos y$&1.0000&T\\
\midrule
\multirow{6}{*}{\textbf{Livermore}}&Livermore-5 & $x^4-x^3+x^2-y$ &  $0.98x^4-0.764x^3+1.74x-y-1.02$&0.9999&T \\
&Livermore-6 & $4x^4+3x^3+2x^2+x$ &  $x^3(3.79x+0.00316x^4+4.36)+2.59$&0.9999&T\\
&Livermore-7 & $\sinh x$ &  $\frac{e^x-e^{-x}}{2}$ &1.0000&T\\
&Livermore-12 & $\frac{x^4}{y^4}$ &  $\frac{x^4}{y^4}$&1.0000&T\\
&Livermore-15 & $x^{\frac{1}{5}}$ &  $x^{0.2}$&1.0000&T\\
&Livermore-20 & $\exp(-x^2)$ &  $\exp(-x^2)$&1.0000&T\\
\midrule
\multirow{2}{*}{\textbf{Constant}}&Constant-4 & $2.7xy$ &  $2.7xy$&1.0000&T\\
&Constant-8& $\log(x+1.4)+\log(x^2+1.3)$ &  $2.48\log(0.27x+1.52)+0.428\log(0.908x^{3.4}+0.871)$&0.9999&F\\
\bottomrule
\end{tabular}}
\end{table*}

\begin{table*}[!h]
\color{black}
\centering
\caption{\textcolor{black}{Representative examples of discovered equations in SRBench.}}
\label{tab: Representative examples in SRBench.}
\scalebox{1}{
\begin{tabular}{@{}llllcc@{}}
\toprule
Benchmark&Name & Ground Truth& Discovered Equation&$R^2$&Sol \\ 
\toprule
\multirow{5}{*}{\textbf{ODE-Strogatz}}&Strogatz\_bacres2 & $10-\frac{xy}{1+0.5x^2}$ & $\frac{-1.434y^{0.992}}{x^{0.768}}-0.001x+0.041y+10.051$&0.9999&F\\
&Strogatz\_glider1 & $-0.05x^2-\sin y$ & $-0.05x^2-\sin y$&1.0000&T\\
&Strogatz\_lv1 & $3x-2xy-x^2$ & $3x-2xy-x^2$&1.0000&T\\
&Strogatz\_lv2 & $2y-xy-y^2$ & $2y-xy-y^2$&1.0000&T\\
&Strogatz\_predprey2 & $y(\frac{x}{1+x}-0.075y)$ & $\frac{-140.822(0.016y+1)^2}{14.927x-2.422y+39.805}+4.065$&0.9904&F\\
\midrule
\multirow{25}{*}{\textbf{Feynman}}&Feynman\_I\_6\_2a & $\frac{e^{-0.5\theta^2}}{\sqrt{2\pi}}$& $0.399e^{-0.5\theta^2}$ &0.9999&T \\
&Feynman\_I\_8\_14 & $\sqrt{(x_2-x_1)^2+(y_2-y_1)^2}$& $\sqrt{(x_2-x_1)^2+(y_2-y_1)^2}$ &1.0000&T \\
&Feynman\_I\_11\_19 & $x_1y_1+x_2y_2+x_3y_3$& $x_1y_1+x_2y_2+x_3y_3$ &1.0000&T \\
&Feynman\_I\_12\_2 & $\frac{q_1q_2r}{4\pi \epsilon r^3}$& $0.130\frac{q_1q_2}{4\pi \epsilon r^2}$ &1.0000&T \\
&Feynman\_I\_12\_11 & $q(Ef+Bv\sin\theta)$&$q(Ef+Bv\sin\theta)$&1.0000&T \\
&Feynman\_I\_13\_4 & $\frac{1}{2}m(v^2+u^2+w^2)$&$0.500m(v^2+u^2+w^2)$&1.0000&T \\
&Feynman\_I\_13\_12 & $\frac{Gm_1m_2}{\frac{1}{r_2}-\frac{1}{r_1}}$&$\frac{Gm_1m_2}{\frac{1}{r_2}-\frac{1}{r_1}}$&1.0000&T \\
&Feynman\_I\_18\_14 & $mrv\sin\theta$&$mrv\sin\theta$&1.0000&T \\
&Feynman\_I\_24\_6 & $\frac{1}{4}m(w^2+w_0^2)x^2$&$0.25m(w^2+w_0^2)x^2$&1.0000&T \\
&Feynman\_I\_25\_13 & $\frac{q}{c}$&$\frac{q}{c}$&1.0000&T \\
&Feynman\_I\_27\_6 & $\frac{1}{\frac{1}{d_1}+\frac{n}{d_2}}$&$\frac{d_2}{\frac{d_2}{d_1}+n}$&1.0000&T \\
&Feynman\_I\_30\_5 & $\arcsin{\frac{\lambda}{nd}}$&$\frac{0.953\lambda^{1.059}}{n^{1.063}(d-0.136)^{0.984}}+0.005$&0.9997&F \\
&Feynman\_I\_32\_5 & $\frac{q^2a^2}{6\pi\epsilon c^3}$&$0.0531\frac{q^2a^2}{\epsilon c^3}$&1.0000&T \\
&Feynman\_I\_37\_4 & $I_1+I_2+\sqrt{I_1I_2}\cos\theta$&$I_1+I_2+2(I_1I_2)^{0.5}\cos\theta$
&1.0000&T\\ 
&Feynman\_I\_38\_12 & $4\pi\epsilon\frac{(\frac{h}{2\pi})^2}{mq^2}$&$0.318\frac{\epsilon h^2}{mq^2}$&1.0000&T \\
&Feynman\_I\_44\_4 & $K_bT\log{\frac{v_2}{v_1}}$&$-K_bT\log{\frac{v_1}{v_2}}$&1.0000&T \\
&Feynman\_I\_47\_23 & $\sqrt{\frac{\gamma p_r}{\rho}}$&$({\frac{\gamma p_r}{\rho}})^{0.5}$&1.0000&T \\
&Feynman\_II\_6\_11 & $\frac{p_d\cos\theta}{4\pi\epsilon r^2}$&$\frac{0.015p_d(2.613\cos(0.986\theta)-0.056)}{\epsilon(1.442r-1.003)}$&0.9941&F \\
&Feynman\_II\_6\_15b & $\frac{3p_d\sin\theta\cos\theta}{4\pi\epsilon r^3}$&$0.119\frac{p_d\sin(2\theta)}{\epsilon r^3}$&1.0000&T \\
&Feynman\_II\_10\_19 & $\frac{\sigma_{den}(1+\chi)}{\epsilon}$&$\frac{\sigma_{den}(1+\chi)}{\epsilon}$&1.0000&T \\
&Feynman\_II\_11\_28 & $1+\frac{n\alpha}{1-\frac{n\alpha}{3}}$&$0.802n^{1.222}(1.464n^{0.889}+0.207\alpha)^{1.418}+1.015$&0.9992&F \\
&Feynman\_II\_13\_23 & $\frac{\rho_{c_0}}{\sqrt{1-\frac{v^2}{c^2}}}$&$\frac{\rho_{c_0}}{\sqrt{1-\frac{v^2}{c^2}}}$&1.0000&T \\
&Feynman\_II\_21\_32 & $\frac{q}{4\pi\epsilon r(1-\frac{v}{c})}$&$0.0531\frac{q}{\epsilon r(1-\frac{v}{c})}$&1.0000&T \\
&Feynman\_II\_35\_18 & $\frac{n_0}{\exp(\frac{momB}{K_bT}+\exp(-\frac{momB}{K_bT})}$&$\frac{0.506n_0^{0.995}}{(\frac{0.2mom^{1.888}B^{1.89}}{(K_bT)^{1.889}+1})^{2.396}}-0.003$&0.9999&F \\
&Feynman\_III\_4\_32 & $\frac{1}{\exp(\frac{hw}{2\pi K_bT})-1}$&$\frac{1}{\exp(0.159\frac{hw}{K_bT})-1}$&1.0000&T \\
&Feynman\_III\_15\_12 & $2U(1-\cos(kd))$&$2U(1-\cos(kd))$&1.0000&T \\
\bottomrule
\end{tabular}}
\end{table*}

\FloatBarrier

\clearpage

\begin{spacing}{0.93}   
\bibliographystyle{UniSymNet}
\bibliography{UniSymNet}
\end{spacing}

\end{document}